\documentclass[11pt]{article}

\usepackage{amsmath,amsfonts,bm}

\def\eqref#1{equation~\ref{#1}}

\def\ceil#1{\lceil #1 \rceil}
\def\floor#1{\lfloor #1 \rfloor}
\def\1{\bm{1}}

\DeclareMathAlphabet{\mathsfit}{\encodingdefault}{\sfdefault}{m}{sl}
\SetMathAlphabet{\mathsfit}{bold}{\encodingdefault}{\sfdefault}{bx}{n}

\newcommand{\E}{\mathbb{E}}

\newcommand{\R}{\mathbb{R}}
\newcommand{\Z}{\mathbb{Z}}
\newcommand{\N}{\mathbb{N}^+}

\def\zero{\boldsymbol{0}}

\def\a{\boldsymbol{a}}
\def\e{\boldsymbol{e}}
\def\p{\boldsymbol{p}}

\def\h{\boldsymbol{h}}
\def\bu{\boldsymbol{u}}
\def\bv{\boldsymbol{v}}
\def\x{\boldsymbol{x}}
\def\y{\boldsymbol{y}}
\def\z{\boldsymbol{z}}

\def\P {\mathbb{P}}
\def\E {\mathbb{E}}
\def\bzeta{\boldsymbol{\zeta}}

\usepackage[a4paper,top=3cm,bottom=3cm,left=3cm,right=3cm,marginparwidth=1.75cm]{geometry}
\usepackage[utf8]{inputenc} 
\usepackage[T1]{fontenc}    
\usepackage{babel}
\usepackage{url}            
\usepackage{booktabs}       
\usepackage{amsfonts,mathrsfs}       
\usepackage{nicefrac}       
\usepackage{microtype}      
\usepackage{lipsum}		
\usepackage{graphicx}
\usepackage{natbib}
\usepackage{doi}
\usepackage{physics}
\usepackage{color}
\usepackage{soul}
\usepackage{multirow}
\usepackage{algorithm}
\usepackage{algorithmicx}
\usepackage{algpseudocode}
\usepackage{caption}
\usepackage{amsmath,amssymb,amsthm}
\usepackage{xr}
\usepackage{xcolor}
\usepackage{blindtext}
\usepackage{caption}
\usepackage{hyperref}       
\usepackage[title]{appendix}
\usepackage{enumerate}
\usepackage{array}
\usepackage{mathbbol}
\usepackage{booktabs}
\usepackage{enumitem}
\usepackage{authblk}  
\usepackage{longtable}

\allowdisplaybreaks

\newtheorem{thm}{Theorem}[section]

\newtheorem{lem}{Lemma}[section]
\newtheorem{cor}{Corollary}[section]
\theoremstyle{definition}
\newtheorem{dfn}{Definition}[section]
\theoremstyle{remark}
\newtheorem{rem}{Remark}[section] 

\numberwithin{equation}{section}

\def\u{{\textnormal{u}}}
\def\b{{\boldsymbol{b}}}

\def\TF{{\textnormal{TF}}}
\def\Attn{{\textnormal{Attn}}}
\def\FFN{{\textnormal{FFN}}}
\def\id{{\textnormal{id}}}

\def\pro{{\textnormal{P}}}
\def\dat{{\textnormal{D}}}
\def\Id{{\mathbb{I}}}
\def\EUAF{{\textnormal{EUAF}}}
\def\spn{{\operatorname{span}}}

\title{A Theoretical Framework for Prompt Engineering:\\
Approximating Smooth Functions with Transformer Prompts}

\makeatletter
\patchcmd{\AB@output}{\par\vspace{1em}}{}{}{}
\renewcommand\AB@affilsepx{\quad}  
\makeatother

\author[1]{Ryumei Nakada}
\author[2]{Wenlong Ji}
\author[3]{Tianxi Cai}
\author[2]{James Zou}
\author[1]{Linjun Zhang}

\affil[1]{Rutgers University}
\affil[2]{Stanford University}
\affil[3]{Harvard University}

\begin{document}
\maketitle

\begin{abstract} 
    Prompt engineering has emerged as a powerful technique for guiding large language models (LLMs) toward desired responses, significantly enhancing their performance across diverse tasks. Beyond their role as static predictors, LLMs increasingly function as intelligent agents, capable of reasoning, decision-making, and adapting dynamically to complex environments. However, the theoretical underpinnings of prompt engineering remain largely unexplored. In this paper, we introduce a formal framework demonstrating that transformer models, when provided with carefully designed prompts, can act as a configurable computational system by emulating a ``virtual'' neural network during inference. Specifically, input prompts effectively translate into the corresponding network configuration, enabling LLMs to adjust their internal computations dynamically. Building on this construction, we establish an approximation theory for $\beta$-times differentiable functions, proving that transformers can approximate such functions with arbitrary precision when guided by appropriately structured prompts. Moreover, our framework provides theoretical justification for several empirically successful prompt engineering techniques, including the use of longer, structured prompts, filtering irrelevant information, enhancing prompt token diversity, and leveraging multi-agent interactions. By framing LLMs as adaptable agents rather than static models, our findings underscore their potential for autonomous reasoning and problem-solving, paving the way for more robust and theoretically grounded advancements in prompt engineering and AI agent design.
\end{abstract}

\section{Introduction}

The emergence of large language models (LLMs) has revolutionized a broad spectrum of natural language processing tasks, but their impact extends far beyond text-based applications. LLMs are increasingly recognized as general-purpose reasoning engines capable of tackling complex problems across multiple domains. These models have achieved remarkable milestones in artificial intelligence, excelling in various tasks including language translation \citep{costa2022no}, text summarization \citep{lewis2019bart}, code generation \citep{chen2021evaluating}, protein engineering \citep{shen2024proteinengine}, image generation \citep{wang2024genartist}, and even solving Olympiad-level mathematical problems \citep{trinh2024solving}. At the heart of the development of LLMs is the transformer model. Transformer, as introduced by \citet{vaswani2017attention}, utilizes a self-attention mechanism to capture dependencies between input tokens without being restricted by their relative position in the sequence. Since first proposed in 2017, the transformer model has become the fundamental building block of most of the state-of-the-art models in various artificial intelligence applications \citep{brown2020language, dosovitskiy2020image, brohan2022rt}. 

A key advantage of using transformers in LLMs is their ability to handle prompts---sequences of input tokens used to guide the model towards generating relevant responses. The use of prompts plays a critical role in enhancing model performance in diverse applications, as they provide execution guidance needed for the model to solve a given task and ensure that the model generation aligns with human instruction. For example, \citet{wei2022chain} and \citet{kojima2022large} found that using a Chain-of-Thought (CoT) prompting can significantly increase the reasoning capabilities of LLMs. Intuitively, the CoT prompting explicitly instructs the model to solve the problem using step-by-step reasoning. This helps decompose a complicated problem into simple steps, reducing the chance of errors in the final answer (Figure~\ref{fig: prompt engineer}). 

\begin{figure}
    \centering
    \includegraphics[width=0.47\linewidth]{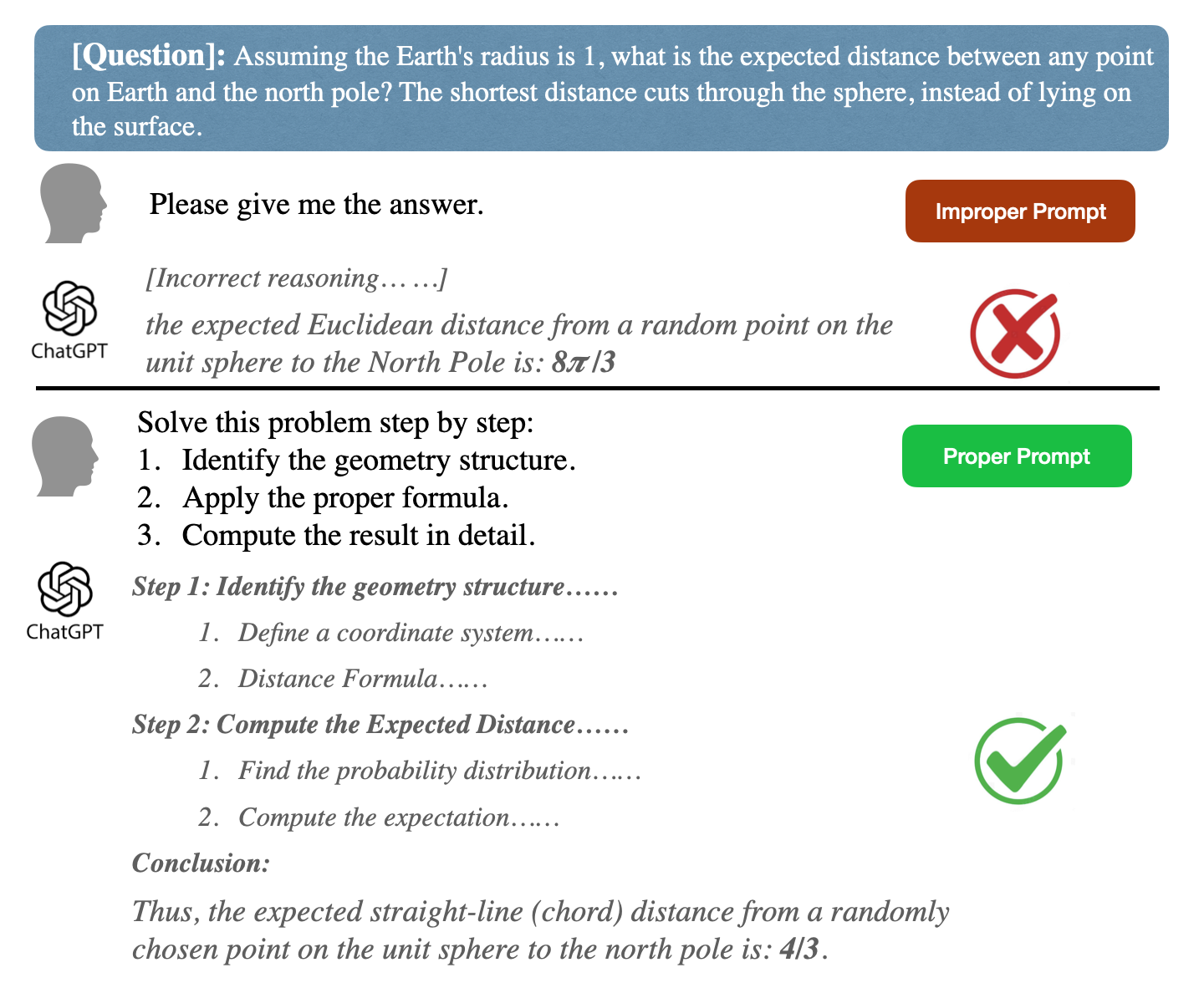}
    \includegraphics[width=0.52\linewidth]{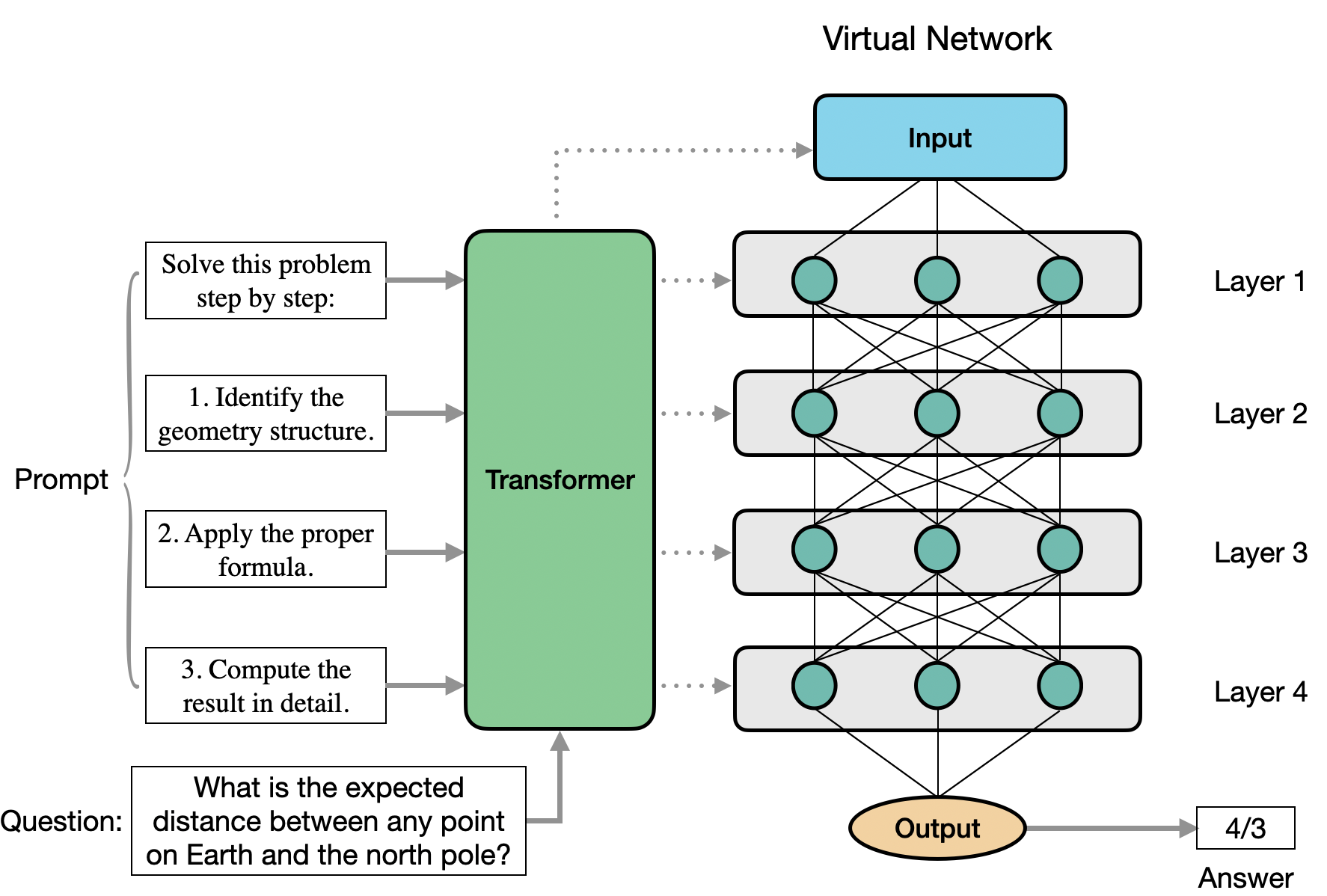}
    \caption{\textbf{Left}: Example of prompt engineering. The responses are collected from GPT-4o, and the detailed computations are omitted for simplicity. Proper prompt design can improve the reasoning ability of large language model generations. \textbf{Right}: Illustration of our theory. Transformer can emulate a ``virtual'' neural network based on the prompts to execute a given task.}
    \label{fig: prompt engineer}
\end{figure}

Despite the practical success of prompt engineering, the mechanism by which transformers utilize input prompts to solve tasks remains poorly understood. Recent advances in approximation theory have shed light on the expressive power of transformers when augmented with prompts, offering new insights into their capabilities. For instance, \citet{wang2024universality,petrov2024prompting,hu2024fundamental} demonstrated that sufficiently long prompts enable transformers to approximate Lipschitz functions arbitrarily well. 
In contrast to these works, which primarily focus on Lipschitz continuous functions, our approach introduces a fundamentally different mechanism by viewing the sequential output as a layer-wise emulation of a neural network, thereby enabling the approximation of smooth functions.

Moreover, the in-context learning (ICL) literature has provided valuable insights into how transformers adapt by leveraging multiple examples embedded within the input. Research in this area has demonstrated that transformers can effectively perform tasks such as gradient descent \citep{akyurek2022learning,von2022transformers}, ridge regression, least squares \citep{akyurek2022learning,guo2023transformers}, algorithm selection \citep{bai2024transformers}, and kernel regression \citep{han2023context}. However, these studies primarily focus on example-driven adaptation, rather than examining the structural role of the prompt itself in shaping the model's internal representation.   

Our work addresses a critical gap in understanding the theoretical foundations of prompt engineering by shifting the focus toward a paradigm in which an instruction precedes the data, effectively acting as a dynamic function applied during inference. In our approach, we adopt a general prompt paradigm where an instruction precedes the data, and analyze how the prompts shape subsequent data processing in transformer models. 
By integrating this perspective with approximation theory, we derive error bounds for approximating $\mathcal{C}^\beta$ functions.
In doing so, our work offers theoretical justifications for several empirically observed phenomena: 
(1) longer and more detailed prompts enhance the model performance \citep{brown2020language,min2022rethinking,garg2022can,fu2022complexity}, (2) filtering out irrelevant tokens improves prediction accuracy \citep{shi2023large,zhou2024can,jiang2024enhancing,wu2024instructing}, (3) increasing prompt diversity enhances expressivity \citep{naik2023diversity,yao2023react}, and (4) multi-agent prompting further reduces approximation error \citep{zhou2022least,jiang2023llm,wu2023autogen,wang2024mixture}. By formalizing these effects within an approximation-theoretic framework, our work bridges the gap between empirical prompt engineering practices and rigorous mathematical theory, offering a principled foundation for optimizing prompt design.

\paragraph{Our Contributions}  

We summarize our contributions as follows:
\begin{enumerate}
    \item We develop a novel theoretical framework that formalizes the role of prompts in transformer models. Specifically, we prove that a \emph{small} transformer can leverage prompt information to configure itself dynamically during inference, effectively forming a ``virtual'' neural network. This insight establishes a direct connection between the expressivity of deep neural networks and prompt-based architectures.
    \item We present the first approximation theory for transformers augmented with prompts for functions in the $\mathcal{C}^\beta$ class. We derive a bound showing that a token length of $O(\varepsilon^{-p/(2\beta)})$ is sufficient to achieve an $\varepsilon$-approximation. Furthermore, we demonstrate that replacing the ReLU activation with the Elementary Universal Activation Function (EUAF) \citep{zhang2022deep} in the first transformer layer enables transformers to approximate arbitrary smooth functions with a finite number of prompt tokens.
    \item Our framework provides a formal foundation for several empirically observed prompt engineering effects, including the impact of prompt length, noise reduction through token filtering, enhanced expressivity from diverse prompts, and the approximation benefits of multi-agent prompting. By grounding these phenomena in a rigorous theoretical context, our work bridges the gap between empirical findings and fundamental approximation theory. 
\end{enumerate}

\paragraph{Organization of Paper}

This paper is structured as follows. In Section~\ref{sec: transformer}, we provide a formal description of transformers and their core components. Section~\ref{sec: emulate} introduces our theoretical results, demonstrating how transformers can emulate neural networks using prompts. In Section~\ref{sec: approximation}, we analyze the capacity of transformers with prompts to approximate smooth functions in the $\mathcal{C}^\beta$ class. In Section~\ref{sec: application}, we provide insights into practical prompt engineering techniques using our theoretical framework. Finally, we discuss the implications of our findings and potential future research in Section~\ref{sec: discussion}.

\subsection{Related Works}

Transformers have exhibited remarkable few‐shot learning capabilities \citep{brown2020language,liu2021makes,zhao2021calibrate,schick2020s,wei2022chain,kojima2022large,zhou2022least}, which have sparked extensive research into prompt engineering. This approach leverages carefully designed prompts to steer model behavior across diverse tasks, eliminating the need for model fine-tuning \citep{qiao2022reasoning,liu2023pre}.

One area of research focuses on automatic prompt tuning, which aims to automate the discovery of task‐specific prompts \citep{shin2020autoprompt,prasad2022grips,diao2022black,zhang2022tempera,pryzant2023automatic,zhou2023survival}. For instance, AutoPrompt \citep{shin2020autoprompt} utilizes a gradient‐based search to generate discrete prompts that effectively elicit desired responses. 
At the same time, continuous prompt tuning techniques have gained traction. These methods introduce learnable vectors (or ``prefixes'') that are prepended to inputs, allowing the model's behavior to be guided without modifying its parameters \citep{li2021prefix,liu2021p,lester2021power,sun2022black,sun2022bbtv2,wang2023review}.
Prefix-Tuning \citep{li2021prefix} pioneered this approach by optimizing a sequence of continuous prompts, demonstrating that a small set of tunable parameters can achieve comparable performance to model fine-tuning. Building on this foundation, subsequent research has further refined and expanded this paradigm. \citet{lester2021power} explored the scalability of prompt tuning, while P-Tuning v2 \citep{liu2021p} demonstrated that prompt tuning could match the performance of full fine-tuning across a variety of tasks and model scales.

From the theoretical point of view, the general ability of transformers to approximate complex functions has been extensively investigated \citep{yun2019transformers,jiang2023approximation,takakura2023approximation}. Moreover, there have been some recent theoretical developments on the impact of prompts on the performance of transformers. We next discuss related existing theoretical work on prompt engineering and ICL. 

\paragraph{Prompt Engineering Theory}

\citet{wang2024universality} and \citet{hu2024fundamental} examined the approximation capacity of transformers when augmented with prompts.
They proved that for any specified error tolerance, one can construct a transformer and determine a corresponding prompt length such that prompt tuning can approximate any Lipschitz sequence-to-sequence function to the desired accuracy. In particular, \citet{hu2024fundamental} showed that even a minimal transformer architecture, consisting of a single self-attention layer and just two feed-forward layers, can serve as a universal approximator. They also derived a lower bound on the soft prompt length required for memorization.
Moreover, \citet{oymak2023role} showed that prompt tuning markedly enhances a model's focus on task-relevant inputs. Their analysis revealed that a few gradient descent steps allow the model to concentrate on these critical tokens, achieving near-optimal performance with exponentially decreasing error rates. 
In addition, \citet{petrov2024prompting} investigated the universal approximation capacity facilitated by prompts. Specifically, they demonstrated that for a target Lipschitz sequence-to-sequence function, an $\varepsilon$-approximation can be attained using prompts of length $O(\varepsilon^{-O(p^2)})$, where $p$ denotes the dimensionality of the function's domain. 
Under a hidden Markov model data-generating process, \citet{wei2021pretrained} demonstrated that soft prompt tuning combined with a linear head effectively recovers the latent information essential for downstream classification tasks. Their results indicate that this soft prompt approach imposes milder non-degeneracy conditions than those required by conventional head fine-tuning.
Beyond these works, we further exploit the sequential nature of transformers to demonstrate their capacity to emulate deep neural networks, which enables us to derive approximation bounds for smooth functions. Specifically, in Section~\ref{sec: approximation} we provide the approximation error for $\mathcal{C}^\beta$ functions, showing that an $\varepsilon$-approximation can be achieved with a prompt length of $O(\varepsilon^{-p/(2\beta)})$.

\paragraph{In-context Learning Theory}

Another closely related line of research in prompt engineering is ICL, a paradigm in which multiple examples are provided within the input to guide model behavior without modifying the model weights. Theoretical investigations have revealed that ICL implicitly performs Bayesian inference \citep{xie2021explanation,zhang2023and,mahankali2023one,hu2024unveiling}. Moreover, transformers are capable of executing implicit gradient descent updates on the provided examples \citep{garg2022can,akyurek2022learning,von2022transformers,dai2022can,cheng2023transformers,ahn2023transformers}. They can also effectively perform tasks such as ridge regression, least squares \citep{akyurek2022learning,guo2023transformers}, algorithm selection \citep{bai2024transformers}, kernel regression \citep{han2023context}, and unsupervised domain adaptation \citep{hataya2024automatic}. In addition, \citet{kim2024transformers} demonstrated that pre-trained transformers can achieve the minimax optimal rate of convergence for ICL by pre-training on a diverse set of nonparametric regression tasks, thereby learning to approximate regression functions in Besov spaces via neural network representations. \citet{guo2023transformers} investigated linear regression with features generated by multi-layer neural networks, revealing that transformers can memorize network weights and effectively run gradient descent on in-context ridge regression. Recently, \citet{wu2024transformers} proved that transformer models can, through ICL, effectively simulate the training process of deep feed-forward neural networks via gradient descent.
Unlike ICL theory, which is inherently example-driven and relies on in-context examples to drive model adaptation, our work provides a result in terms of the prompt length required to describe the model adaptation.

Regarding transformers' ability to process sequential problems, recent research employs circuit complexity to analyze their expressiveness -- their capacity to represent and compute complex functions effectively. \citet{feng2023towards} and \citet{li2024chain} demonstrate that when transformers are allowed to generate intermediate reasoning steps as Chain-of-Thought (CoT), they effectively increase their computational depth, enabling them to tackle complex, inherently sequential tasks. In contrast, constant-depth transformers operating with finite precision remain confined to relatively simple computational classes. Furthermore, \citet{chen2024theoretical} established a lower bound on the transformer size required for multi-step sequential function composition, implying that CoT prompting can significantly reduce parameter requirements by distributing computation across multiple tokens. See \citet{ji2025overview} for an overview of the theoretical understanding of LLMs.

\subsection{Notation}

For any two sequences of positive numbers $\{a_k\}_k$ and $\{b_k\}_k$, 
indexed by $k \in \mathcal{K}$, we write $a_k \lesssim b_k$ if there exists a constant $C > 0$, independent of $k$, such that $\sup_{k \in \mathcal{K}} a_k / b_k < C$ holds. 
We denote $a_k = O(b_k)$ when $a_k \lesssim b_k$, and similarly, $a_k = \Omega(b_k)$ if $a_k \gtrsim b_k$ (i.e., if $b_k \lesssim a_k$). Additionally, $a_k = \Theta(b_k)$ if $a_k = O(b_k)$ and $a_k = \Omega(b_k)$ hold.
For any matrix $A$, we denote its operator norm by $\|A\|$.
For any vector $\a = (a_1, \dots, a_D)^\top \in \R^D$, define the max norm as $\|\a\|_{\max} = \max_{j \in [D]} |a_j|$.
For any positive integer $I$, define $[I]$ as the set $\{1,2,\cdots,I\}$. By convention, we set $[0] = \emptyset$.
We use $a \vee b$ and $a \wedge b$ to denote $\max(a, b)$ and $\min(a, b)$, respectively.
For any vector $\a = (a_1, \dots, a_D)^\top \in \R^D$, denote its $j$-th entry by $(\a)_j := a_j$. For indices $t_1 \leq t_2 \in \N$, define the sub-vector $(\a)_{t_1:t_2} := (a_{t_1}, a_{t_1+1}, \dots, a_{t_2})^\top \in \R^{t_2 - t_1+1}$.
For a matrix $A$ expressed as $[\a_1, \dots, \a_D]$, we write $(A)_j := \a_j$.
We denote by $\mathcal{B}_D(R)$ the ball of radius $R > 0$ in $\R^D$ centered at $\zero_D$.
Let $\Id\{\text{(condition)}\}$ denote the indicator function, which takes the value $1$ if the condition is satisfied and $0$ otherwise. 
We use the notation $\mathcal{C}^\beta([0,1]^p)$ (with $\beta \in \N$) to denote the class of functions with domain $[0, 1]^p$ that are $\beta$-times continuously differentiable.
For any $f \in \mathcal{C}^\beta([0, 1]^p)$, define its norm by $\|f\|_{\mathcal{C}^\beta([0, 1]^p)} = \max\{\|\partial^{\bm \alpha} f\|_{\mathcal{L}^\infty([0, 1]^p)} : \|\bm \alpha\|_1 \leq \beta, \bm \alpha \in (\{0\}\cup\N)^p\}$. 
Furthermore, denote by $\mathcal{C}_\u^\beta([0, 1]^p) = \{f \in \mathcal{C}^\beta([0, 1]^p): \|f\|_{\mathcal{C}^\beta([0, 1]^p)} \leq 1\}$ the unit ball in $\mathcal{C}^\beta([0, 1]^p)$. Let $\mathcal{C}([0,1]^p)$ denote the set of continuous functions on $[0, 1]^p$.
For convenience, the notations introduced in later sections are summarized in Table~\ref{tab: notations} in Appendix.

\section{Preliminaries}\label{sec: transformer}

In this section we introduce the basic setups for our analysis.

\subsection{Transformers}

We follow the notation in \citet{akyurek2022learning,von2022transformers,bai2024transformers} to introduce transformers, which consist of two main types of layers: the self-attention layer and the feed-forward transformation layer. 

\paragraph{Self-attention.} 
The core of transformers is the self-attention mechanism, which determines how much focus each token should have on others in a sequence. Each input token is projected by three sets of matrices: Query ($Q$) representing the token we are currently processing, Key ($K$) representing all tokens in the sequence, and Value ($V$) containing the actual information to be aggregated. Given the token matrix for any sequence of length $n \in \N$ and dimension $D$, $H = [\h_1, \dots, \h_n] \in \R^{D\times n}$, the self-attention layer with $M$ heads and $3M$ projection matrices, $\mu=\{(Q_m, K_m, V_m)\}_{m \in [M]}$, (each $Q_m$, $K_m$, and $V_m$ are $D \times D$ matrices) takes $H$ as input and outputs
$$
\Attn_\mu(H)_j := \h_j + \sum_{m\in[M]} \sum_{j' \in [n]} \sigma(\langle Q_m \h_j, K_m \h_{j'} \rangle) V_m \h_{j'}, \ \ j \in [n],
$$
where $\sigma: \R \rightarrow \R$ is the activation function. The attention layer computes a weighted average of transformed token representations, with the attention weights determined by the similarity between query and key projections.

\paragraph{Feed-forward transformation.} 
Every Transformer layer incorporates a Feed-forward Network (FFN) that processes tokens independently. The FFN applies a non-linear, token-wise transformation that complements the global interactions captured by self-attention.
Let the token embedding dimension $D$ and the hidden dimension $D' \in \N$ be given. 
Given a matrix $H\in\R^{D\times n}$ with any sequence length $n \in \N$, the feed-forward transformation layer with parameters $\nu=(W_1,W_2) \in \R^{D' \times D} \times \R^{D \times D'}$, takes $H$ as input and outputs
$$
\FFN_\nu(H): = H+W_2 \sigma(W_1 H).
$$
where $\sigma(W_1 H)$ is obtained by applying the activation function $\sigma : \R \to \R$ elementwise to the matrix $W_1 H$. We refer to $D'$ as the FFN width. 

Unless specified, we consider a transformer with the ReLU activation function, i.e., $\sigma(x) := \max\{0, x\}$. We denote a transformer layer by $\TF_\theta(H) = \FFN_\nu \circ \Attn_\mu(H)$, where $\theta = (\mu, \nu)$. With a slight abuse of notation, we write a multi-layer transformer with parameters $\Theta = (\theta_1, \dots, \theta_l)$ as $\TF_{\Theta} = \TF_{\theta_l} \circ \dots \circ \TF_{\theta_1}$.

\paragraph{Sequential Generation.}

Given input tokens $H = [\h_1, \dots, \h_n]$, a transformer iteratively generates a sequence of tokens $\h_{n+1}, \h_{n+2}, \dots$ by repeatedly performing next-token prediction.
Formally, for any transformer $\TF_\Theta$ and each $v \in \{0\} \cup \N$, the subsequent token $\h_{n+v+1}$ given all previous tokens $\h_1, \dots, \h_{n+v}$ is generated by $\h_{n+v+1} = \TF_\Theta([\h_1, \dots, \h_{n+v}])_{n+v}$, i.e., by selecting the last token from the output of transformer layers. The detailed transformer algorithm is summarized in Algorithm~\ref{alg: iterative generation}.
Note that this formulation differs from that of standard language models, in which a new token is sampled from a categorical distribution after applying a softmax to the linear transformation of the transformer's output. This difference arises because we focus on approximating smooth functions in Euclidean space rather than modeling discrete words in text. We impose this simplification to ensure the mathematical tractability of the problem.

\begin{algorithm}[H]
    \caption{Simplified Iterative Generation by Transformer}\label{alg: iterative generation}
    \begin{algorithmic}[1]
    \Require Transformer $\TF_{\Theta}$, initial sequence $H = [\h_1, \h_2, \dots, \h_n] \in \mathbb{R}^{D \times n}$,
    and maximum generation steps $K$.
    \Ensure Generated sequence $[\h_{n+1},\h_{n+2}, \dots, \h_{n+K}]$.
    \State Initialize $H_0 \gets H$.
    \For{$v \gets 0$ \textbf{to} $K-1$}
        \State Compute $\tilde H_v \gets \TF_{\Theta}(H_v)$.
        \State Set $\h_{n+v+1} \gets (\tilde H_v)_{n+v}$, i.e., the last column of $\tilde H_v$.
        \State Update $H_{v+1} \gets [H_v, \h_{n+v+1}]$.
    \EndFor
    \State \Return $[\h_{n+1},\h_{n+2}, \dots, \h_{n+K}]$.
    \end{algorithmic}
\end{algorithm}

\subsection{Prompts}\label{sec: prompts setup}

Here we define the class of continuous, or ``soft'', prompts.
For simplicity, we adopt a composite positional encoding instead of the typical additive scheme, as used in recent ICL literature \citep{guo2023transformers,bai2024transformers,hu2024unveiling,nakada2024synthetic}. Notably, the T5 model \citep{raffel2020exploring} employs an encoding method that effectively adds a scalar to the attention logits, which is similar in spirit to using a separate dimension for positional encoding. Since the positional encoding is separated from the word embedding, the token embedding dimension (i.e., the dimension of $\h_j$) differs from that of the word embeddings. 

Let $\mathcal{U} \subset \R^d$ denote the set of word embeddings, where $d$ is the embedding dimension.\footnote{We use the term `tokens' to refer to the $D$-dimensional direct inputs to the transformer, and we call the $d$-dimensional part of these tokens the (word) embeddings.}
We assume that the initial input tokens are divided into $T$ prompt tokens and $N$ data tokens: 
\begin{align*}
    H & = [\underbrace{\h_1, \h_2, \dots, \h_T}_{\text{prompt tokens}}, \underbrace{\h_{T+1}, \h_{T+2}, \dots, \h_{T+N}}_{\text{data tokens}}] \in \R^{D \times (T+N)}.
\end{align*}

Data tokens correspond to the variables to which the procedures described in the prompt are applied. For example, we assume the following type of input:
``Add 1 to the following numbers: 1, 2, 3, 4, 5'', where the tokens corresponding to ``Add 1 to the following numbers:'' are the prompt tokens, and ``1, 2, 3, 4, 5'' correspond to the data tokens.

Specifically, with the choice of token dimension $D = 4d+8$, we consider the following prompt tokens with their associated embeddings:
\begin{align}
    \h_j := (\underbrace{\bu_j^\top}_{\text{$j$-th word embedding}}, \underbrace{\p_j^\top}_{\text{$j$-th positional encoding}})^\top \in \R^{4d+8} \ \ \text{ for $j \in [T]$},\label{def: prompt token}
\end{align}
where $\bu_j \in \mathcal{U} \subset \R^d$ is the word embedding, and $\p_j \in \R^{3d+8}$ is the positional encoding defined as $\p_j = p(w_j, j, S) := (\zero_{3d+4}^\top, 1, S, S w_j, S j)^\top$.
The scaling parameter $S>0$ ensures that the positional encoding remains significant relative to the word embeddings. Increasing this parameter enables transformers to correctly model global interactions even with potentially long input sequences.
The constant $1$ serves as a bias term for the attention layers and FFNs, while the zeros in $\p_j$ function as a temporal memory to store computed variables in each transformer layer. 
We later establish the existence of a transformer that processes a prompt and generates a value computed by a neural network constructed from the prompt.
For such a transformer, each word embedding corresponds to a specific (scaled) singular vector in the model weights of a virtual network, and the integer $w_j \in \Z$ encodes the layer of the virtual network to which the $j$-th word belongs. More precisely, the word embedding $\bu_j$ corresponds to one of the left or the right (scaled) singular vectors of the $\ell$-th layer weight if $w_j = 2\ell-1$ or $w_j = 2\ell$, respectively.
At a high level, $w_j$ imposes a hierarchical structure on the prompt. From ICL perspective, it generalizes the notion of a sample index by indicating the ``layer'' or processing stage a token contributes to. Similarly, in CoT prompting, $w_j$ marks sequential reasoning steps, guiding the layer-by-layer computation in the emulated neural network.

Denote the class of meaningful prompts defined in \eqref{def: prompt token} with \textit{exactly} $T \in \N$ tokens, a virtual layer index of at most $L$, 
scale $S > 0$, and bounded $(w_j)_{j \in [T]}$, by $\mathcal{P}_{\mathcal{U}}(T,L,S) \subset \R^{(4d+8) \times T}$, defined by
\begin{align*}
    \mathcal{P}_{\mathcal{U}}(T,L,S) = \{[\h_1, \dots, \h_T]: \h_j = (\bu_j^\top, p(w_j, j, S)^\top)^\top,\ w_j \in [2L],\ \bu_j \in \mathcal{U}\}.
\end{align*}
We also denote the class of prompts of length \textit{at most} $T$ by 
\begin{align*}
    \mathcal{Q}_{\mathcal{U}}(T,L,S) := \bigcup_{T' \in [T]} \mathcal{P}_{\mathcal{U}}(T',L,S).
\end{align*}

Given data $(\z_i)_{i \in [N]} \subset \R^d$, we define data tokens $(\h_{T+i})_{i \in [N]}$ following a prompt with length $T$ and scale $S$ as
\begin{align}
    \h_{T+i} = h^\dat(\z_i, T+i, S) := (\z_i^\top, p(0, T+i, S)^\top)^\top \in \R^{4d+8}.\label{eq: h data}
\end{align}
Table~\ref{tab: input tokens} summarizes the input, which consists of embeddings of the $T$ prompt tokens and $N$ data tokens.

\renewcommand{\arraystretch}{1.1}
\begin{table}[H]
    \centering
    \begin{tabular}{|c||c|c|c||c|c|c|}
    \hline
    \multicolumn{1}{|c||}{} & \multicolumn{3}{c||}{Prompt tokens} & \multicolumn{3}{c|}{Data tokens} \\ \hline
    Indices & $\h_1$ & \dots & $\h_T$ & $\h_{T+1}$ & \dots & $\h_{T+N}$\\ \hline\hline
    $[d]$ & $\bu_1$ & \dots & $\bu_T$ & $\z_1$ & \dots & $\z_N$\\ \hline
    $[4d+8]\setminus[d]$ & $p(w_1, 1, S)$ & \dots & $p(w_T,T,S)$ & $p(0,T+1,S)$ & \dots & $p(0,T+N,S)$ \\ \hline
    \end{tabular}
    \caption{A sketch of the input consisting of $T$ prompt tokens and $N$ data tokens.}
    \label{tab: input tokens}
\end{table}

\begin{rem}
    Although the notion of $w_j$ is ambiguous, we can understand its role by considering a prompt consisting of multiple sentences conveying different meanings. 
    For such a prompt, altering the length of the first sentence (while preserving its meaning) should not affect the interpretation of a word in the subsequent sentences. 
    In other words, $w_j$ captures the position of the $j$-th word in the overall high-level structure of the prompt.
    From a technical perspective, as seen in Section~\ref{sec: emulate}, $w_j$ helps specify the virtual `layer' to which the $j$-th word belongs.
    Alternatively, one may treat $w_j$ as a computed value from previous layers of a transformer by considering deeper transformer structure. Nevertheless, since our primary focus is on the analysis of the capacity of transformers with prompts, we do not address the determination of $w_j$ in this work.
\end{rem}

\section{Prompting Transformers to Emulate Neural Networks}\label{sec: emulate}

In this section, we demonstrate the expressive power of transformers when guided by carefully designed prompts. We show that transformers can effectively construct a \emph{virtual} neural network---a network architecture dynamically instantiated with weights encoded directly in the prompt.
In essence, Theorem 3.1 establishes that for any specific neural network, there exists a prompt that causes a transformer to sequentially generate tokens equivalent to the neural network's output.
The emulated virtual neural network has weights determined by the word embeddings of the prompt tokens, and its network architecture is specified via positional encodings.

To this end, we define the class of \emph{coarse} weight matrices, whose rank-$1$ factors are drawn from any elements in $\mathcal{U}$:
\begin{align}
    \mathcal{W}(r, d, \mathcal{U}, B) = \biggl\{W \in \R^{d \times d}: W = \sum_{k \in[r]} \tilde \bu_k \bu_k^\top, \ \ \|W\| \leq B^2, \ \ \tilde \bu_k, \bu_k \in\mathcal{U} \biggr\}.\label{eq: coarse NN class}
\end{align}
This class captures the coarse structure of weight matrices realizable via prompt-based emulation.
Note that when $\mathcal{U} = \R^d$, $\mathcal{W}(r, d, \mathcal{U}, B)$ exactly coincides with the set of rank-$r$ matrices in $\R^{d\times d}$ with operator norm bounded by $B^2$. 

We then define the class of coarse neural networks with ReLU activation function from $\R^d$ to $\R^d$ by
\begin{align*}
    \mathcal{G}(\mathbf{r}, d, \mathcal{U}, B) = \qty{g: g(\z) = W_L \sigma( W_{L-1} \sigma ( \dots \sigma (W_1 \z) \dots )), W_\ell \in \mathcal{W}(r_\ell,d,\mathcal{U},B)}
\end{align*}
where $\mathbf{r} = (r_1, r_2, \dots, r_L)$ denotes the ranks of weight matrices.

To investigate the capacity of emulation, we define the following transformer emulator.
\begin{dfn}\label{def: virtual NN}
    Let $d,T,N,L \in \N$ and $S > 0$.
    For any $\z_1, \z_2, \dots, \z_N\in[0,1]^d$, define the corresponding $N$ data tokens by
    $H^\dat = [h^\dat(\z_1, T+1, S), \dots, h^\dat(\z_N, T+N, S)] \in \R^{(4d+8)\times N}$.
    Given a prompt $H^\pro \in \mathcal{P}_{\mathcal{U}}(T,L,S)$ and data tokens $H^\dat$, the transformer $\TF_\Theta$, provided with the initial input sequence $H = [H^\pro, H^\dat] \in \R^{(4d+8)\times (T+N)}$, sequentially generates tokens $\h_{T+N+1}, \h_{T+N+2}, \dots$ according to Algorithm~\ref{alg: iterative generation} with $K\geq NL$. 
    The transformer emulation from $\R^d$ to $\R^d$ is then defined as
    $$
        \hat g_{\Theta,H^\pro,N,L,d}(\z_i) := (\h_{T+NL+i})_{1:d} \in \R^d.
    $$
\end{dfn}

We show that there exists a \emph{fixed} transformer that can generate arbitrary coarse neural networks given an appropriate prompt. 
Notably, this model dynamically constructs the equivalent of a virtual neural network using the word embeddings provided in the prompt, without relying on a specific embedding space $\mathcal{U}$.
\begin{thm}\label{thm: prompt engineering NN}
    Fix any $d \in \N$. There exists a $7$-layer transformer parameterized by $\Theta^*$ such that for any $B \geq 1$, $L \in \N$, $\mathbf{r} \in (\N \cup \{0\})^L$, $\mathcal{U} \subset \mathcal{B}_d(B)$, and $g \in \mathcal{G}(\mathbf{r}, d, \mathcal{U}, B)$, there exists a prompt $H^\pro \in \mathcal{P}_{\mathcal{U}}(T,L,S)$ with $T = 2\sum_{\ell \in [L]} r_\ell$ and $S \geq d T B^{4L}$ so that for any $N \in \N$, $(\z_i)_{i \in [N]} \subset [0, 1]^d$, 
    \begin{align*}
        \hat g_{\Theta^*,H^\pro,N,L,d}(\z_i) = g(\z_i) \ \ \text{ for $i \in [N]$.}
    \end{align*}
\end{thm}
Theorem~\ref{thm: prompt engineering NN} states that any coarse neural network can be precisely emulated by a specific transformer $\TF_{\Theta^*}$ with an appropriate prompt.
The number of prompt tokens required to approximate the coarse neural networks is bounded by the sum of the ranks of the weight matrices.
The prescribed scale $S$ in the positional encoding ensures that $\TF_{\Theta^*}$ can process each prompt token with embeddings in $\mathcal{B}_d(B)$ and construct a virtual neural network whose weights lie within $\mathcal{W}(r_1, d, \mathcal{U}, B) \times \dots \times \mathcal{W}(r_L, d, \mathcal{U}, B)$.

The result also highlights the importance of allowing transformers to generate longer responses: every $N$ generated tokens $\h_{T+NL+1}, \h_{T+NL+2}, \dots, \h_{T+NL+N}$ correspond to the simultaneous outputs of the $L$ layer neural network $g(\z_1), g(\z_2), \dots, g(\z_N)$, thus, generating longer responses yields outputs corresponding to deeper neural networks with enhanced approximation capacity.
Moreover, the transformer $\TF_{\Theta^*}$ is independent of the parameter norm bound $B$, the number of data tokens $N$, and the depth $L$ of the emulated neural network.
This result implies that a simple transformer architecture can precisely emulate arbitrary coarse neural networks when provided with appropriately designed prompts.
The detailed property of the transformer $\TF_{\Theta^*}$ along with the proof of Theorem~\ref{thm: prompt engineering NN} is deferred to Appendix~\ref{sec: emulate ap}.

\begin{proof}[Proof Sketch]
    Observe that the weight matrix in the $\ell$-th layer, $W_\ell = \sum_{k \in[r_\ell]} \tilde \bu_k^{(\ell)} \bu_k^{(\ell)}$, decomposes into $2 r_\ell$ rank-$1$ factor vectors, namely the collections $(\bu_k^{(\ell)})_{k \in[r_\ell]}$ and $(\tilde \bu_k^{(\ell)})_{k \in[r_\ell]}$. 
    Our construction employs a transformer that allocates a total of $2\sum_{\ell \in [L]} r_\ell$ rank-$1$ factors across the corresponding virtual layers. To accomplish this, we design the prompt tokens in the following form:
    \begin{align*}
        \h_{2s-1} = \begin{pmatrix}
            \tilde \bu_{\kappa(s)}^{(\pi(s))}\\
            \p_{2s-1}
        \end{pmatrix}, \ \ 
        \h_{2s} = \begin{pmatrix}
            \bu_{\kappa(s)}^{(\pi(s))}\\
            \p_{2s}
        \end{pmatrix} \ \ \text{ for $s \in [T/2]$},
    \end{align*}
    where $\p_j = p(w_j, j, S)$ with $w_{2s-1} = 2\pi(s)-1$ and $w_{2s} = 2\pi(s)$ for each $s \in [T/2]$. 
    Here, $\pi(s)$ designates the virtual layer associated with the tokens $\h_{2s-1}$ and $\h_{2s}$, while $\kappa(s) \in [r_{\pi(s)}]$ indexes the rank-$1$ factors within that layer.
    
    For simplicity, consider the case of initial word prediction with a single data token. Given the input $[\h_1, \dots, \h_T, \h_{T+1}]$, where the data token $\h_{T+1}$ encodes the datum $\z$, the transformer first computes the inner products $\bu_{\kappa(s)}^{(\pi(s)) \top} \z$ for all $s \in [T/2]$, and subsequently evaluates $\tilde \bu_{\kappa(s)}^{(\pi(s))} \bu_{\kappa(s)}^{(\pi(s)) \top} \z$ for each $s \in [T/2]$.
    Aggregating these contributions over all indices $s$ with $\pi(s) = 1$ yields $\sum_{s \in [T/2]} \tilde \bu_{\kappa(s)}^{(\pi(s))} \bu_{\kappa(s)}^{(\pi(s)) \top} \z \Id\{\pi(s)=1\} = W_1 \z$.

    To execute these inner product and scalar-vector multiplication operations exactly, we leverage the capability of a single transformer layer to implement sequence-to-sequence functions of the form
    \begin{align*}
        F([\h_1, \dots, \h_n])_j = \h_j + \sum_{k \in [n]} (\a_1^\top \h_j+\a_2^\top \h_k) \Id\{\a_3^\top \h_j = \a_4^\top \h_k\} V \h_k
    \end{align*}
    for each $j \in [n]$, where $\a_1, \dots, \a_4$ are fixed constant vectors and $V$ is a matrix. 
    By appropriately choosing $\a_1,\dots,\a_4$ and $V$, and by stacking multiple layers as needed, one can implement the above procedures exactly.
\end{proof}

\section{Approximation Bounds of Prompts for Smooth Functions}\label{sec: approximation}

In Section~\ref{sec: emulate}, we showed that a transformer, given carefully designed prompts, can approximate fixed-width neural networks with arbitrary precision. Building on this result and existing literature on neural network expressivity, we establish the capacity of prompt engineering to approximate a class of smooth functions.
Specifically, given a set of word embeddings $\mathcal{U} \subset \R^d$, we consider approximating a function $f: \R^p \to \R$ using the transformer-generated tokens. We start with prompt tokens $H^\pro \in \mathcal{P}_{\mathcal{U}}(T,L,S)$ and a single data token $\h_{T+1} = h^\dat([\x^\top,1,\zero_{d-p-1}^\top]^\top, T+1, S)$ for some ``input'' to the virtual neural network $\x \in \R^p$. 
For the transformer $\TF_{\Theta^*}$ (as introduced in Theorem~\ref{thm: prompt engineering NN}), we select the $L$-th generated token, and interpret its first coordinate as an approximation to $f(\x)$. 
In this scenario, the virtual network depth $L$ is the same as the inference depth, or the number of generated tokens.
While we focus on scalar-valued functions for simplicity, the framework easily extends to vector-valued functions.

\begin{dfn}[Function Approximator via Transformers]\label{def: function approx}
    Let $d,T,L \in \N$, $S > 0$, $\mathcal{U} \subset \R^d$, and $p \in [d-1]$ be given.
    For any input $\x\in\R^p$, define the corresponding data embedding $\z = [\x^\top, 1, \zero_{d-p-1}^\top]^\top$.
    Given a prompt $H^\pro\in \mathcal{P}_{\mathcal{U}}(T,L,S)$,
    the function approximator from $\R^p$ to $\R$ is defined as the first coordinate of $\hat g_{\Theta,H^\pro,1,L,d}(\z)$, i.e.,
    $$
        \hat f_{\Theta,H^\pro,L,p}(\x) := \hat g_{\Theta,H^\pro,1,L,1}(\z).
    $$
\end{dfn}

We first provide the approximation error for smooth functions when $\mathcal{U} = \R^d$. Using the approximation error results from \citet{lu2021deep}, we derive the following approximation capacity for transformers with prompts.
\begin{cor}\label{cor: approximation error} 
    Fix $p \in \N$, $\beta \in \N$ and $\varepsilon \in (0, 1/2)$.
    Set $d \geq 1+17\beta^{p+1} 3^p p$. 
    Then, there exists a constant $C = C(\beta,p,d) > 0$ such that
    \begin{align*}
        \sup_{f \in \mathcal{C}_\u^\beta([0, 1]^p)} \inf_{\substack{S>0\\H^\pro \in \mathcal{Q}_{\R^d}(T,L,S)}} \|\hat f_{\Theta^*,H^\pro,L,p} - f\|_{\mathcal{L}^\infty([0, 1]^p)} \leq C \varepsilon
    \end{align*}
    holds for $L = \tilde O(\varepsilon^{-p/(2\beta)})$ and $T = O(d L)$. 
\end{cor}
Corollary~\ref{cor: approximation error} states that a prompt of length $T = \tilde O(\varepsilon^{-p/(2\beta)})$ is sufficient to approximate $\beta$-times differentiable functions within arbitrary precision $\varepsilon$, provided that the embedding dimension $d$ and inference depth $L$ are sufficiently large.
Recall that the choice of parameters $\Theta^*$ depend solely on $d$, and are independent of $\beta$, $p$ and $\epsilon$. 
Moreover, the approximation error for smooth functions with a prompt of length $T$ is $\tilde O(T^{-2\beta/p})$, which can be contrasted with the optimal approximation rate of $\tilde O(W^{-2\beta/p})$ \citep{lu2021deep} for ReLU neural networks, where $W$ represents the number of non-zero parameters in the neural network. The proof of Corollary~\ref{cor: approximation error} is deferred to Appendix~\ref{sec: approximation ap}.
We also discuss the optimality of inference depth $L$ to achieve $\varepsilon$ error in Appendix~\ref{sec: additional theory lb ap}.

\subsection{Universal Approximation with Short Prompts}\label{sec: approximation EUAF}

We now demonstrate that substituting the activation function in the initial feed-forward layer can significantly reduce the upper bound on the number of tokens needed to approximate continuous functions.
This result builds on the findings of \citet{zhang2022deep}, who established that any continuous function $\R^p \to \R$ can be approximated to arbitrary precision using neural networks with a fixed depth and $O(p^2)$ width. The key insight from their work is the design of an activation function that encapsulates both non-linearity and periodicity, enabling more efficient function approximation. Specifically, they introduce the Elementary Universal Activation Function (EUAF), defined as
\begin{align*}
    \sigma_\EUAF(x) := \sigma_1(x) \Id\{x \geq 0\} + \sigma_2(x) \Id\{x < 0\},    
\end{align*}
where $\sigma_1$ and $\sigma_2$ are, respectively, a sawtooth function and a soft sine function, given by
\begin{align*}
    \sigma_1(x) = |x - 2\floor{(x+1)/2}|, \ \ \sigma_2(x) = \frac{x}{1 + |x|}.
\end{align*}
We then consider a transformer in which the activation function in the first feed-forward transformation is replaced by EUAF. Specifically, consider a transformer defined as
\begin{align}
    \FFN_{\nu_L} \circ \Attn_{\mu_L} \circ \dots \circ \FFN_{\nu_2} \circ \Attn_{\mu_2} \circ \FFN_{\EUAF,\nu_1} \circ \Attn_{\mu_1},\label{eq: EUAF transformer}
\end{align}
where $\FFN_{\EUAF,\nu_1}$ employs the EUAF activation function $\sigma_\EUAF$ instead of the standard ReLU activation. 
For brevity, we denote by $\TF_{\Theta}$ the transformer with the structure \eqref{eq: EUAF transformer} parameterized by $\Theta$ throughout this subsection.
\begin{cor}\label{cor: approximation error EUAF}
    Fix $p \in \N$ and set $d = 36p(2p+1)$.
    Then, there exists an $8$-layer transformer $\TF_{\Theta^\#}$, where the activation function in the first transformer layer is replaced by $\sigma_\EUAF$, with at most $8$ heads and FFN width of $O(d)$, such that
    \begin{align*}
        \sup_{f \in \mathcal{C}([0, 1]^p)} \inf_{\substack{S>0\\H^\pro \in \mathcal{Q}_{\R^d}(T,L,S)}} \|\hat f_{\Theta^\#,H^\pro,L,p} - f\|_{\mathcal{L}^\infty([0, 1]^p)} \leq \varepsilon
    \end{align*}
    with $T = 24d$ and $L = 12$.
\end{cor}
Corollary~\ref{cor: approximation error EUAF} demonstrates that only a fixed number of prompt tokens, namely $T = 24d$, are needed to approximate continuous functions with arbitrary precision.
This represents a significant improvement over the upper bound on prompt length from Corollary~\ref{cor: approximation error}, which requires $T = \tilde O(\varepsilon^{-p/(2\beta)})$ even for smoother functions compared to continuous functions.
Furthermore, only an $O(1)$ inference depth is required to achieve arbitrary precision approximation.

This corollary follows from Theorem~\ref{thm: prompt engineering EUAF NN ap}, which establishes the expressivity of transformers when the first feed-forward neural network is replaced by an EUAF-based architecture.
The proof utilizes the fact that an EUAF neural network 
$$
    \x \mapsto W_L \sigma_\EUAF(W_{L-1} \sigma_\EUAF( \dots \sigma_\EUAF ( W_1 \x) \dots))
$$ 
exhibits rich expressivity with fixed width and depth.

Moreover, by suitably modifying our proof, one can obtain an analogous result for any Lipschitz activation function, provided that the activation function in the first transformer layer is replaced appropriately.

\section{A Unified Framework for Empirical Insights in Prompt Engineering}\label{sec: application}

In this section, we apply our theoretical framework and the transformer model $\TF_{\Theta^*}$ from Theorem~\ref{thm: prompt engineering NN} to provide a theoretical foundation for four effective prompt engineering techniques. Specifically, we analyze how (i) extending prompt length and detail enhances the expressivity, (ii) filtering out irrelevant tokens reduces approximation errors, (iii) increasing prompt diversity improves the expressivity, and (iv) multi-agent collaboration refines task decomposition. These insights offer a formal understanding of how carefully designed prompts can optimize a transformer's performance.

\subsection{Longer Prompts Enhance Virtual Network Capacity}\label{sec: application length}

Recent studies suggest that longer, more detailed prompts enhance contextual richness and improve internal reasoning, enabling models to better capture task structures. In ICL, \citet{brown2020language,min2022rethinking} and \citet{garg2022can} show that increasing the number of demonstration examples empirically and numerically improves performance. Additionally, research on complexity-based prompting \citep{fu2022complexity} emphasizes the role of both example quality and length in enhancing reasoning and performance.

Building on these empirical findings, we provide a theoretical analysis of how prompt length affects the capacity of transformer models. As indicated by Corollary~\ref{cor: approximation error}, longer prompts can effectively emulate deeper virtual neural networks, enabling the approximation of a broader class of functions. In particular, we derive the following explicit result:
\begin{cor}[Approximation Error in terms of Prompt Length]\label{cor: approximation error by prompt length}
    Fix $p, \beta \in \N$.
    Set $d \geq 17\beta^{p+1} 3^p p+1$.
    Then, for any $T \geq 2d(108 \beta^2+3p)$, by setting $L = \floor{T/(2d)}$, we obtain the following bound: 
    \begin{align*}
        \sup_{f \in \mathcal{C}_\u^\beta([0, 1]^p)} \inf_{\substack{S>0\\H^\pro \in \mathcal{Q}_{\R^d}(T,L,S)}} \|\hat f_{\Theta^*,H^\pro,L,p} - f\|_{\mathcal{L}^\infty([0, 1]^p)} = \tilde O\qty(\frac{1}{T^{2\beta/p}}).
    \end{align*}
\end{cor}
Corollary~\ref{cor: approximation error by prompt length} demonstrates that the approximation error for functions in $\mathcal{C}_\u^\beta([0,1]^p)$ scales as $\tilde O(T^{-2\beta/p})$ provided that the inference depth $L$ grows proportionally to the prompt length $T$. This scaling law not only supports complexity-based prompting but also highlights how longer prompts effectively increase the depth of the emulated virtual network, enhancing the transformer's expressive power. The proof of Corollary~\ref{cor: approximation error by prompt length} is deferred to Appendix~\ref{sec: application ap}.

To empirically validate our theoretical results, we conducted experiments on three benchmark mathematical question-answering datasets. For each dataset, we randomly selected five questions with the longest overall lengths (i.e., including both the question and its answer) 
to serve as in-context examples, and compared performance against prompts containing the shortest examples. To mitigate the effect of highly similar questions, the examples were randomly sampled from a pool of 50 longest/shortest examples. The experimental results, reported in Table~\ref{tab: GPT length}, indicate that longer examples consistently yield higher performance across three datasets for both GPT-3.5 and GPT-4o mini. This result implies that increasing the amount of contextual detail may enhance the model's internal representations and overall performance.

\begin{table}[h!] 
\centering 
\caption{Performance comparison of in-context learning with long examples and short examples on multiple math datasets, GPT-3.5 and GPT-4o mini} 
\begin{tabular}{lcccccc} 
\toprule 
\multirow{2}{*}{\textbf{Prompt Style}} & \multicolumn{3}{c}{\textbf{GPT-3.5}} & \multicolumn{3}{c}{\textbf{GPT-4o mini}} \\ \cmidrule(lr){2-4} \cmidrule(lr){5-7} & \textbf{GSM8K} & \textbf{AQUA} & \textbf{MATH} & \textbf{GSM8K} & \textbf{AQUA} & \textbf{MATH} \\ \midrule 
Short Examples & 78.00\% & 57.48\% & 25.50\% & 92.50\% & 78.35\% & 55.50\% \\
Long Examples & 82.50\% & 58.27\% & 27.50\% & 94.50\% & 79.63\% & 57.00\% \\ 
\bottomrule 
\end{tabular} 
\label{tab: GPT length} 
\end{table}

\subsection{Filtering Irrelevant Tokens to Mitigate Noise}\label{sec: application irrelevant}

Recent studies have shown that irrelevant information in prompts can significantly degrade model performance. \citet{zhou2024can} found that adding irrelevant reasoning steps to CoT prompts leads to substantial accuracy drops. Similarly, \citet{jiang2024enhancing} and \citet{wu2024instructing} observed that inserting irrelevant content into prompts degrades performance and proposed methods to identify and filter such noise. \citet{shi2023large} introduced the GSM8K-IC dataset, which includes deliberately inserted irrelevant content, demonstrating that model accuracy declines sharply in its presence. We model irrelevant information in a prompt as random tokens appended to the original prompt.
For any given prompt $H^\pro \in \mathcal{P}_{\R^d}(T,L,S)$, we append tokens with random word embeddings $\bv_1, \bv_2, \dots, \bv_K$ to form a new prompt $H^\pro \oplus [\bv_1, \bv_2, \dots, \bv_K] \in \mathcal{P}_{\R^d}(T+K,L,S)$ defined as:
\begin{small}
\begin{align*}
    H^\pro \oplus [\bv_1, \bv_2, \dots, \bv_K] := [H^\pro, \tilde \h_{T+1}, \tilde \h_{T+2}, \dots, \tilde \h_{T+K}], \ \text{ where } \tilde \h_{T+k} := \begin{pmatrix}
        \bv_k\\
        p(2L+k, T+k, S)
    \end{pmatrix}.
\end{align*}
\end{small}\noindent
Here, note that the operator $\oplus$ depends on $T,L$, and $S$, with the appended irrelevant tokens treated as originating from an additional $\floor{K/2}$ virtual layers to streamline the analysis.

We formalize the detrimental effect of irrelevant tokens via the following result.
\begin{cor}\label{cor: irrelevant}
    Fix any $B \geq 1$, $p, \beta \in \N$. Set $d \geq p + 1$. 
    Let $K \sim \operatorname{Poi}(\lambda)$ for some $\lambda = \Theta(1)$, and assume that $\bv_1, \bv_2, \dots, \bv_K$ are independently drawn from the uniform distribution over $\mathcal{B}_d(B)$.
    Then, for any $T, L \in \N$ and $S \geq dB^{4L+4} (T+2)^{2L+2}$, the following holds:
    \begin{align*}
        \sup_{f \in \mathcal{C}_\u^\beta([0, 1]^p)} \inf_{H^\pro \in \mathcal{Q}_{\mathcal{B}_d(B)}(T,L,S)} \E\bigl[\|\hat f_{\Theta^*,H^\pro \oplus [\bv_1,\bv_2,\dots,\bv_K],L+1,p} - f\|_{\mathcal{L}^\infty([0, 1]^p)}\bigr] &\gtrsim 1.
    \end{align*}
\end{cor}
Corollary~\ref{cor: irrelevant} demonstrates that adding irrelevant tokens imposes a constant lower bound on the approximation error, regardless of how large inference depth $L$ or prompt length $T$ are. In contrast, without irrelevant information, Corollary~\ref{cor: approximation error by prompt length} shows that
\begin{align*}
    \sup_{f \in \mathcal{C}_\u^\beta([0, 1]^p)} \inf_{\substack{S > 0\\H^\pro \in \mathcal{P}_{\R^d}(T,L,S)}} \|\hat f_{\Theta^*,H^\pro,L,p} - f\|_{\mathcal{L}^\infty([0, 1]^p)} &= \tilde O\qty(\frac{1}{T^{2\beta/p}}).
\end{align*}
This contrast highlights the substantial negative impact of irrelevant tokens on performance. The proof of Corollary~\ref{cor: irrelevant} is deferred to Appendix~\ref{sec: application ap}.

Beyond treating irrelevant information as noisy tokens, it can also be viewed as contamination within the CoT reasoning steps, which similarly yields a constant lower bound on the approximation error. Further details on this perspective are provided in Appendix~\ref{sec: additional theory irrelevant ap}.

Empirical results further support these findings. Experiments on the GSM8K-IC dataset \citep{shi2023large}, summarized in Table~\ref{tab: filter out gsm8k}, show that model performance declines when tested on corrupted prompts. However, when explicitly instructed to ignore irrelevant content, both GPT-3.5 and GPT-4o mini models show improved accuracy.

\begin{table}[h]
\centering
\begin{tabular}{lccc}
\toprule
\textbf{Prompt Style} & \textbf{GPT-3.5} & \textbf{GPT-4o mini} \\
\midrule
Zero-Shot CoT + Filter Out & 84.50\% & 93.00\% \\
Zero-Shot CoT & 80.50\% & 91.50\% \\
\bottomrule
\end{tabular}
\caption{Performance of filtering out irrelevant information compared to standard Chain-of-Thought prompting on the GSM8K-IC dataset.}
\label{tab: filter out gsm8k}
\end{table}

\subsection{Enhancing Expressivity via Prompt Diversity}\label{sec: application diversity}

Recent empirical studies show that incorporating diverse prompts can significantly enhance LLM performance. For example, \citet{yao2023react} proposes ReAct, a prompting approach that guides LLMs to iteratively generate reasoning and actions, thereby enhancing performance in question-answering and fact verification tasks. Similarly, the Diversity-of-Thoughts technique by \citet{naik2023diversity} demonstrates that prompting an LLM to consider multiple distinct approaches and aggregating them can significantly boost performance in arithmetic reasoning, planning, and commonsense reasoning.
We now provide a theoretical analysis of how prompt diversity influences approximation error in the emulated neural network. According to Theorem~\ref{thm: prompt engineering NN}, the weight matrix of the emulated neural network is determined by the word embeddings of the prompt. If the prompts exhibit low diversity, meaning they span a low-dimensional subspace, the function approximation capacity of the emulated network is inherently limited by the rank of this weight matrix.

Formally, for any $d \geq r$, let 
$$
    \mathcal{U}_{d,r}(B) := \spn\{\e_1, \e_2, \dots, \e_r\} \cap \mathcal{B}_d(B) \subset \R^d,
$$
where $\e_j$ is the $j$-th standard basis vector in $\R^d$. The set $\mathcal{U}_{d,r}(B)$ represents a bounded subset formed by the first $r$ coordinates, capturing the effective diversity of the word embeddings.

We now state the following result.
\begin{cor}\label{cor: diversity}
    Fix $B > 0$, $d,\beta,p,L \in \N$ with $d \geq 17\beta^{p+1}3^{p}p + 1$, and $L\geq 108\beta^2+3p$. Then, for any $r \in \N$ with $17\beta^{p+1}3^{p}p + 1 \leq r \leq d$, then, for any $T \geq 2rL$,
    \begin{align}
        \sup_{f \in \mathcal{C}_\u^\beta([0, 1]^p)} \inf_{\substack{S > 0\\H^\pro \in \mathcal{Q}_{\mathcal{U}_{d,r}(B)}(T,L,S)}} \|\hat f_{\Theta^*,H^\pro,L,p} - f\|_{\mathcal{L}^\infty([0, 1]^p)} &= \tilde O\qty(\frac{1}{r^{2\beta/p}}).\label{eq: diversity ub}
    \end{align}
\end{cor}
Corollary~\ref{cor: diversity} establishes that the approximation error is bounded by $O(r^{-2\beta/p})$, which is driven by the measure of diversity $r$, provided the inference depth, prompt length, and latent dimension $d$ are sufficiently large. In other words, as the diversity $r$ increases, the approximation error decreases, with the prompt length growing in the same order as the diversity.
The proof, given in Appendix~\ref{sec: application ap}, follows by treating $r$ as the effective width of the virtual neural network instantiated by the transformer model. 
The result theoretically supports the empirical success of the Diversity-of-Thoughts \citep{naik2023diversity} and ReAct \citep{yao2023react}. 

To validate this, we conduct experiments on mathematical question-answering datasets. For each question, we prompt the LLM to consider three different approaches and then instruct the LLM to aggregate these approaches to generate a final answer. The experimental results are shown in Table~\ref{tab: diversity of thoughts}.

\begin{table}[h!]
    \centering
    \caption{Performance Comparison of Diversity of Thoughts and Chain-of-Thought on Multiple Math Datasets, GPT-3.5 and GPT-4o mini} 
    \begin{tabular}{lcccccc} 
        \toprule \multirow{2}{*}{\textbf{Prompt Style}} & \multicolumn{3}{c}{\textbf{GPT-3.5}} & \multicolumn{3}{c}{\textbf{GPT-4o mini}} \\ \cmidrule(lr){2-4} \cmidrule(lr){5-7} & \textbf{GSM8K} & \textbf{AQUA} & \textbf{MATH} & \textbf{GSM8K} & \textbf{AQUA} & \textbf{MATH} \\ \midrule 
        Chain-of-Thought & 78.50\% & 58.27\% & 28.50\% & 95.50\% & 66.14\% & 57.50\% \\ 
        Diversity of Thoughts & 81.50\% & 61.02\% & 29.50\% & 93.50\% & 74.02\% & 63.00\% \\\bottomrule 
    \end{tabular} \label{tab: diversity of thoughts} 
\end{table}

\subsection{Multi-Agent Collaboration in Prompt Engineering}\label{sec: application agents}

Multi-agent systems enhance LLM capabilities by enabling collaboration among specialized agents, improving reasoning, robustness, and adaptability. 
Empirical studies consistently demonstrate that both parallel and sequential multi-agent interactions lead to significant performance gains. For example, \citet{wang2022self} introduces self-consistency, which aggregates diverse CoT trajectories via parallel majority voting, effectively simulating multiple parallel agents. The Tree-of-Thoughts framework \citep{yao2023tree} structures reasoning as a tree, integrating parallel and sequential explorations to enhance problem-solving. As discussed in Section~\ref{sec: application diversity}, the ReAct framework \citep{yao2023react} synergizes sequential reasoning with actionable steps, where each step informs the next. Similarly, \citet{zhou2022least} propose sequential task decomposition, incrementally solving complex problems by building on previously solved sub-tasks. 

Further advancements include frameworks such as LLM-BLENDER \citep{jiang2023llm} and AutoGen \citep{wu2023autogen}, which ensemble multiple LLMs through dynamic ranking and coordinated reasoning, respectively, to produce more reliable and coherent outputs. More recently, \citet{wang2024mixture} introduced the MoA framework, a layered approach where multiple specialized agents within each layer address distinct sub-tasks, and their outputs are sequentially integrated to construct robust solutions. For a broader perspective on the effectiveness of structured multi-agent strategies, see the recent surveys \citep{han2024llm, li2024survey, guo2024large}.

We formalize multi-agent task decomposition through parallel and sequential agent interactions \citep{yao2023tree,wang2024mixture}. A complex task is partitioned into interdependent layers, where agents process sub-tasks independently. Their outputs are then aggregated by an LLM, synthesizing intermediate results into a structured final solution.
Let $\mathcal{A} = [|\mathcal{A}|]$ denote a set of agents, where each agent $a \in \mathcal{A}$ produces structured token representations:
\begin{align*} 
    \mathcal{R}_{\mathcal{U}}(T^a,\ell^a,S) = \{[\h_1^a, \h_2^a, \dots, \h_{T^a}^a] : \h_j^a = (\bu_j^{a \top}, p(w_j^a, j, S)^\top)^\top,\ \bu_j^a \in \mathcal{U}, \ceil{w_j^a/2} = \ell^a\}.
\end{align*} 
Here, $\ell^a$ denotes the hierarchical layer corresponding to agent $a$'s assigned sub-task, and $T^a$ represents the length of the agent's generated output.
Denote the output $H^a \in \mathcal{R}_{\mathcal{U}}(T^a,\ell^a,S)$ from agent $a \in \mathcal{A}$ by 
$$
    (H^a)_j = ((\h_j^a)_{1:d}^\top, p(w_j^a, j, S)^\top)^\top \subset \R^{(4d+8)\times T^a}.
$$
We construct the concatenated input to the aggregator LLM as
\begin{align}
    H^\mathcal{A} := \qty[\begin{pmatrix}
        (\h_1^1)_{1:d}\\
        p(w_1^1, 1, S)
    \end{pmatrix}, \begin{pmatrix}
        (\h_2^1)_{1:d}\\
        p(w_2^1, 2, S)
    \end{pmatrix}, \dots, \begin{pmatrix}
        (\h_{T^{|\mathcal{A}|}}^{|\mathcal{A}|})_{1:d}\\
        p(w_{T^{|\mathcal{A}|}}^{|\mathcal{A}|}, \sum_{a \in \mathcal{A}} T^a, S)
    \end{pmatrix}].\label{eq: H A}
\end{align}
We then analyze the performance of an aggregator LLM that integrates the multi-agent outputs. Given $T^a$ for $a \in \mathcal{A}$, we define the total contribution from agents assigned to the $\ell$-th layer of a sequential decomposition as $T_\ell^\mathcal{A} = \sum_{a \in \mathcal{A}} T^a \Id\{\ell^a = \ell\}$ and let $T^\mathcal{A} := \min_{\ell \in [L]} T_\ell^\mathcal{A}$ denote the minimum contribution across all layers.
\begin{cor}\label{cor: agents} 
    Fix $B \geq 1$, $\mathcal{A}$, $p, \beta, d, L \in \N$ with $d \geq 17\beta^{p+1}3^{p}p + 1$, and $L\geq 108\beta^2+3p$. Also fix $(\ell^a)_{a \in \mathcal{A}} \subset [L]$ and $(T^a)_{a \in \mathcal{A}} \subset \N$.
    If $34\beta^{p+1}3^{p}p + 2 \leq T^\mathcal{A}$, then,
    \begin{align*}
        \sup_{f \in \mathcal{C}_\u^\beta([0, 1]^p)} \inf_{\substack{S>0\\H^a \in \mathcal{R}_{\R^d}(T^a,\ell^a,S) \text{ for $a \in \mathcal{A}$}}} \|\hat f_{\Theta^*,H^\mathcal{A},L,p} - f\|_{\mathcal{L}^\infty([0, 1]^p)} &= \tilde O\qty(\frac{1}{L^{2\beta/p} (2d \wedge T^\mathcal{A})^{2\beta/p}}).
    \end{align*}
\end{cor}
Corollary~\ref{cor: agents} demonstrates that the approximation error scales as $\tilde O(1/(L T^\mathcal{A})^{2\beta/p})$ when $T^\mathcal{A} \leq 2d$. In essence, the corollary states that both a deeper sequential breakdown of tasks (increasing $L$), and enhanced parallel agent contributions (increasing $T^\mathcal{A}$ up to $2d$) are beneficial in multi-agent frameworks. Conversely, there may be no gain in adding more agents once the agents' minimal contribution exceeds $2d$. The proof is deferred to Appendix~\ref{sec: application ap}.

\section{Discussion}\label{sec: discussion}

Our results highlight the remarkable expressive power of transformers when guided by well-designed prompts, enabling them to dynamically emulate diverse neural network architectures. This perspective underscores the critical role of prompt design—through optimized positional encoding and word embeddings—in shaping model behavior, much like architecture selection in traditional neural networks. It also provides theoretical insight into key practical observations: longer and more detailed prompts enhance performance by improving expressivity and enabling the model to approximate more complex functions; irrelevant or noisy prompts introduce approximation errors and degrade performance; prompt diversity enhances expressivity and generalization by increasing the rank of the emulated neural network weights; and multi-agent prompting improves performance by enabling structured reasoning and task decomposition. Understanding these dynamics enables us to move beyond trial-and-error prompt engineering and develop principled, theoretically informed strategies for improving LLM performance across various tasks.

One direct application of this theoretical understanding is designing more effective prompt structures to maximize model utility without increasing computational costs. By leveraging insights into function approximation capacity, we can optimize prompt composition to elicit more expressive reasoning from LLMs. For instance, balancing between length and complexity ensures that prompts remain concise while maintaining strong task guidance. Additionally, filtering irrelevant content prevents unnecessary degradation of performance. This approach suggests that structured prompting strategies can serve as an efficient alternative to increasing model size, making smaller models more capable through targeted prompt optimization.

A promising avenue for future research is quantifying inference-time scaling—how LLMs utilize additional computation time to refine responses as in \citep{muennighoff2025s1,el2025competitive,guo2025deepseek}. Empirical evidence suggests that LLMs generate longer, more detailed outputs when given more inference time, even from short prompts. Establishing scaling laws that relate to model size, inference depth, and performance could lead to more efficient deployment strategies. Instead of relying on the growth of model size, we could leverage adaptive inference mechanisms that allow models to adjust their reasoning depth based on task complexity, improving efficiency without sacrificing accuracy. 

While our analysis focused on the role of prompts with fixed positional encoding, investigating how the initial layers of transformers assign the \emph{virtual} layer to each word token remains an intriguing direction for future research. Moreover, extending our framework to other model architectures, such as mixture-of-expert models and retrieval-augmented generation, could provide deeper insights into how prompting strategies generalize across different AI systems. We leave these important directions for future exploration.

\section*{Acknowledgments}

We would like to express our sincere gratitude to Kenji Kawaguchi and Masaaki Imaizumi for their valuable insights and engaging discussions. 

\bibliographystyle{apalike}
\bibliography{main}

\newpage

\appendix
\clearpage
\begin{center}
  {\Large \bfseries Appendix}
\end{center}
\vspace{1em}

In this Appendix, Section~\ref{sec: notation} introduces additional notations used in proofs, and Table~\ref{tab: notations} summarizes key notations.

Section~\ref{sec: additional theory ap} presents additional theoretical results that extend the discussion in the main text, including lower bounds on approximation error with prompts in Section~\ref{sec: additional theory lb ap} and an analysis of how irrelevant CoT segments deteriorate performance in Section~\ref{sec: additional theory irrelevant ap}.

Section~\ref{sec: emulate ap} provides detailed information on the construction of the transformer to emulate virtual neural networks in Section~\ref{sec: emulate} of the paper, as well as the corresponding proofs; in particular, Theorem~\ref{thm: prompt engineering ap} introduces the transformer $\TF_{\Theta^*}$ with its behavior for prompts, while Theorem~\ref{thm: prompt engineering NN ap} details the capacity analysis of $\TF_{\Theta^*}$ in emulating ReLU neural networks. In parallel, Theorem~\ref{thm: prompt engineering EUAF ap} introduces the transformer $\TF_{\Theta^\#}$ and Theorem~\ref{thm: prompt engineering EUAF NN ap} provides its property of emulating neural networks with the EUAF activation function.

Section~\ref{sec: approximation ap} provides the details for the approximation error bounds presented in the paper. Namely, Corollaries~\ref{cor: approximation error ap} and \ref{cor: approximation error EUAF ap} provide the approximation error bounds for prompts with $\TF_{\Theta^*}$ and $\TF_{\Theta^\#}$, respectively. 

Section~\ref{sec: application ap} presents the proofs for the application of our theoretical framework to longer prompts (Section~\ref{sec: application length}), filtering out irrelevant information (Section~\ref{sec: application irrelevant}), prompt diversity (Section~\ref{sec: application diversity}), and multi-agent setups (Section~\ref{sec: application agents}).

Finally, Section~\ref{sec: experiments ap} provides detailed descriptions of the experimental protocols presented in Section~\ref{sec: application}.

\section{Notation}\label{sec: notation}

We first introduce the notations used in this section. The vector $\1_d \in \R^d$ denotes the $d$-dimensional vector of ones, and $\zero_d \in \R^d$ denotes the $d$-dimensional vector of zeros.
In any vector, the symbol $*$ denotes an arbitrary real-valued entry whose precise value may depend on the context but is not relevant to the claim.
For any positive integer $I$, define $[I]=\{1,2,\cdots,I\}$ with the convention that $[0] = \emptyset$.
For a transformer parameter set $\Theta = (\theta_1,\dots,\theta_l)$, where each layer's parameter is given by 
$$
    \theta_{l'} = ((Q_{l',m},K_{l',m},V_{l',m})_{m \in [M_{l'}]}, (W_{l',1}, W_{l',2})),
$$
we define the following norms:
\begin{align}
    \|\Theta\|_0 &= \sum_{l' \in [l]} \qty(\|W_{l',1}\|_0 + \|W_{l',2}\|_0 + \sum_{m \in [M_{l'}]} (\|Q_{l',m}\|_0 + \|K_{l',m}\|_0 + \|V_{l',m}\|_0)),\nonumber\\
    \|\Theta\|_{\max} &= \max_{l' \in [l]} \qty{ \|W_{l',1}\|_{\max} \vee \|W_{l',2}\|_{\max} \vee \max_{m \in [M_{l'}]} \{\|Q_{l',m}\|_{\max} \vee \|K_{l',m}\|_{\max} \vee \|V_{l',m}\|_{\max}\}}.\label{eq: norms}
\end{align}
Table~\ref{tab: notations} summarizes the key notations used in our paper.

\begin{table}[H]
    \centering
    \renewcommand{\arraystretch}{1.2}
    \begin{longtable}[c]{c|p{0.8\textwidth}}
    \hline
    \multicolumn{1}{c|}{\textbf{Notation}} & \multicolumn{1}{c}{\textbf{Definition/Meaning}} \\
    \hline
    $d$, $D(:=4d+8)$ & The word embedding dimension and the token dimension, respectively. \\
    $\mathcal{U}\subset\R^d$ & A set of word embeddings. \\
    $H^\pro\in\R^{D\times T}$ & The matrix of prompt tokens of length $T$, i.e., 
    $H^\pro=[\h_1,\h_2,\dots,\h_T]$, with each prompt token defined as
    $$\h_j=\begin{pmatrix}\bu_j\\ \p_j\end{pmatrix}\in\R^{4d+8},$$ 
    where $\bu_j\in\mathcal{U}$ is the word embedding and $\p_j$ is the positional encoding defined later. \\
    $H^\dat\in\R^{D\times N}$ & The matrix of data tokens of length $N$, i.e., 
    $H^\dat=[\h_{T+1},\h_{T+2},\dots,\h_{T+N}]$,
    with each data token defined as
    $$\h_{T+i}=\begin{pmatrix}\z_i\\ \p_{T+i}\end{pmatrix}\in\R^{4d+8},$$ 
    where $\z_i\in\R^d$ is the input embedding to the virtual neural network and $\p_{T+i}$ is the positional encoding defined later. \\
    $w_j\in[2L]$ & An integer encoding the virtual layer index of the $j$-th prompt token (for $j\in[T]$). \\
    $S > 0$ & A scaling parameter ensuring that the positional encodings remain significant relative to the word embeddings. \\
    $\p_j \in \R^{3d+8}$ & The positional encoding of token $j$, defined as
    $\p_j=p(w_j,j,S)$ for $j\in[T]$, and $\p_{T+i}=p(0,T+i,S)$ for $i \in [N]$, where
    $$p(w,j,S)=\begin{pmatrix}\zero_{3d+4}\\ 1\\ S\\ S w\\ S j\end{pmatrix}\in\R^{3d+8}.$$ \\
    $\mathcal{P}_{\mathcal{U}}(T,L,S)$ & The set of prompts with \emph{exactly} $T$ prompt tokens. \\
    $\mathcal{Q}_{\mathcal{U}}(T,L,S)$ & The collection of prompts with \emph{at most} $T$ prompt tokens, i.e.,
    $\mathcal{Q}_{\mathcal{U}}(T,L,S)=\bigcup_{T'=1}^{T}\mathcal{P}_{\mathcal{U}}(T',L,S)$. \\
    $\mathcal{G}(\mathbf{r},d,\mathcal{U},B)$ & For $\mathbf{r}=(r_1,\dots,r_L)$, the set of functions from $\R^d$ to $\R^d$ implemented by $L$-layer neural networks with width $d$, with weight matrices having operator norms bounded by $B^2$ and ranks $r_1,r_2,\dots,r_L$, whose rank-$1$ factors are drawn from $\mathcal{U}$. \\
    $\hat{g}_{\Theta,H^\pro,N,L,d}(\z_i)$ & The output of the emulated virtual neural network corresponding to input $\z_i$, obtained by taking the first $d$ entries of a later generated token. \\
    $\hat{f}_{\Theta,H^\pro,L,p}(\x)$ & The output of the emulated virtual network from $\R^p$ to $\R$ defined as 
    $$\hat{f}_{\Theta,H^\pro,L,p}(\x)=\hat{g}_{\Theta,H^\pro,1,L,1}([\x^\top, 1, \zero_{d-p-1}^\top]^\top).$$\vspace{-1em}\\
    \hline
    \end{longtable}
    \caption{Key notations for prompts, tokens, and positional encodings.}
    \label{tab: notations}
\end{table}

\section{Additional Theoretical Results}\label{sec: additional theory ap}

In this section, we provide additional theoretical results that complement our main findings.

\subsection{Lower Bounds on Prompt Approximation Error}\label{sec: additional theory lb ap}

We establish lower bounds on the approximation error in terms of the effective dimension of $\mathcal{U}$ and the depth $L$ of the virtual neural networks.
For any $d \geq r$, and $B > 0$, we define the set $\mathcal{U}_{d,r}(B) := \spn\{\e_1, \e_2, \dots, \e_r\} \cap \mathcal{B}_d(B)$.
\begin{cor}\label{cor: approximation error lb ap}
    Let $B \geq 1$ and $d,p,\beta \in \N$ with $d \geq p+1$.
    Then, there exists a constant $C'(\beta, p) > 0$ such that for all $T, L, r \in \N$ with $p \leq r \leq d$, we have
    \begin{align*}
        \sup_{f \in \mathcal{C}_\u^\beta([0, 1]^p)} \inf_{\substack{S \geq d B^{4L} T^{2L} \vee 2L\\H^\pro \in \mathcal{Q}_{\mathcal{U}_{d,r}(B)}(T,L,S)}} \|\hat f_{\Theta^*,H^\pro,L,p} - f\|_{\mathcal{L}^\infty([0, 1]^p)} \geq C'(\beta, p) \{r^2 L^2 \log(rL)\}^{-\beta/p}.
    \end{align*}
\end{cor}
We also present an alternative version of Corollary~\ref{cor: approximation error lb ap}.
\begin{cor}\label{cor: approximation error lb epsilon ap}
    Let $B \geq 1$ and $d,T,p,\beta \in \N$.
    Then, there exists a constant $C''(\beta, p, d) > 0$ such that for any $\varepsilon \in (0, 1)$, one can select a threshold $S^* = S^*(d,B,p,\beta,\varepsilon) > 0$ with the following property: if $L \leq C'' \varepsilon^{-p/(2\beta)}$, then,
    \begin{align*}
        \sup_{f \in \mathcal{C}_\u^\beta([0, 1]^p)} \inf_{\substack{S \geq S^*\\H^\pro \in \mathcal{Q}_{\mathcal{B}_d(B)}(T,L,S)}} \|\hat f_{\Theta^*,H^\pro,L,p} - f\|_{\mathcal{L}^\infty([0, 1]^p)} \geq \varepsilon
    \end{align*}
\end{cor}
Note that Corollary~\ref{cor: approximation error} shows that an approximation error of at most $\varepsilon$ can be achieved provided the depth satisfies
$L = O(\varepsilon^{-p/(2\beta)}\log(1/\varepsilon))$,
with the overall prompt length scaling as $T = O(dL)$.
Conversely, the lower bound in Corollary~\ref{cor: approximation error lb epsilon ap} states that if $L \lesssim \varepsilon^{-p/(2\beta)}$, then the approximation error cannot be reduced below $\varepsilon$. This implies that achieving an error smaller than $\varepsilon$ requires a depth of at least $\Omega(\varepsilon^{-p/(2\beta)})$, so the upper bound in Corollary~\ref{cor: approximation error} is nearly optimal in its dependence on $L$.

\subsection{The Detrimental Effects of Irrelevant Prompt Tokens}\label{sec: additional theory irrelevant ap}

In addition to the modeling in Section~\ref{sec: application irrelevant}, we also model irrelevant information as semantically unrelated tokens injected into the prompt.
Given $B \geq 1$, let $\mathcal{\tilde U}(B), \mathcal{U}(B) \subset \mathcal{B}_d(B)$ be two bounded subsets with $\tilde T, T, \tilde L, L \in \N$ and $\tilde S,S > 0$. We define the concatenation of prompts $\tilde H^\pro \in \mathcal{P}_{\mathcal{\tilde U}(B)}(\tilde T,\tilde L,\tilde S)$ and $H^\pro \in \mathcal{P}_{\mathcal{U}(B)}(T,L,S)$ by
\begin{align*}
    (\tilde H^\pro \oplus' H^\pro)_j = \begin{cases}
         \begin{pmatrix}
             ((\tilde H^\pro)_j)_{1:d},\\
             p(w_j, j, S')
         \end{pmatrix} & \text{ for $j \leq \tilde T$},\\
         \begin{pmatrix}
             ((H^\pro)_{j-\tilde T})_{1:d},\\
             p(2\tilde L + w_{j-\tilde T}, j, S')
         \end{pmatrix} & \text{ for $j \geq \tilde T+1$},
    \end{cases}
\end{align*}
with $S' \geq d B^{4(L + \tilde L)} (T+\tilde T)^{2(L+\tilde L)} \vee (2(L+\tilde L) + 1)$.
Thus, $\tilde H^\pro \oplus' H^\pro \in \mathcal{P}_{\mathcal{\tilde U}(B) \cup \mathcal{U}(B)}(\tilde T+T,\tilde L+L,S')$.
This construction allows us to interpret $\tilde H^\pro$ as an irrelevant CoT segment prefixed to the primary prompt $H^\pro$.

For simplicity, assume that the relevant information is encoded in the subset $\mathcal{U}(B) = \spn\{\e_j : j \in \mathcal{J}\} \cap \mathcal{B}_d(B)$, where $\e_j$ is the $j$-th standard basis and $\mathcal{J} \subset [d]$.
In contrast, the irrelevant information is encoded in the orthogonal subset $\mathcal{\tilde U}(B) = \spn\{\e_j : j \in [d]\setminus\mathcal{J}\} \cap \mathcal{B}_d(B)$. This decomposition formally partitions the prompt's semantic space: $\mathcal{U}(B)$ carries the task-relevant features, while $\mathcal{\tilde U}(B)$ captures irrelevant content.

\begin{cor}\label{cor: irrelevant cot}
    Fix any $B \geq 1$, $p, d, \beta \in \N$, and $\mathcal{J} \subset [d]$. Let $\mathcal{U}(B)$ and $\mathcal{\tilde U}(B)$ be two bounded subsets of $\mathcal{B}_d(B)$ defined above.
    Then, for any $T,\tilde T,L,\tilde L \in \N$,
    \begin{align*}
        \sup_{f \in \mathcal{C}_\u^\beta([0, 1]^p)} \inf_{\substack{S > 0, S' > 0\\H^\pro \in \mathcal{P}_{\mathcal{U}(B)}(T,L,S)\\\tilde H^\pro \in \mathcal{P}_{\mathcal{\tilde U}(B)}(\tilde T,\tilde L,\tilde S)}} \|\hat f_{\Theta^*,\tilde H^\pro \oplus' H^\pro,\tilde L + L,p} - f\|_{\mathcal{L}^\infty([0, 1]^p)} &\geq 1
    \end{align*}
    holds.
\end{cor}
Corollary~\ref{cor: irrelevant cot} shows that when the prompt is always prefixed by an irrelevant CoT segment, the approximation error is bounded below by a constant, even for large $L$ and $T$. It is also straightforward to derive the same lower bound when the irrelevant CoT segments are randomly inserted into the prompt.

\subsection{Proofs}

Before proving Corollary~\ref{cor: approximation error lb ap}, we formally define the class of ReLU neural networks.
Let $\mathcal{A}(\x; W_\ell, \b_\ell) := W_\ell \x + \b_\ell$.
Given $r, p, L \in \N$, define the class of functions implemented by standard ReLU neural networks with depth $L$ and width $r$ from $\R^p$ to $\R$ as
\begin{align}
    \mathcal{N}(r, p, L) := \{ &\mathcal{A}(\cdot; \bar W_L, \bar b_L) \circ \sigma \circ \mathcal{A}(\cdot; \bar W_{L-1}, \bar \b_{L-1}) \circ \sigma \circ \dots \circ \sigma \circ \mathcal{A}(\cdot; \bar W_1, \bar \b_1)\nonumber\\
    &\quad : \bar W_1 \in \R^{r \times p}, \{\bar W_2, \dots, \bar W_{L-1}\} \subset \R^{r \times r}, \bar W_L \in \R^{1 \times r},\nonumber\\
    &\quad\quad \{\bar \b_1, \bar \b_2, \dots, \bar \b_{L-1}\} \subset \R^r, \bar b_L \in \R\}.\label{eq: NN class}
\end{align}

We now state a lower bound for the approximation error achieved by neural networks, derived via a lower bound on the VC dimension to achieve an $\varepsilon$-approximation as in \citet{lu2021deep}.
\begin{lem}\label{lem: approximation NN lb}
    Fix any $p, \beta \in \N$. Then, there exists a constant $C'(\beta, p) > 0$ such that the following holds for all $L, r \in \N$ with $r \geq p$:
    \begin{align*}
        \sup_{f \in \mathcal{C}_\u^\beta([0,1]^p)} \inf_{\bar \phi \in \mathcal{N}(r, p, L)} \|f - \bar \phi\|_{\mathcal{L}^\infty([0, 1]^p)} \geq C'(\beta, p) \qty{r^2 L^2 \log(r L) }^{-\beta/p}.
    \end{align*}
\end{lem}

\begin{proof}[Proof of Lemma~\ref{lem: approximation NN lb}]
    By applying the contrapositive of Theorem~2.4 of \citet{lu2021deep}, there exists a constant $C''' = C'''(\beta, p) > 0$ such that for any $\varepsilon > 0$, if the VC dimension \citep{vapnik2015uniform} of $\mathcal{N}(r,p,L)$, denoted by $\operatorname{VCDim}(\mathcal{N}(r,p,L))$, satisfies $1 \leq \operatorname{VCDim}(\mathcal{N}(r,p,L)) \leq C'''(\beta, p) \varepsilon^{-p/\beta}$, then 
    \begin{align*}
        \sup_{f \in \mathcal{C}_\u^\beta([0,1]^p)} \inf_{\bar \phi \in \mathcal{N}(r,p,L)} \|f - \bar \phi\|_{\mathcal{L}^\infty([0, 1]^p)} > \varepsilon.
    \end{align*}
    As stated in \citet{lu2021deep}, the VC dimension of $\mathcal{N}(r,p,L)$ is bounded by $O(r^2 L^2 \log(rL))$. By choosing $\varepsilon = \{(C'''(\beta, p))^{-1} r^2 L^2 \log(rL)\}^{-\beta/p}$, we obtain the desired lower bound.
\end{proof}

We now proceed to the proofs of Corollaries~\ref{cor: approximation error lb ap} and \ref{cor: approximation error lb epsilon ap}.
\begin{proof}[Proof of Corollary~\ref{cor: approximation error lb ap}]
    We first show that the class of virtual neural networks that can be realized by prompts is a subset of the class of neural networks defined in \eqref{eq: NN class}. Fix any $H^\pro \in \mathcal{Q}_{\mathcal{U}_{d,r}(B)}(T,L,S)$ and set $S \geq d B^{4L} T^{2L} \vee 2L$.
    From Corollary~\ref{cor: prompt engineering random ap} and Definition~\ref{def: function approx}, there exist $(W^\pro_1)_{\ell \in [L]} \subset \R^{d \times d}$ such that $H^\pro$ realizes a virtual neural network 
    $$
        \hat f_{\Theta^*,H^\pro,L,p}(\x) = (W^\pro_L \sigma( W^\pro_{L-1} \sigma( \dots \sigma( W^\pro_1 [\x^\top, 1, \zero_{d-p-1}^\top]^\top) \dots )))_1.
    $$
    From the explicit formula of $W^\pro_\ell$ given in \eqref{eq: W ell random N1}, the singular vectors of $W^\pro_\ell$ lie in $\spn\{\e_1, \e_2, \dots, \e_r\}$. Define $\bar W_\ell := [I_r, O] W^\pro_\ell [I_r, O]^\top$ for $\ell \in [L-1]\setminus\{1\}$, $\bar W_1 := [I_r, O] W^\pro_1 [I_p, O]^\top$, $\bar W_L := [1, \zero_{d-1}] W^\pro_L [I_r, O]^\top$. 
    Also define $\bar \b_1 = W^\pro_1 \e_{p+1}$.
    Then, 
    $$
        \hat f_{\Theta^*,H^\pro,L,p}(\x) = \mathcal{A}(\cdot; \bar W_L, 0) \circ \sigma \circ \mathcal{A}(\cdot; \bar W_{L-1}, \zero_r) \circ \sigma \circ \dots \circ \sigma \circ \mathcal{A}(\cdot; \bar W_1, \b_1),
    $$
    which implies that $\hat f_{\Theta^*,H^\pro,L,p}(\x) \in \mathcal{N}(r, p, L)$.
    
    Since $H^\pro \in \mathcal{Q}_{\mathcal{U}_{d,r}(B)}(T,L,S)$ was arbitrary, we have $\{\hat f_{\Theta^*,H^\pro,L,p}: H^\pro \in \mathcal{Q}_{\mathcal{U}_{d,r}(B)}(T,L,S)\} \subset \mathcal{N}(r,p,L)$.
    Therefore,
    \begin{align*}
        &\sup_{f \in \mathcal{C}_\u^\beta([0, 1]^p)} \inf_{\substack{S \geq d B^{4L} T^{2L} \vee 2L\\H^\pro \in \mathcal{Q}_{\mathcal{U}_{d,r}(B)}(T,L,S)}} \|\hat f_{\Theta^*,H^\pro,L,p} - f\|_{\mathcal{L}^\infty([0, 1]^p)}\\
        &\quad\geq \sup_{f \in \mathcal{C}_\u^\beta([0, 1]^p)} \inf_{\bar \phi \in \mathcal{N}(r, p, L)} \|\bar \phi - f\|_{\mathcal{L}^\infty([0, 1]^p)}\\
        &\quad\geq C'(\beta, p) \{r^2 L^2 \log(rL)\}^{-\beta/p},
    \end{align*}
    where we used Lemma~\ref{lem: approximation NN lb}. This completes the proof.
\end{proof}

\begin{proof}[Proof of Corollary~\ref{cor: approximation error lb epsilon ap}]
    Suppose that $S^* > 0$ is chosen sufficiently large.
    By setting $r = d$ and choosing $\varepsilon \leq C'(\beta, p) (d^2 L^2 \log(dL))^{-\beta/p}$ in Corollary~\ref{cor: approximation error lb ap}, we have
    \begin{align*}
        \sup_{f \in \mathcal{C}_\u^\beta([0, 1]^p)} \inf_{\substack{S \geq S^*\\H^\pro \in \mathcal{Q}_{\mathcal{B}_d(B)}(T,L,S)}} \|\hat f_{\Theta^*,H^\pro,L,p} - f\|_{\mathcal{L}^\infty([0, 1]^p)} \geq \varepsilon.
    \end{align*}
    By the choice of $\varepsilon$, we obtain $L \leq C'' \varepsilon^{-p/(2\beta)}$, where $C'' = C''(\beta, p,d) > 0$ is a constant depending on $\beta, p$ and $d$.
    Hence we may set $S^* = dB^{4C'' \varepsilon^{-p/(2\beta)}} T^{2C'' \varepsilon^{-p/(2\beta)}} \vee 2C'' \varepsilon^{-p/(2\beta)}$ to conclude the proof.
\end{proof}

\begin{proof}[Proof of Corollary~\ref{cor: irrelevant cot}]
    It suffices to show the lower bound for $f \equiv 1 \in \mathcal{C}_\u^\beta([0,1]^p)$. Fix any $S, S' > 0$, $H^\pro \in \mathcal{P}_{\mathcal{U}(B)}(T,L,S)$ and $\tilde H^\pro \in \mathcal{P}_{\mathcal{\tilde U}(B)}(T,L,S)$.
    From Corollary~\ref{cor: prompt engineering random ap}, the transformer $\TF_{\Theta^*}$ with the prompt $\tilde H^\pro \oplus' H^\pro$ produces sequential outputs $\h_{\tilde T + T + 2}, \h_{\tilde T + T + 3}, \dots$ with $\h_{\tilde T+T+2} \in \spn\{(\h_j)_{1:d} : w_j=1\}$ and
    \begin{align}
        (\h_{\tilde T + T+1+\ell})_{1:d} &= \qty(\sum_{j' \in [\tilde T + T-1]} \Id\{w_{j'}=2\ell-1, w_{j'+1}=2\ell\} (\h_{j'})_{1:d} (\h_{j'+1})_{1:d}^\top) \sigma((\h_{\tilde T + T+\ell})_{1:d})\label{eq: h noisy}
    \end{align}
    for all $\ell \in [\tilde L + L] \setminus\{1\}$.
    This implies that $(\h_{\tilde T + T+1+\tilde L})_{1:d} \in \spn\{(\h_1)_{1:d}, (\h_2)_{1:d}, \dots, (\h_{\tilde T})_{1:d}\} \subset \mathcal{\tilde U}(B) = \spn\{\e_j: j \not\in \mathcal{J}\}$.
    Since ReLU activation function maps $0$ to $0$, we have $\sigma((\h_{\tilde T + T+1+\tilde L})_{1:d}) \in \mathcal{\tilde U}(B)$.
    
    On the other hand, note that for $j' \in [\tilde T]$ we have $w_{j'} \leq 2\tilde L$ while for $j' \in [\tilde T + T - 1] \setminus [\tilde T]$, $w_{j'} \geq 2\tilde L + 1$. Setting $\ell = \tilde L + \ell'$ with $\ell' \in [L]$ in \eqref{eq: h noisy} yields
    \begin{small}
    \begin{align}
        &(\h_{\tilde T + T+1+\tilde L+\ell'})_{1:d}\nonumber\\
        &\quad= \qty(\sum_{j' \in [\tilde T + T - 1]} \Id\{w_{j'}=2(\tilde L+\ell')-1, w_{j'+1}=2(\tilde L+\ell')\} (\h_{j'})_{1:d} (\h_{j'+1})_{1:d}^\top) \sigma((\h_{\tilde T + T+\tilde L+\ell'})_{1:d})\nonumber\\
        &\quad= \qty(\sum_{j' \in [\tilde T + T - 1]\setminus[\tilde T]} \Id\{w_{j'}=2(\tilde L+\ell')-1, w_{j'+1}=2(\tilde L+\ell')\} (\h_{j'})_{1:d} (\h_{j'+1})_{1:d}^\top) \sigma((\h_{\tilde T + T+\tilde L+\ell'})_{1:d}).\label{eq: h t t 1 l'}
    \end{align}
    \end{small}\noindent
    Since $(\h_{j'+1})_{1:d} \in \mathcal{U}(B)$ for $j' \in [\tilde T + T - 1] \setminus [\tilde T]$ and $\sigma((\h_{\tilde T + T+1+\tilde L})_{1:d}) \in \mathcal{\tilde U}(B)$, we have
    \begin{align*}
        (\h_{j'+1})_{1:d}^\top \sigma((\h_{\tilde T + T+1+\tilde L})_{1:d}) = 0 \text{ for all $j' \in [\tilde T + T - 1]\setminus[\tilde T]$},
    \end{align*}
    and hence $(\h_{\tilde T + T+1+\tilde L+1})_{1:d} = \zero_d$.
    By repeatedly applying \eqref{eq: h t t 1 l'}, we obtain
    \begin{align*}
        (\h_{\tilde T + T+1+\tilde L+\ell'})_{1:d} &= \zero_d \text{ for all $\ell' \in [L]$}.
    \end{align*}
    In summary, $\TF_{\Theta^*}$ with the prompt $\tilde H^\pro \oplus' H^\pro$ yields $\hat f_{\Theta^*,\tilde H^\pro \oplus' H^\pro,L,p}(\x) = 0$ for any $\x \in [0, 1]^p$. Since $H^\pro$, $\tilde H^\pro$, $S$, and $S'$ are arbitrary, we conclude that
    \begin{align*}
        \inf_{\substack{S > 0, S' > 0\\H^\pro \in \mathcal{P}_{\mathcal{U}(B)}(T,L,S)\\\tilde H^\pro \in \mathcal{P}_{\mathcal{\tilde U}(B)}(\tilde T,\tilde L,\tilde S)}} \|\hat f_{\Theta^*,\tilde H^\pro \oplus' H^\pro,\tilde L + L,p} - f\|_{\mathcal{L}^\infty([0, 1]^p)} &= \|f\|_{\mathcal{L}^\infty([0, 1]^p)} = 1.
    \end{align*}
\end{proof}

\section{Technical Details and Proofs for Section~\ref{sec: emulate}}\label{sec: emulate ap}

In this section, we provide proofs for the results on the capacity of transformers to emulate neural networks that were presented in Section~\ref{sec: emulate}.

\subsection{Theoretical Results on Transformer Emulation}
We now state the main theorems and corollaries regarding the capacity of transformers with prompts to emulate neural networks.
\begin{thm}\label{thm: prompt engineering ap}
    Fix $d \in \N$. Then, there exists a $7$-layer transformer $\TF_{\Theta^*}$ satisfying the following properties (P1)--(P3):
    \begin{enumerate}[label=(P\arabic*)]
        \item $\TF_{\Theta^*}$ has at most $8$ heads in attention layers and FFN width of $O(d)$,
        \item $\Theta^*$ only depends on $d$, and $\|\Theta^*\|_0 \lesssim d$, $\|\Theta^*\|_{\max} \lesssim 1$,
        \item Fix any $B \geq 1$, $L,T \in \N$, and $\mathcal{U} \subset \mathcal{B}_d(B)$. 
        Set $S \geq d B^{4L} T^{2L} \vee 2L$. For any prompt $H^\pro = [\h_1, \dots, \h_T] \in \mathcal{P}_{\mathcal{U}}(T,L,S)$, $N \in \N$, $(\z_i)_{i \in [N]} \subset [0, 1]^d$, $\TF_{\Theta^*}$ given the initial input 
        $$H = [H^\pro, h^\dat(\z_1, T+1, S), \dots, h^\dat(\z_N, T+N, S)] \in \R^{(4d+8)\times (T+N)}$$ 
        sequentially outputs $\h_{T+N+1}, \h_{T+N+2}, \dots$ satisfying
        \begin{align*}
            &\h_{T+N\ell+i} = \begin{pmatrix}
                W^\pro_\ell \sigma(W^\pro_{\ell-1} \sigma( \dots \sigma(W^\pro_{1,i} \z_i)\dots))\\
                p(-\ell, T+N\ell+i, S)
            \end{pmatrix}
        \end{align*}
        for all $i \in [N]$ and $\ell \in [L]$, where
        \begin{align}
            W^\pro_{1,i} &:= \sum_{j' \in [T-1]} (\h_{j'})_{1:d} \Id\{w_{j'}=1\} (\h_{j'+1})_{1:d}^\top \Id\{w_{j'+1}=2\}\nonumber\\
            &\quad+ \Id\{i=1\} (\h_T)_{1:d} \Id\{w_T=1\} \sum_{j' \in [T]} (\h_{j'})_{1:d}^\top \Id\{w_{j'}=1\}\nonumber\\
            W^\pro_\ell &:= \sum_{j' \in [T-1]} (\h_{j'})_{1:d} \Id\{w_{j'}=2\ell-1\} (\h_{j'+1})_{1:d}^\top \Id\{w_{j'+1}=2\ell\} \text{ for $\ell \in [L]\setminus\{1\}$}.\label{eq: W ell random}
        \end{align}
    \end{enumerate}
\end{thm}

Next, we state a corollary for the special case when $N = 1$ without proof.
\begin{cor}\label{cor: prompt engineering random ap}
    Fix any $B \geq 1$, $d, T,L \in \N$, and $\mathcal{U} \subset \mathcal{B}_d(B)$. Set $S \geq d B^{4L} T^{2L} \vee 2L$.
    Then, for any prompt $H^\pro = [\h_1, \dots, \h_T] \in \mathcal{P}_{\mathcal{U}}(T,L,S)$ and $\z \in [0, 1]^d$, $\TF_{\Theta^*}$ given the initial input 
    $$H = [H^\pro, h^\dat(\z, T+1, S)] \in \R^{(4d+8)\times (T+1)}$$ 
    sequentially outputs $\h_{T+2}, \h_{T+3}, \dots$. In particular,
    $(\h_{T+\ell+1})_{1:d} \in \spn\{(\h_j)_{1:d}: j \in [T], w_j = 2\ell-1\}$ and
    \begin{align*}
        (\h_{T+\ell+1})_{1:d} &= W^\pro_\ell \sigma((\h_{T+\ell})_{1:d})
    \end{align*}
    hold for all $\ell \in [L]$, where
    \begin{align}
        W^\pro_\ell &= \sum_{j' \in [T-1]} (\h_{j'})_{1:d} \Id\{w_{j'}=2\ell-1\} (\h_{j'+1})_{1:d}^\top \Id\{w_{j'+1}=2\ell\}\nonumber\\
        &\quad+ \Id\{\ell=1\} (\h_T)_{1:d} \Id\{w_T=2\ell-1\} \sum_{j' \in [T]} (\h_{j'})_{1:d}^\top \Id\{w_{j'}=2\ell-1\} \text{ for $\ell \in [L]$}.\label{eq: W ell random N1}
    \end{align}
\end{cor}

We now present the main result regarding the approximation capacity of prompts for coarse neural networks. 
\begin{thm}[Restatement of Theorem~\ref{thm: prompt engineering NN}]\label{thm: prompt engineering NN ap}
    Fix $d \in \N$. Let $\TF_{\Theta^*}$ be the transformer introduced in Theorem~\ref{thm: prompt engineering ap}.
    Fix any $B \geq 1$, $L \in \N$, $r_1, \dots, r_L \in \N\cup\{0\}$, $\mathcal{U} \subset \mathcal{B}_d(B)$, and $(W_1, \dots, W_L) \in \mathcal{W}(r_1, d, \mathcal{U}, B) \times \dots \times \mathcal{W}(r_L, d, \mathcal{U}, B)$. Let $\bar{r} = \max_{\ell \in [L]} r_\ell$.
    Then, there exists a prompt $H^\pro \in \mathcal{P}_{\mathcal{U}}(T,L,S)$ with $T = 2\sum_{\ell \in [L]} r_\ell$ and $S \geq d \bar{r} B^{4L} \vee 2L$ such that for any $N \in \N$ and $(\z_i)_{i \in [N]} \subset [0, 1]^d$, $\TF_{\Theta^*}$ given the initial input 
    $$H = [H^\pro, h^\dat(\z_1, T+1, S), \dots, h^\dat(\z_N, T+N, S)] \in \R^{(4d+8)\times (T+N)}$$ 
    sequentially outputs $\h_{T+N+1}, \h_{T+N+2}, \dots$. In particular,
    \begin{align*}
        &\h_{T+N\ell+i} = \begin{pmatrix}
            W_\ell \sigma(W_{\ell-1} \sigma( \dots \sigma(W_1 \z_i)\dots))\\
            p(-\ell, T+N\ell+i, S)
        \end{pmatrix}
    \end{align*}
    for all $i \in [N]$ and $\ell \in [L]$.
\end{thm}
Compared to Theorem~\ref{thm: prompt engineering ap}, the bound for $S$ in Theorem~\ref{thm: prompt engineering NN ap} is improved because we only require the transformer $\TF_{\Theta^*}$ to work on a certain subset of $\mathcal{P}_{\mathcal{U}}(T,L,S)$.

We present a result for a transformer with the activation function in the first feed-forward transformation replaced by the EUAF. For brevity, we also denote by $\TF_{\Theta}$ the transformer with the structure given in \eqref{eq: EUAF transformer}.
\begin{thm}\label{thm: prompt engineering EUAF ap}
    Fix $d \in \N$. Then, there exists an $8$-layer transformer $\TF_{\Theta^\#}$ with the activation function in the first feed-forward transformation replaced by the EUAF satisfying the following properties (P1)--(P3):
    \begin{enumerate}[label=(P\arabic*)]
        \item $\TF_{\Theta^\#}$ has at most $8$ heads in attention layers and FFN width of $O(d)$,
        \item $\Theta^\#$ only depends on $d$, and $\|\Theta^\#\|_0 \lesssim d$, $\|\Theta^\#\|_{\max} \lesssim 1$,
        \item Fix any $B \geq 1$, $L,T \in \N$, and $\mathcal{U} \subset \mathcal{B}_d(B)$. 
        Set $S \geq d B^{4L} T^{2L} \vee 2L$. For any prompt $H^\pro = [\h_1, \dots, \h_T] \in \mathcal{P}_{\mathcal{U}}(T,L,S)$, $N \in \N$, $(\z_i)_{i \in [N]} \subset [0, 1]^d$, $\TF_{\Theta^\#}$ given the initial input 
        $$H = [H^\pro, h^\dat(\z_1, T+1, S), \dots, h^\dat(\z_N, T+N, S)] \in \R^{(4d+8)\times (T+N)}$$ 
        sequentially outputs $\h_{T+N+1}, \h_{T+N+2}, \dots$. In particular,
        \begin{align*}
            &\h_{T+N\ell+i} = \begin{pmatrix}
                W^\pro_\ell \sigma_\EUAF(W^\pro_{\ell-1} \sigma_\EUAF( \dots \sigma_\EUAF(W^\pro_{1,i} \z_i)\dots))\\
                p(-\ell, T+N\ell+i, S)
            \end{pmatrix}
        \end{align*}
        for all $i \in [N]$ and $\ell \in [L]$, where $W^\pro_{1,i}$ and $W^\pro_\ell$ are defined in \eqref{eq: W ell random}.
    \end{enumerate}
\end{thm}

An analogous statement to Theorem~\ref{thm: prompt engineering NN ap} for $\TF_{\Theta^\#}$ holds.
\begin{thm}\label{thm: prompt engineering EUAF NN ap}
    Fix $d \in \N$. Let $\TF_{\Theta^\#}$ be the transformer introduced in Theorem~\ref{thm: prompt engineering EUAF ap}.
    Fix any $B \geq 1$, $L \in \N$, $r_1, \dots, r_L \in \N\cup\{0\}$, $\mathcal{U} \subset \mathcal{B}_d(B)$, and $(W_1, \dots, W_L) \in \mathcal{W}(r_1, d, \mathcal{U}, B) \times \dots \times \mathcal{W}(r_L, d, \mathcal{U}, B)$. Let $\bar{r} = \max_{\ell \in [L]} r_\ell$.
    Then, there exists a prompt $H^\pro \in \mathcal{P}_{\mathcal{U}}(T,L,S)$ with $T = 2\sum_{\ell \in [L]} r_\ell$ and $S \geq d \bar{r} B^{4L} \vee 2L$ such that for any $N \in \N$ and $(\z_i)_{i \in [N]} \subset [0, 1]^d$, $\TF_{\Theta^\#}$ given the initial input 
    $$H = [H^\pro, h^\dat(\z_1, T+1, S), \dots, h^\dat(\z_N, T+N, S)] \in \R^{(4d+8)\times (T+N)}$$ 
    sequentially outputs $\h_{T+N+1}, \h_{T+N+2}, \dots$. In particular,
    \begin{align*}
        &\h_{T+N\ell+i} = \begin{pmatrix}
            W_\ell \sigma_\EUAF(W_{\ell-1} \sigma_\EUAF( \dots \sigma_\EUAF(W_1 \z_i)\dots))\\
            p(-\ell, T+N\ell+i, S)
        \end{pmatrix}
    \end{align*}
    for all $i \in [N]$ and $\ell \in [L]$.
\end{thm}
The proof of Theorem~\ref{thm: prompt engineering EUAF NN ap} follows by a direct analogy to that of Theorem~\ref{thm: prompt engineering NN ap} and is omitted.

\subsection{Proofs}
Before presenting the proofs, we introduce two auxiliary functions that combine a linear function and an indicator function.
For any $B > 0$ and $d \in \N$, define the functions $\phi_{d,B}: \R^{d+2} \to \R^d$ and $\psi_{d,B}: \R^{d+2} \to \R^d$ by
\begin{align}
    \phi_{d,B}(\z; j, j') &:= -\sigma\qty(\z + 4B \qty(j' - j + \frac{1}{2}) \1_d) + 2\sigma\qty(\z + 4B \qty(j' - j + \frac{1}{4}) \1_d)\nonumber\\
    &\quad- 2\sigma\qty(\z + 4B \qty(j' - j - \frac{1}{4}) \1_d) + \sigma\qty(\z + 4B \qty(j' - j - \frac{1}{2}) \1_d),\label{eq: phi}\\
    \psi_{d,B}(\z; j, j') &:= \sigma\qty(\frac{1}{2} \z + B \qty(j' - j + \frac{1}{2})\1_d) - \sigma\qty(-\frac{1}{2} \z + B \qty(j' - j + \frac{1}{2}) \1_d).\nonumber
\end{align}
We frequently use the following lemma.
\begin{lem}\label{lem: phi}
    Fix $B > 0$ and $d \in \N$.
    Then, for any $\z$ with $\|\z\|_{\max} \leq B$, and $j, j' \in \Z$,
    \begin{align*}
        \phi_{d,B}(\z; j, j') = \z \Id\{j = j'\}, \ \ \psi_{d,B}(\z; j, j') = \z \Id\{j' \geq j\}.
    \end{align*}
\end{lem}
In addition, note that if $\|\z\|_{\max} \leq B$, then $\|\phi_{d,B}(\z; j, j')\|_{\max} \leq B$ for any $j, j' \in \Z$. Moreover, for any $s > 0$,
\begin{align*}
    \phi_{d,B}(\z; j, j') = s \phi_{d,B/s}(\z/s; j, j'), \ \ \psi_{d,B}(\z; j, j') = s \psi_{d,B/s}(\z/s; j, j').
\end{align*}

\begin{proof}[Proof of Lemma~\ref{lem: phi}]\label{proof: lem: phi}
    It suffices to prove the claim for the case $d=1$.
    Fix any $x \in \mathbb{R}$ with $|x| \leq B$.
    \paragraph{Analysis of $\phi_{d,B}$}
    We first consider the case when $j = j'$. Then, 
    $$
        \phi_{d,B}(x; j,j) = -\sigma(x + 2B) + 2\sigma(x + B) - 2\sigma(x - B) + \sigma(x - 2B).
    $$
    Since $|x| \leq B$, it follows that $x + 2B > 0$, $x + B \geq0$, $x - B \leq 0$, and $x - 2B < 0$. Hence,
    $$
    \sigma(x+2B)=x+2B,\quad \sigma(x+B)=x+B,\quad \sigma(x-B)=0,\quad \sigma(x-2B)=0.
    $$
    Thus,
    $$
    \phi_{d,B}(x; j,j)= -(x+2B) + 2(x+B) = x.
    $$
    Next, consider the case when $j \neq j'$. Define $k = j' - j \neq 0$. 
    For $k \geq1$, note that
    $$
    x+4B\Bigl(k+\frac{1}{2}\Bigr) > 0,\quad x+4B\Bigl(k+\frac{1}{4}\Bigr) > 0,\quad x+4B\Bigl(k-\frac{1}{4}\Bigr) > 0,\quad x+4B\Bigl(k-\frac{1}{2}\Bigr) \geq0.
    $$
    Thus, the ReLU function acts linearly, and we have
    \begin{small}
    \begin{align*}
        \phi_{d,B}(x; j,j') &= -\Bigl(x+4B\Bigl(k+\frac{1}{2}\Bigr)\Bigr) + 2\Bigl(x+4B\Bigl(k+\frac{1}{4}\Bigr)\Bigr) - 2\Bigl(x+4B\Bigl(k-\frac{1}{4}\Bigr)\Bigr) + \Bigl(x+4B\Bigl(k-\frac{1}{2}\Bigr)\Bigr)\\
        &= 0.
    \end{align*}
    \end{small}\noindent
    For $k \leq -1$, each of the terms $x+4B(k\pm 1/2)$ and $x+4B(k\pm 1/4)$ is negative, so $\sigma$ evaluates to $0$ on all terms and therefore $\phi_{d,B}(x; j,j') = 0$.
    Combining these cases, we obtain
    $$
    \phi_{d,B}(x; j,j') = x \Id\{j = j'\}.
    $$
    
    \paragraph{Analysis of $\psi_{d,B}$}
    When $j' \geq j$, we have
    $$
    \frac{x}{2}+B\Bigl(k+\frac{1}{2}\Bigr) \geq-\frac{B}{2}+B\Bigl(k+\frac{1}{2}\Bigr) \geq0,
    $$
    and similarly,
    $$
    -\frac{x}{2}+B\Bigl(k+\frac{1}{2}\Bigr) \geq-\frac{B}{2}+B\Bigl(k+\frac{1}{2}\Bigr) \geq0.
    $$
    Thus, the ReLU function again acts linearly, so
    $$
    \psi_{d,B}(x; j,j') = \Bigl(\frac{x}{2}+B\Bigl(k+\frac{1}{2}\Bigr)\Bigr) - \Bigl(-\frac{x}{2}+B\Bigl(k+\frac{1}{2}\Bigr)\Bigr) = x.
    $$
    When $j' < j$, the expressions $\frac{x}{2}+B\Bigl(k+1/2\Bigr)$ and $-\frac{x}{2}+B\Bigl(k+1/2\Bigr)$ are negative, so both ReLU terms equal $0$. Hence, $\psi_{d,B}(x; j,j') = 0$.
    Therefore,
    $$
    \psi_{d,B}(x; j,j') = x \Id\{j' \geq j\}.
    $$
    This completes the proof.
\end{proof}

We next define an identity mapping realized by an attention layer and a feed-forward network.
\begin{dfn}
    We denote by $\Attn_{\mu_\id}$ and $\FFN_{\nu_\id}$ an attention layer and a feed-forward neural network, respectively, that together realize the identity mapping $\z\mapsto \z:\R^D\to\R^D$. In particular, we take $\mu_\id = (O_D,O_D,O_D)$ and $\nu_\id = (O_D,O_D)$.
\end{dfn}

We now proceed to the proofs.
\begin{proof}[Proof of Theorem~\ref{thm: prompt engineering ap}]
    We prove properties (P1)--(P3) simultaneously.

    Set $S \geq d B^{4L} T^{2L} \vee 2L$.
    Recall that $H^\pro = [\h_1, \h_2, \dots, \h_T]$ denotes the prompt.
    For $j \geq T+1$, define the indices
    \begin{align}
        \xi(j) = \floor{(j-T-1) / N} \ \ \text{ and } \ \ \tau(j) = ((j-T-1) \operatorname{mod} N) + 1.\label{eq: xi and tau}
    \end{align}
    Here, $\xi(j)$ represents the layer index and $\tau(j)$ the token index within that layer.  
    Next, for $j \geq T+1$, set $\p_j = p(w_j, j, S)$ with $w_j = -\xi(j)$. Define data tokens $(\h_j)_{j \geq T+1}$ as follows:
    \begin{align}
        \h_j := \begin{pmatrix}
            \z_{\tau(j)}^{(\xi(j))}\\
            \p_j
        \end{pmatrix},\label{eq: data tokens}
    \end{align}
    where $\z_i^{(0)} := \z_i$, $\z_i^{(1)} := W^\pro_{1,i} \y_i^{(0)}$, $\z_i^{(\ell)} := W^\pro_\ell \y_i^{(\ell-1)}$ for $\ell \geq 2$ with $\y_i^{(0)} := \z_i$ and $\y_i^{(\ell)} := \sigma(\z_i^{(\ell)})$ for $\ell \geq 1$.
    Note that data tokens $\h_{T+1}, \dots, \h_{T+N}$ defined in \eqref{eq: data tokens} match those introduced in \eqref{eq: h data}. Let $H^\dat = [\h_{T+1}, \dots, \h_{T+N}]$.
    
    Our goal is to construct a transformer $\TF_{\Theta^*}$ such that for any $v \in \{0\}\cup[NL-1]$, and for the corresponding input $H_v \in \R^{D\times (T+N+v)}$, defined as
    \begin{align}
        H_v := \begin{cases}
            [H^\pro, H^\dat] & \ \text{ for $v=0$},\\
            [H^\pro, H^\dat, \h_{T+N+1}, \dots, \h_{T+N+v}] & \ \text{ for $v \in [NL-1]$},
        \end{cases}\label{eq: H}
    \end{align}
    we have 
    $$
        \TF_{\Theta^*}(H_v)_{T+N+v} = \h_{T+N+v+1}.
    $$
    We prove the existence of such a transformer by induction on $v$;
    specifically, we show that there exists a parameter $\Theta^*$ satisfying $\TF_{\Theta^*}(H_v)_{T+N+v} = \h_{T+N+v+1}$ for any fixed $v \in \{0\} \cup [NL-1]$.
    For brevity, let $n = T+N+v$ and $D = 4d+8$.\footnote{The proof of Theorem~\ref{thm: prompt engineering ap} requires only $D=2d+8$ token dimension. However, as we will see later, the proof of Theorem~\ref{thm: prompt engineering EUAF ap} requires $D=4d+8$ token dimension. For the convenience of presentation, we choose $D=4d+8$ throughout the paper.}
    Also denote $\ell^* := \xi(n-N+1) = \xi(T+v+1)$ and $i^* := \tau(n-N+1) = \tau(T+v+1)$. 
    Note that since $(\h_j)_{1:d} \in \mathcal{U} \subset \mathcal{B}_d(B)$ for $j \leq T$ and $\|\z_i\| \leq d^{1/2}$ for $i \in [N]$, we have $\|W^\pro_\ell\| \leq T B^2$ for $\ell \in [L]$, and consequently 
    $$
        \|\z_i^{(\ell)}\| \leq (T B^2)^\ell d^{1/2} \text{ for $i \in [N]$ and $\ell \in [L] \cup \{0\}$}.
    $$
    
    We outline the proof in three parts:
    \begin{enumerate}[label=\textbf{Part (\alph*).},leftmargin=5em]
        \item We design transformer layers such that the output includes $\y_{i^*}^{(\ell^*)}$. To accomplish this, we selectively apply ReLU activation to $\z_{i^*}^{(\ell^*)}$ based on the value of $\ell^*$.
        \item We construct layers that output $(\h_j)_{1:d}^\top \y_{i^*}^{(\ell^*)}$ for each $j \in [T]$.
        \item Finally, we build layers that aggregate the products $(\h_j)_{1:d} (\h_{j+1})_{1:d}^\top \y_{i^*}^{(\ell^*)}$ for $j \in [T]$.
    \end{enumerate}
    See Tables~\ref{table: construction a} and \ref{table: construction b} for sketches of the internal representations after Parts (a) and (b), respectively.

    \renewcommand{\arraystretch}{1.1}
    \begin{table}[H]
        \centering
        \begin{tabular}{|c||c||c|}
        \hline
        Index & Prompt tokens ($j \leq T$) & Data tokens ($j \geq T+1$)\\ \hline\hline
        $[d]$ & {\small $(\h_j)_{1:d}$} & {\small $\y_{\tau(j)}^{(\xi(j))}$} \\ \hline
        $\vdots$ & $\zero$ & $\zero$ \\ \hline
        $D-3$ & $1$ & $1$ \\ \hline
        $D-2$ & $S$ & $S$ \\ \hline
        $D-1$ & $S w_j$ & $-S \xi(j)$ \\ \hline
        $D$ & $S j$ & $S j$ \\ \hline
        \end{tabular}
        \caption{A sketch of the internal representations after Part (a).}
        \label{table: construction a}
    \end{table}

    \renewcommand{\arraystretch}{1.1}
    \begin{table}[H]
        \centering
        \begin{tabular}{|c||c||c|}
        \hline
        Index & Prompt tokens ($j \leq T$) & Data tokens ($j \geq T+1$)\\ \hline\hline
        $[d]$ & {\small $(\h_j)_{1:d}$} & {\small $\y_{i^*}^{(\ell^*)} \Id\{j=T+v+1\}$} \\ \hline
        $\vdots$ & $\zero$ & $\zero$ \\ \hline
        $D-5$ & $\alpha_j$ & $\alpha_j$ \\ \hline
        $D-4$ & $SN$ & $SN$ \\ \hline
        $D-3$ & $1$ & $1$ \\ \hline
        $D-2$ & $S$ & $S$ \\ \hline
        $D-1$ & $S w_j$ & $-S \xi(j)$ \\ \hline
        $D$ & $S j$ & $S j$ \\ \hline
        \end{tabular}
        \caption{A sketch of the internal representations after Part (b). Here $\alpha_j$ is defined in \eqref{eq: alpha}.}
        \label{table: construction b}
    \end{table}

    \paragraph{Part (a).}
    In this part we selectively apply ReLU activation function to $N$ data tokens $\z_1^{(0)}, \dots, \z_N^{(0)}$, and $v$ generated tokens $\z_1^{(1)}, \dots, \z_{\tau(n)}^{(\xi(n))}$.
    To this aim we construct a two-layer transformer $\TF_{(\theta_1^*,\theta_2^*)}$ that takes $H_v$, defined in \eqref{eq: H}, as input and outputs $H_v^{[2]} = [\h_1^{[2]}, \dots, \h_n^{[2]}] = \TF_{(\theta_1^*,\theta_2^*)}(H_v)$, where
    \begin{align}
        \h_j^{[2]} = \h_j \ \text{for $j \leq T$,\ \ and }
        \h_j^{[2]} = \begin{pmatrix}
            \y_{\tau(j)}^{(\xi(j))}\\
            \p_j
        \end{pmatrix} \ \text{for $j \geq T+1$}.\label{eq: h [2] j}
    \end{align}
    We set the parameters $\mu_1^* = \mu_\id$, and $\nu_1^* = (W_{1,1}^*, W_{1,2}^*)$ such that
    \begin{align*}
        W_{1,2}^* \sigma(W_{1,1}^* \h) &= W_{1,2}^* \sigma(-(\h)_{1:d})
        = \begin{pmatrix}
            \zero_d\\
            \sigma(-(\h)_{1:d})\\
            \zero_{D-2d}
        \end{pmatrix}.
    \end{align*}
    Let $\theta_1^* = (\mu_\id, \nu_1^*)$. Then,
    \begin{align*}
        H_v^{[1]} = [\h_1^{[1]}, \dots, \h_n^{[1]}] := \TF_{\theta_1^*}(H_v)_j = \h_j + \begin{pmatrix}
            \zero_d\\
            \sigma(-(\h)_{1:d})\\
            \zero_{D-2d}
        \end{pmatrix} = \begin{pmatrix}
            (\h_j)_{1:d}\\
            \sigma(-(\h_j)_{1:d})\\
            (\h_j)_{(2d+1):D}
        \end{pmatrix}.
    \end{align*}
    Next we set the parameter $\theta_2^* = (\mu_\id, \nu_2^*)$ with $\nu_2^* = (W_{2,1}^*, W_{2,2}^*)$ such that
    \begin{footnotesize}
    \begin{align*}
        &W_{2,2}^* \sigma(W_{2,1}^* \h)\\
        &\quad= W_{2,2}^* \begin{pmatrix}
            \sigma((1/2) (\h)_{(d+1):2d} - (\h)_{D-1} \1_d-(1/2)(\h)_{D-2} \1_d)\\
            \sigma(-(1/2) (\h)_{(d+1):2d} - (\h)_{D-1} \1_d-(1/2)(\h)_{D-2} \1_d)\\
            \sigma((\h)_{(d+1):(2d)})
        \end{pmatrix}\\
        &\quad= \begin{pmatrix}
            \sigma((1/2) (\h)_{(d+1):2d} - (\h)_{D-1}\1_d-(1/2)(\h)_{D-2}\1_d) - \sigma(-(1/2) (\h)_{(d+1):2d} - (\h)_{D-1}\1_d-(1/2)(\h)_{D-2}\1_d)\\
            -\sigma((\h)_{(d+1):2d})\\
            \zero_{D-2d}\\
        \end{pmatrix}.
    \end{align*}
    \end{footnotesize}\noindent
    Then, since $\|(\h_j)_{1:d}\|_{\max} \leq B \vee \|\z_{i^*}^{(\ell^*)}\| \leq B \vee d^{1/2} (T B^2)^{\ell^*} \leq S$ for all $j \in [n]$,
    \begin{small}
    \begin{align*}
        &\TF_{\theta_2^*}(H_v^{[1]})_j\\
        &\quad= \h_j^{[1]} + W_{2,2}^* \sigma(W_{2,1}^* \h_j^{[1]})\\
        &\quad= \begin{pmatrix}
            (\h_j)_{1:d}\\
            \sigma(-(\h_j)_{1:d})\\
            (\h_j)_{(2d+1):D}
        \end{pmatrix}\\
        &\quad\quad+ \begin{pmatrix}
            \sigma((1/2) \sigma(-(\h_j)_{1:d}) - S w_j \1_d-(1/2)S \1_d) - \sigma(-(1/2) \sigma(-(\h_j)_{1:d}) - S w_j \1_d-(1/2)S \1_d)\\
            -\sigma(\sigma(-(\h_j)_{1:d}))\\
            \zero_{D-2d}
        \end{pmatrix}\\
        &\quad= \begin{pmatrix}
            (\h_j)_{1:d}\\
            \sigma(-(\h_j)_{1:d})\\
            (\h_j)_{(2d+1):D}
        \end{pmatrix} + \begin{pmatrix}
            \psi_{d,S}(\sigma(-(\h_j)_{1:d}); 1, -w_j)\\
            -\sigma(-(\h_j)_{1:d})\\
            \zero_{D-2d}
        \end{pmatrix}\\
        &\quad= \begin{pmatrix}
            \sigma((\h_j)_{1:d}) \Id\{w_j \leq -1\} + (\h_j)_{1:d} \Id\{w_j \geq 0\}\\
            \zero_d\\
            (\h_j)_{(2d+1):D}
        \end{pmatrix}.
    \end{align*}    
    \end{small}\noindent
    By definition of $w_j=-\xi(j)$ for $j \geq T+1$, the first $d$ coordinates of $\TF_{\theta_2^*}(H_v^{[1]})_j$ applies ReLU activation function to $(\h_j)_{1:d}$ if $j \geq T + N+1$, and keeps $(\h_j)_{1:d}$ if $j \leq T+N$.
    This implies that $\TF_{(\theta_1^*, \theta_2^*)}(H_v)_j = \h_j$ for $j \leq T$ and $\TF_{(\theta_1^*,\theta_2^*)}(H_v)_j = (\y_{\tau(j)}^{(\xi(j)) \top}, (\h_j)_{(d+1):D}^\top)^\top$ for $j \geq T+1$.
    We can verify that $\theta_1^*$ and $\theta_2^*$ consist of constants, $\TF_{(\theta_1^*,\theta_2^*)}$ has $1$ head with FFN width of $O(d)$,
    \begin{align}
        &\|(\theta_1^*, \theta_2^*)\|_0 \lesssim d, \ \ \|(\theta_1^*, \theta_2^*)\|_{\max} \lesssim 1,\label{eq: theta 1 2 bound}\\
        & \text{ and } \max_{j \in [n]} \|(\h_j^{[2]})_{1:d}\|_{\max} \leq \max_{j \in [n]} \|(\h_j)_{1:d}\|_{\max} \leq B \vee d^{1/2} (T B^2)^{\ell^*}.\nonumber
    \end{align}

    \paragraph{Part (b).}
    In this part we zero out the data tokens $(\h_j)_{1:d}$ for indices $j \geq T+1$ with $\xi(j) \neq \ell^*$ or $\tau(j) \neq i^*$, i.e., we only keep $\y_{i^*}^{(\ell^*)}$ among data tokens and generated tokens.
    In the first half of Part (b), we construct a two-layer transformer $\TF_{(\theta_3^*,\theta_4^*)}$ that takes $H_v^{[2]}$ in \eqref{eq: h [2] j} as input and produces the output $H_v^{[4]} := \TF_{(\theta_3^*,\theta_4^*)}(H_v^{[2]}) =: [\h_1^{[4]}, \dots, \h_n^{[4]}]$, where
    \begin{align}
        \h_j^{[4]} &= \begin{pmatrix}
            (\h_j^{[2]})_{1:d} \delta_j\\
            \p_j^{[4]}
        \end{pmatrix}, \ \ \p_j^{[4]} := \begin{pmatrix}
            \zero_{D-d-5}\\
            S N\\
            1\\
            S\\
            S w_j\\
            S j
        \end{pmatrix}.\label{eq: h [4] j}
    \end{align}
    Here $\delta_j := \Id\{j \leq T\} + \Id\{j \geq T+1, \xi(j) = \ell^*, \tau(j) = i^*\}$.
    We first construct a transformer $\TF_{\theta_3^*}$ that outputs $H_v^{[3]} = [\h_1^{[3]}, \dots, \h_n^{[3]}] = \TF_{\theta_3^*}(H_v^{[2]})$, where
    \begin{align*}
        \h_j^{[3]} := \begin{pmatrix}
            (\h_j^{[2]})_{1:d}\\
            \p_j^{[3]}
        \end{pmatrix}, \ \ \p_j^{[3]} := \begin{pmatrix}
            \zero_{D-d-7}\\
            S \delta_j\\
            S (T+v)\\
            S N\\
            1\\
            S\\
            S w_j\\
            S j
        \end{pmatrix}.
    \end{align*}
    We set the parameters $\mu_3^* = \{(Q_{3,m}^*, K_{3,m}^*, V_{3,m}^*)\}_{m \in [8]}$ such that for a vector $\h \in \R^D$,
    \begin{align*}
        Q_{3,1}^* \h &= \begin{pmatrix}
            (\h)_{D-3}\\
            \zero_{D-1}
        \end{pmatrix}, 
        K_{3,1}^* \h = \begin{pmatrix}
            -(\h)_{D-1}\\
            \zero_{D-1}
        \end{pmatrix}, 
        Q_{3,2}^* \h = \begin{pmatrix}
            (\h)_{D-3}\\
            (\h)_{D-3}\\
            \zero_{D-2}
        \end{pmatrix},
        K_{3,2}^* \h = \begin{pmatrix}
            -(\h)_{D-2}\\
            -(\h)_{D-1}\\
            \zero_{D-2}
        \end{pmatrix},\\
        Q_{3,3}^* \h &= \begin{pmatrix}
            (\h)_{D-3}\\
            \zero_{D-1}
        \end{pmatrix}, 
        K_{3,3}^* \h = \begin{pmatrix}
            (\h)_{D-1}\\
            \zero_{D-1}
        \end{pmatrix}, 
        Q_{3,4}^* \h = \begin{pmatrix}
            (\h)_{D-3}\\
            (\h)_{D-3}\\
            \zero_{D-2}
        \end{pmatrix},
        K_{3,4}^* \h = \begin{pmatrix}
            -(\h)_{D-2}\\
            (\h)_{D-1}\\
            \zero_{D-2}
        \end{pmatrix},\\
        Q_{3,5}^* \h &= \begin{pmatrix}
            (\h)_{D-3}\\
            (\h)_{D-3}\\
            \zero_{D-2}
        \end{pmatrix}, 
        K_{3,5}^* \h = \begin{pmatrix}
            4(\h)_{D-1}\\
            3(\h)_{D-2}\\
            \zero_{D-2}
        \end{pmatrix}, 
        Q_{3,6}^* \h = \begin{pmatrix}
            (\h)_{D-3}\\
            (\h)_{D-3}\\
            \zero_{D-2}
        \end{pmatrix}, 
        K_{3,6}^* \h = \begin{pmatrix}
            4(\h)_{D-1}\\
            2(\h)_{D-2}\\
            \zero_{D-2}
        \end{pmatrix},\\
        Q_{3,7}^* \h &= \begin{pmatrix}
            (\h)_{D-3}\\
            \zero_{D-1}
        \end{pmatrix}, 
        K_{3,7}^* \h = \begin{pmatrix}
            4(\h)_{D-1}\\
            \zero_{D-1}
        \end{pmatrix}, 
        Q_{3,8}^* \h = \begin{pmatrix}
            (\h)_{D-3}\\
            (\h)_{D-3}\\
            \zero_{D-2}
        \end{pmatrix}, 
        K_{3,8}^* \h = \begin{pmatrix}
            4(\h)_{D-1}\\
            -(\h)_{D-2}\\
            \zero_{D-2}
        \end{pmatrix}, 
    \end{align*}
    and
    \begin{small}
    \begin{align*}
        V_{3,1}^* \h &= V_{3,3}^* \h = (\h)_{D-3} \e_{D-5},\ \ V_{3,2}^* \h = V_{3,4}^* \h = - (\h)_{D-3} \e_{D-5},\\
        V_{3,5}^* \h &= - (\h)_{D-3} \e_{D-4},\ \ V_{3,6}^* \h = 2 (\h)_{D-3} \e_{D-4},\ \ V_{3,7}^* \h = -2 (\h)_{D-3} \e_{D-4},\ \ V_{3,8}^* \h = (\h)_{D-3} \e_{D-4}.
    \end{align*}
    \end{small}\noindent
    Then, the output from the attention layer $\Attn_{\mu_3^*}$ is given by
    \begin{align*}
        &\Attn_{\mu_3^*}(H_v^{[2]})_j\\
        &\quad= \h_j^{[2]} + \sum_{j' \in [n]} \qty{\sigma\qty(S/2 + S \qty(-1-w_{j'}+1/2)) - \sigma\qty(-S/2 + S \qty(-1-w_{j'}+1/2))} \e_{D-5}\\
        &\quad\quad+ \sum_{j' \in [n]} \qty{\sigma\qty(S/2 + S \qty(w_{j'}-1+1/2)) - \sigma\qty(-S/2 + S \qty(w_{j'}-1+1/2))} \e_{D-5}\\
        &\quad\quad+ \sum_{j' \in [n]} \biggl\{-\sigma\qty(S + 4 S \qty(w_{j'}+1/2)) +2 \sigma\qty(S + 4 S \qty(w_{j'}+1/4))\\
        &\quad\quad- 2\sigma\qty(S + 4 S \qty(w_{j'}-1/4)) + \sigma\qty(S + 4 S \qty(w_{j'}-1/2))\biggr\} \e_{D-4}\\
        &\quad= \h_j^{[2]} + \sum_{j' \in [n]} \psi_{1,S}(S; w_{j'}, -1) \e_{D-5} + \sum_{j' \in [n]} \psi_{1,S}(S; 1, w_{j'}) \e_{D-5} + \sum_{j' \in [n]} \phi_{1,S}(S; 0, w_{j'}) \e_{D-4}\\
        &\quad= \h_j^{[2]} + S \sum_{j' \in [n]} \Id\{w_{j'} \leq -1\} \e_{D-5} + S \sum_{j' \in [n]} \Id\{w_{j'} \geq 1\} \e_{D-5} + S \sum_{j' \in [n]} \Id\{w_{j'} = 0\} \e_{D-4}\\
        &\quad= \begin{pmatrix}
            (\h_j^{[2]})_{1:d}\\
            \zero_{D-d-6}\\
            S (T+v)\\
            S N\\
            1\\
            S\\
            S w_j\\
            S j
        \end{pmatrix},
    \end{align*}
    where the last equality follows since $w_{j'} \leq -1$ if and only if $j' \geq T+1$ and $\xi(j') \geq 1$, which is equivalent to $j' \geq T+N+1$ by definition of $\xi$ in \eqref{eq: xi and tau}.
    We then set $\nu_3^* = (W_{3,1}^*, W_{3,2}^*)$ such that
    \begin{align*}
        W_{3,2}^* \sigma(W_{3,1}^* \h) &= W_{3,2}^* \begin{pmatrix}
            \sigma(7 (\h)_{D-2} + 4 (\h)_{D-5} - 4 (\h)_D)\\
            \sigma(6 (\h)_{D-2} + 4 (\h)_{D-5} - 4 (\h)_D)\\
            \sigma(4 (\h)_{D-2} + 4 (\h)_{D-5} - 4 (\h)_D)\\
            \sigma(3 (\h)_{D-2} + 4 (\h)_{D-5} - 4 (\h)_D)\\
            \sigma(  (\h)_{D-1})\\
            \sigma(- (\h)_{D-2} +   (\h)_{D-1})
        \end{pmatrix}\\
        &= \bigl\{-\sigma(7 (\h)_{D-2} + 4 (\h)_{D-5} - 4 (\h)_D) + 2\sigma(6 (\h)_{D-2} + 4 (\h)_{D-5} - 4 (\h)_D)\\
        &\quad- 2\sigma(4 (\h)_{D-2} + 4 (\h)_{D-5} - 4 (\h)_D) + \sigma(3 (\h)_{D-2} + 4 (\h)_{D-5} - 4 (\h)_D)\bigr\} \e_{D-6}\\
        &\quad+ \bigl\{\sigma((\h)_{D-1}) - \sigma(-(\h)_{D-2} + (\h)_{D-1})\bigr\} \e_{D-6}.
    \end{align*}
    Let $\theta_3^* = (\mu_3^*, \nu_3^*)$. Then,
    \begin{align*}
        \TF_{\theta_3^*}(H_v^{[2]})_j &= \begin{pmatrix}
            (\h_j^{[2]})_{1:d}\\
            \zero_{D-d-6}\\
            S (T+v)\\
            S N\\
            1\\
            S\\
            S w_j\\
            S j
        \end{pmatrix}\\
        &\quad + \biggl\{- \sigma\qty(S + 4 S \qty(T+v+1-j + 1/2)) + 2\sigma\qty(S + 4 S \qty(T+v+1-j + 1/4))\\
        &\quad- 2\sigma\qty(S + 4 S \qty(T+v+1-j - 1/4)) + \sigma\qty(S + 4 S \qty(T+v+1-j - 1/2)) \biggr\} \e_{D-6}\\
        &\quad+ \biggl\{ \sigma\qty(S/2 + S \qty(w_j - 1 + 1/2)) - \sigma\qty(-S/2 + S \qty(w_j - 1 + 1/2)) \biggr\} \e_{D-6}\\
        &= \begin{pmatrix}
            (\h_j^{[2]})_{1:d}\\
            \zero_{D-d-6}\\
            S (T+v)\\
            S N\\
            1\\
            S\\
            S w_j\\
            S j
        \end{pmatrix} + \qty{\phi_{1,S}(S; j, T+v+1) + \psi_{1,S}(S; 1, w_j)} \e_{D-6}.
    \end{align*}
    Note that $\TF_{\theta_3^*}$ is the desired transformer, since 
    \begin{align*}
        \phi_{1,S}(S; j, T+v+1) + \psi_{1,S}(S; 1, w_j) &= S\Id\{w_j \geq 1\} + S\Id\{j=T+v+1\}\\
        &= S\Id\{j \leq T\} + S\Id\{j \geq T+1, \xi(j) = \ell^*, \tau(j) = i^*\} = S\delta_j,
    \end{align*}
    which follows from the fact that $j=T+v+1$ if and only if $\xi(j) = \ell^*$ and $\tau(j) = i^*$.
    Next we set the parameter $\theta_4^* = (\mu_\id, \nu_4^*)$ with $\nu_4^* = (W_{4,1}^*, W_{4,2}^*)$ such that
    \begin{small}
    \begin{align*}
        W_{4,2}^* \sigma(W_{4,1}^* \h) &= W_{4,2}^* \begin{pmatrix}
            \sigma((\h)_{1:d} + (- 4(\h)_{D-6} + 2(\h)_{D-2}) \1_d)\\
            \sigma((\h)_{1:d} + (- 4(\h)_{D-6} + (\h)_{D-2}) \1_d)\\
            \sigma((\h)_{1:d} + (- 4(\h)_{D-6} - (\h)_{D-2}) \1_d)\\
            \sigma((\h)_{1:d} + (- 4(\h)_{D-6} - 2(\h)_{D-2}) \1_d)\\
            \sigma((\h)_{D-5})\\
            \sigma((\h)_{D-6})
        \end{pmatrix}\\
        &= \begin{pmatrix}
            \sigma((\h)_{1:d} + (- 4(\h)_{D-6} + 2(\h)_{D-2}) \1_d) - 2\sigma((\h)_{1:d} + (- 4(\h)_{D-6} + (\h)_{D-2}) \1_d)\\
            \zero_{D-d}
        \end{pmatrix}\\
        &\quad+ \begin{pmatrix}
            2\sigma((\h)_{1:d} + (- 4(\h)_{D-6} - (\h)_{D-2}) \1_d) - \sigma((\h)_{1:d} + (- 4(\h)_{D-6} - 2(\h)_{D-2}) \1_d)\\
            \zero_{D-d}
        \end{pmatrix}\\
        &\quad- \sigma((\h)_{D-5}) \e_{D-5} - \sigma((\h)_{D-6}) \e_{D-6}.
    \end{align*}
    \end{small}\noindent
    Then, for any $j \in [n]$,
    \begin{align*}
        \TF_{\theta_4^*}(H_v^{[3]})_j &= \h_j^{[3]} - \begin{pmatrix}
            -\sigma\qty((\h_j^{[3]})_{1:d} + 4 S \qty(-\delta_j + 1/2) \1_d) + 2\sigma\qty((\h_j^{[3]})_{1:d} + 4 S \qty(-\delta_j + 1/4) \1_d)\\
            \zero_{D-d}
        \end{pmatrix}\\
        &\quad- \begin{pmatrix}
            -2\sigma\qty((\h_j^{[3]})_{1:d} + 4 S \qty(-\delta_j - 1/4) \1_d) + \sigma\qty((\h_j^{[3]})_{1:d} + 4 S \qty(-\delta_j - 1/2) \1_d)\\
            \zero_{D-d}
        \end{pmatrix}\\
        &\quad - (\h_j^{[3]})_{D-5} \e_{D-5} - (\h_j^{[3]})_{D-6} \e_{D-6}\\
        &= \h_j^{[3]} - \begin{pmatrix}
            \phi_{d,S}((\h_j^{[3]})_{1:d}; \delta_j, 0)\\
            \zero_{D-d}
        \end{pmatrix} - (\h_j^{[3]})_{D-5} \e_{D-5} - (\h_j^{[3]})_{D-6} \e_{D-6}\\
        &= \begin{pmatrix}
            (\h_j^{[2]})_{1:d} \delta_j\\
            \p_j^{[4]}
        \end{pmatrix},
    \end{align*}
    where the last equality follows since $(\h_j^{[3]})_{1:d} - (\h_j^{[3]})_{1:d} \Id\{\delta_j = 0\} = (\h_j^{[3]})_{1:d} \delta_j = (\h_j^{[2]})_{1:d} \delta_j$, and $\|(\h_j^{[3]})_{1:d}\|_{\max} = \|(\h_j^{[2]})_{1:d}\|_{\max} \leq B \vee d^{1/2} (T B^2)^{\ell^*} \leq S$.
    We can verify that $\theta_3^*$ and $\theta_4^*$ consist of constants, $\TF_{(\theta_3^*,\theta_4^*)}$ has $8$ heads with FFN width of $O(d)$, and
    \begin{align}
        &\|(\theta_3^*, \theta_4^*)\|_0 \lesssim d, \ \ \|(\theta_3^*, \theta_4^*)\|_{\max} \lesssim 1,\label{eq: theta 3 4 bound}\\
        & \text{ and } \max_{j \in [n]} \|(\h_j^{[4]})_{1:d}\|_{\max} \leq \max_{j \in [n]} \|(\h_j^{[2]})_{1:d}\|_{\max} \leq B \vee d^{1/2} (T B^2)^{\ell^*}.\nonumber
    \end{align}

    In the latter half of Part (b) we construct a transformer $\TF_{\theta_5^*}$ that takes $H_v^{[4]}$ in \eqref{eq: h [4] j} as input and outputs $H_v^{[5]} = [\h_1^{[5]}, \dots, \h_n^{[5]}] := \TF_{\theta_5^*}(H_v^{[4]})$, where
    \begin{align}
        \h_j^{[5]} = \begin{pmatrix}
            (\h_j^{[4]})_{1:d}\\
            \p_j^{[5]}
        \end{pmatrix}, \ \ 
        \p_j^{[5]} = \begin{pmatrix}
            \zero_{D-d-6}\\
            \alpha_j\\
            S N\\
            1\\
            S\\
            S w_j\\
            S j
        \end{pmatrix} \ \ \text{ for $j \in [n]$}.\label{eq: h [5] j}
    \end{align}
    with
    \begin{align}
        \alpha_j := (\h_j^{[4]})_{1:d}^\top \sum_{j' \in [n]: w_{j'} = 1 - w_j/2} (\h_{j'}^{[4]})_{1:d}.\label{eq: alpha}
    \end{align}
    We set the parameters $\mu_5^* = \{(Q_{5,m}^*, K_{5,m}^*, V_{5,m}^*)\}_{m \in [4]}$ such that for a vector $\h \in \R^D$,
    \begin{small}
    \begin{align*}
        Q_{5,1}^* \h &= \begin{pmatrix}
            (\h)_{1:d}\\
            (\h)_{D-3}\\
            (\h)_{D-3}\\
            -4(\h)_{D-1}\\
            \zero_{D-d-3}
        \end{pmatrix}, 
        K_{5,1}^* \h = \begin{pmatrix}
            (\h)_{1:d}\\
            -8(\h)_{D-1}\\
            10(\h)_{D-2}\\
            (\h)_{D-3}\\
            \zero_{D-d-3}
        \end{pmatrix}, 
        Q_{5,2}^* \h = \begin{pmatrix}
            (\h)_{1:d}\\
            (\h)_{D-3}\\
            (\h)_{D-3}\\
            -4(\h)_{D-1}\\
            \zero_{D-d-3}
        \end{pmatrix},
        K_{5,2}^* \h = \begin{pmatrix}
            (\h)_{1:d}\\
            -8(\h)_{D-1}\\
            9(\h)_{D-2}\\
            (\h)_{D-3}\\
            \zero_{D-d-3}
        \end{pmatrix},\nonumber\\
        Q_{5,3}^* \h &= \begin{pmatrix}
            (\h)_{1:d}\\
            (\h)_{D-3}\\
            (\h)_{D-3}\\
            -4(\h)_{D-1}\\
            \zero_{D-d-3}
        \end{pmatrix}, 
        K_{5,3}^* \h = \begin{pmatrix}
            (\h)_{1:d}\\
            -8(\h)_{D-1}\\
            7(\h)_{D-2}\\
            (\h)_{D-3}\\
            \zero_{D-d-3}
        \end{pmatrix}, 
        Q_{5,4}^* \h = \begin{pmatrix}
            (\h)_{1:d}\\
            (\h)_{D-3}\\
            (\h)_{D-3}\\
            -4(\h)_{D-1}\\
            \zero_{D-d-3}
        \end{pmatrix},
        K_{5,4}^* \h = \begin{pmatrix}
            (\h)_{1:d}\\
            -8(\h)_{D-1}\\
            6(\h)_{D-2}\\
            (\h)_{D-3}\\
            \zero_{D-d-3}
        \end{pmatrix},
    \end{align*}
    \end{small}\noindent
    and
    \begin{small}
    \begin{align*}
        V_{5,1}^* \h &= - (\h)_{D-3} \e_{D-5},\ \ V_{5,2}^* \h = 2 (\h)_{D-3} \e_{D-5},\ \ V_{5,3}^* \h = -2 (\h)_{D-3} \e_{D-5},\ \ V_{5,4}^* \h = (\h)_{D-3} \e_{D-5}.
    \end{align*}
    \end{small}\noindent
    Then, the output from the attention layer $\Attn_{\mu_5^*}$ is given by
    \begin{small}
    \begin{align*}   
        &\Attn_{\mu_5^*}(H_v^{[4]})_j\\
        &\quad= \h_j^{[4]} + \sum_{j' \in [n]} \biggl\{ -\sigma\qty(\alpha_{j,j'} + 4 S \qty( - 2 w_{j'} + 2 - w_j + 1/2)) + 2 \sigma\qty(\alpha_{j,j'} + 4 S \qty( - 2 w_{j'} + 2 - w_j + 1/4))\\
        &\quad\quad -2 \sigma\qty(\alpha_{j,j'} + 4 S \qty( - 2 w_{j'} + 2 - w_j - 1/4)) + \sigma\qty(\alpha_{j,j'} + 4 S \qty( - 2 w_{j'} + 2 - w_j - 1/2)) \biggr\} \e_{D-5}\\
        &\quad= \h_j^{[4]} + \sum_{j' \in [n]} \phi_{1,S}(\alpha_{j,j'}; w_j, -2w_{j'}+2) \e_{D-5},
    \end{align*}
    \end{small}\noindent
    where $\alpha_{j,j'} := (\h_j^{[4]})_{1:d}^\top (\h_{j'}^{[4]})_{1:d}$. 
    Note that $|\alpha_{j,j'}| \leq (B \vee d^{1/2} (T B^2)^{\ell^*})^2 \leq S$ for all $j, j' \in [n]$. This gives
    \begin{align*}
        \Attn_{\mu_5^*}(H_v^{[4]})_j &= \h_j^{[4]} + \sum_{j' \in [n]} \alpha_{j,j'} \Id\{w_j=-2w_{j'}+2\} \e_{D-5}\\
        &= \h_j^{[4]} + \alpha_j \e_{D-5}.
    \end{align*}
    Furthermore, by definition of $\alpha_j$ in \eqref{eq: alpha} and the fact that $w_j \geq 1$ for $j \in [T]$, it follows that $\max_{j \in [T]} |\alpha_j| \leq B \max_{\ell\in\{0\}\cup[\ell^*], i \in [N]}\|\y_i^{(\ell)}\|$. Similarly, $\max_{j \in [n]\setminus[T]} |\alpha_j| \leq \max_{\ell\in\{0\}\cup[\ell^*],i \in [N]}\|\y_i^{(\ell)}\| BT$.
    Choosing $\theta_5^* = (\mu_5^*, \nu_\id)$ gives the desired transformer layer.
    We can verify that $\theta_5^*$ consists of constants, $\TF_{\theta_5^*}$ has $4$ heads with FFN width $0$, and
    \begin{align}
        \|\theta_5^*\|_0 \lesssim d, \ \ \|\theta_5^*\|_{\max} \lesssim 1, \ \ \text{and} \ \ \max_{j \in [n]} |(\h_j^{[5]})_{D-5}| = \max_{j\in[n]} |\alpha_j| \leq T B d^{1/2} (T B^2)^{\ell^*} \leq S.\label{eq: theta 5 bound}
    \end{align}
    
    \paragraph{Part (c).}
    In this part, we construct a transformer $\TF_{\theta_6^*}$ that takes $H_v^{[5]}$ in \eqref{eq: h [5] j} from Part (b) as input and outputs $H_v^{[6]} = [\h_1^{[6]}, \dots, \h_n^{[6]}] := \TF_{\theta_6^*}(H_v^{[5]})$, where
    \begin{align*}
        \h_j^{[6]} &= \begin{pmatrix}
            (\h_j^{[5]})_{1:d}\\
            \p_j^{[6]}
        \end{pmatrix}.
    \end{align*}
    with
    \begin{align}
        \p_j^{[6]} &= \begin{pmatrix}
            \zero_{D-d-7}\\
            *\\
            \alpha_{j+1}\\
            S N\\
            1\\
            S\\
            S w_j\\
            S j
        \end{pmatrix} \ \ \text{for $j \leq T$},\ \ 
        \p_j^{[6]} = \begin{pmatrix}
            \zero_{D-d-7}\\
            *\\
            *\\
            S N\\
            1\\
            S\\
            S w_j\\
            S j
        \end{pmatrix} \ \ \text{ for $T+1 \leq j \leq T+N-1$},\\
        &\text{ and }
        \p_j^{[6]} = \begin{pmatrix}
            \zero_{D-d-7}\\
            S w_{j+1}\\
            *\\
            S N\\
            1\\
            S\\
            S w_j\\
            S j
        \end{pmatrix} \ \ \text{for $T+N \leq j$}.\label{eq: h [6] j}
    \end{align}
    Here the entry $*$ varies depending on the context, but its actual value does not affect the claim.
    We set the parameters $\mu_6^* = \{(Q_{6,m}^*, K_{6,m}^*, V_{6,m}^*)\}_{m \in [8]}$ similar to Part (b) so that the output from the attention layer $\Attn_{\mu_6^*}$ becomes
    \begin{footnotesize}
    \begin{align*}
        &\Attn_{\mu_6^*}(H_v^{[5]})_j\\
        &\quad= \h_j^{[5]} + \sum_{j' \in [n]} \biggl\{-\sigma\qty((\h_{j'}^{[5]})_{D-5} + 4S \qty(j'- (j+1) + 1/2)) + 2  \sigma\qty((\h_{j'}^{[5]})_{D-5} + 4S \qty(j'- (j+1) + 1/4))\\
        &\quad\quad-2\sigma\qty((\h_{j'}^{[5]})_{D-5} + 4S \qty(j'- (j+1) - 1/4)) + \sigma\qty((\h_{j'}^{[5]})_{D-5} + 4S \qty(j'- (j+1) - 1/2)) \biggr\} \e_{D-7}\\
        &\quad\quad+ \sum_{j' \in [n]} \biggl\{-\sigma\qty(S w_{j'} - S + 4(S)^2 \qty(j'- (j-N+1) + 1/2)) + 2  \sigma\qty(S w_{j'} - S + 4(S)^2 \qty(j'- (j-N+1) + 1/4))\\
        &\quad\quad- 2\sigma\qty(S w_{j'} - S + 4(S)^2 \qty(j'- (j-N+1) - 1/4)) + \sigma\qty(S w_{j'} - S + 4(S)^2 \qty(j'- (j-N+1) - 1/2)) \biggr\} \e_{D-6}\\
        &\quad= \h_j^{[5]} + \sum_{j' \in [n]} \phi_{1,S}((\h_{j'}^{[5]})_{D-5}; j+1, j') \e_{D-7} + \sum_{j' \in [n]} \phi_{1,(S)^2}(S w_{j'} - S; j-N+1, j') \e_{D-6}.
    \end{align*}
    \end{footnotesize}\noindent
    Note that $|(\h_{j'}^{[5]})_{D-5}| \leq S$ and $|S w_{j'} - S| \leq 2SL \leq (S)^2$ holds for all $j' \in [n]$. Hence,
    \begin{align*}
        \Attn_{\mu_6^*}(H_v^{[5]})_j &= \h_j^{[5]} + \alpha_{j+1} \Id\{j \in [n-1]\} \e_{D-7} + S (w_{j-N+1} - 1) \Id\{j\in [n]\setminus[N-1]\} \e_{D-6}.
    \end{align*}
    For $j \leq T$,
    \begin{align*}
        \Attn_{\mu_6^*}(H_v^{[5]})_j &= \begin{pmatrix}
            (\h_j^{[5]})_{1:d}\\
            \p_j^{[5]}
        \end{pmatrix} + \begin{pmatrix}
            \zero_{D-8}\\
            \alpha_{j+1}\\
            *\\
            \zero_6
        \end{pmatrix}.
    \end{align*}
    Also, for $j \geq T+N$, it holds that
    \begin{align*}
        \Attn_{\mu_6^*}(H_v^{[5]})_j = \begin{pmatrix}
            (\h_j^{[5]})_{1:d}\\
            \p_j^{[5]}
        \end{pmatrix} + \begin{pmatrix}
            \zero_{D-8}\\
            *\\
            S (w_{j-N+1}-1)\\
            \zero_6
        \end{pmatrix} = \begin{pmatrix}
            (\h_j^{[5]})_{1:d}\\
            \p_j^{[5]}
        \end{pmatrix} + \begin{pmatrix}
            \zero_{D-8}\\
            *\\
            S w_{j+1}\\
            \zero_6
        \end{pmatrix},
    \end{align*}
    where the last equality follows since $w_{j-N+1}-1=-\xi(j-N+1)-1=-\xi(j+1)=w_{j+1}$ for $j \geq T+N$.
    We then construct a feed-forward neural network $\FFN_{\nu_6^*}$ with parameter $\nu_6^* = (W_{6,1}^*, W_{6,2}^*)$ such that
    \begin{align*}
        W_{6,2}^* \sigma(W_{6,1}^* \h) &= \begin{pmatrix}
            \zero_{D-8}\\
            -(\h)_{D-7}\\
            0\\
            -(\h)_{D-5} + (\h)_{D-7}\\
            \zero_5
        \end{pmatrix}.
    \end{align*}
    Choosing $\theta_6^* = (\mu_6^*, \nu_6^*)$ yields the desired transformer $\TF_{\theta_6^*}$ which outputs
    \begin{align*}
        \TF_{\theta_6^*}(H_v^{[5]})_j & = \begin{pmatrix}
            (\h_j^{[5]})_{1:d}\\
            \p_j^{[6]}
        \end{pmatrix} \text{ for $j \leq T$ or $j \geq T+N$ }.
    \end{align*}
    Furthermore, since $\TF_{\theta_6^*}(H_v^{[5]})$ only changes the $(D-5)$-th and $(D-6)$-th rows of $H_v^{[5]}$, we obtain the desired output for $T+1 \leq j \leq T+N-1$.
    Note that the bound for each $(D-5)$-th entry of $\h_j^{[6]}$ is the same as the bound for the corresponding entry of $\h_j^{[5]}$, since $(\h_j^{[6]})_{D-5} = (\h_{j+1}^{[5]})_{D-5} \Id\{j \leq n-1\}$.
    We can verify that $\theta_6^*$ consists of constants, $\TF_{\theta_6^*}$ has $8$ heads with FFN width of $O(1)$, and
    \begin{align}
        \|\theta_6^*\|_0 \lesssim 1, \ \ \|\theta_6^*\|_{\max} \lesssim 1, \ \ \text{and} \ \ \max_{j \in [n]} |(\h_j^{[6]})_{D-5}| \leq T B (B \vee d^{1/2} (T B^2)^{\ell^*}) \leq S.\label{eq: h 6 bound}
    \end{align}
    
    We then construct a transformer $\TF_{\theta_7^*}$ that takes $H_v^{[6]}$ in \eqref{eq: h [6] j} as input and outputs $H_v^{[7]} = [\h_1^{[7]}, \dots, \h_n^{[7]}] := \TF_{\theta_7^*}(H_v^{[6]})$ such that
    \begin{align*}
        \h_n^{[7]} = \begin{pmatrix}
            W^\pro_{\ell^*+1} \y_{i^*}^{(\ell^*)}\\
            \p_{n+1}
        \end{pmatrix}.
    \end{align*}
    We set the parameters $\mu_7^* = \{(Q_{7,m}^*, K_{7,m}^*, V_{7,m}^*)\}_{m \in [4]}$ as in Part (b) so that the output from the attention layer $\Attn_{\mu_7^*}$ becomes
    \begin{small}
    \begin{align*}
        \Attn_{\mu_7^*}(H_v^{[6]})_j
        &= \h_j^{[6]} + \sum_{j' \in [n]} \biggl\{ - \sigma\qty((\h_{j'}^{[6]})_{D-5} + 4S \qty( w_{j'} + 2 (\h_j^{[6]})_{D-6}/S + 1 + 1/2))\\
        &\quad+ 2 \sigma\qty((\h_{j'}^{[6]})_{D-5} + 4S \qty( w_{j'} + 2 (\h_j^{[6]})_{D-6}/S + 1 + 1/4))\\
        &\quad- 2 \sigma\qty((\h_{j'}^{[6]})_{D-5} + 4S \qty( w_{j'} + 2 (\h_j^{[6]})_{D-6}/S + 1 - 1/4))\\
        &\quad+ \sigma\qty((\h_{j'}^{[6]})_{D-5} + 4S \qty( w_{j'} + 2 (\h_j^{[6]})_{D-6}/S + 1 - 1/2)) \biggr\} \begin{pmatrix}
            \zero_d\\
            (\h_{j'}^{[6]})_{1:d}\\
            \zero_{D-2d}
        \end{pmatrix}\\
        &= \h_j^{[6]} + \sum_{j' \in [n]} \phi_{1,S}((\h_{j'}^{[6]})_{D-5}; -2(\h_j^{[6]})_{D-6}/S-1, w_{j'}) \begin{pmatrix}
            \zero_d\\
            (\h_{j'}^{[6]})_{1:d}\\
            \zero_{D-2d}
        \end{pmatrix}.
    \end{align*}
    \end{small}\noindent
    From \eqref{eq: h 6 bound}, we have $|(\h_{j'}^{[6]})_{D-5}| \leq S$ for all $j' \in [n]$.
    We focus on the $n$-th output token $\Attn_{\mu_7^*}(H_v^{[6]})_n$. 
    Note that
    \begin{align*}
        (\h_n^{[6]})_{D-6} = S w_{n+1} = -S\xi(n+1) = -S\floor{(N+v)/N} = - S(1 + \floor{v/N}) = -S (1+\ell^*),
    \end{align*}
    and thus
    \begin{align}
        \phi_{1,S}((\h_{j'}^{[6]})_{D-5}; -2(\h_n^{[6]})_{D-6}/S -1, w_{j'}) (\h_{j'}^{[6]})_{1:d} &= (\h_{j'}^{[6]})_{1:d} (\h_{j'}^{[6]})_{D-5} \Id\{w_{j'}=2\ell^*+1\}.\label{eq: phi j'}
    \end{align}
    Note that for $j' \in [n] \setminus [T]$, \eqref{eq: phi j'} becomes $0$, since $w_{j'} \leq 0$ for $j' \geq T+1$ and $\ell^* \geq 0$.
    Then, the last output token from $\Attn_{\mu_7^*}(H_v^{[6]})$ becomes
    \begin{align*}
        \Attn_{\mu_7^*}(H_v^{[6]})_n &= \h_n^{[6]} + \begin{pmatrix}
            \zero_d\\
            \sum_{j' \in [T]} (\h_{j'}^{[6]})_{1:d} (\h_{j'}^{[6]})_{D-5} \Id\{w_{j'}=2\ell^*+1\}\\
            \zero_{D-2d}
        \end{pmatrix}\\
        &= \begin{pmatrix}
            \y_{\tau(n)}^{(\xi(n))} \Id\{\xi(n) = \ell^*,\tau(n) = i^*\}\\
            \sum_{j' \in [T]} (\h_{j'}^{[6]})_{1:d} \alpha_{j'+1} \Id\{w_{j'}=2\ell^*+1\}\\
            \zero_{D-2d-7}\\
            S w_{n+1}\\
            *\\
            SN\\
            1\\
            S\\
            S w_n\\
            S n
        \end{pmatrix}.
    \end{align*}
    We then construct a feed-forward neural network $\FFN_{\nu_7^*}$ with parameter $\nu_7^* = (W_{7,1}^*, W_{7,2}^*)$ as in Part (b) such that
    \begin{align*}
        W_{7,2}^* \sigma(W_{7,1}^* \h) &= \begin{pmatrix}
            -(\h)_{1:d} + (\h)_{(d+1):(2d)}\\
            -(\h)_{(d+1):(D-4)}\\
            \zero_2\\
            -(\h)_{D-1}+(\h)_{D-6}\\
            (\h)_{D-2}
        \end{pmatrix}.
    \end{align*}
    Let $\theta_7^* = (\mu_7^*, \nu_7^*)$.
    This gives
    \begin{align*}
        \TF_{\theta_7^*}(H_v^{[6]})_n &= \begin{pmatrix}
            \sum_{j' \in [T]} (\h_{j'}^{[6]})_{1:d} \alpha_{j'+1} \Id\{w_{j'}=2\ell^*+1\}\\
            \zero_{D-d-4}\\
            1\\
            S\\
            S w_{n+1}\\
            S (n+1)
        \end{pmatrix}\\
        &= \begin{pmatrix}
            \sum_{j' \in [T]} (\h_{j'}^{[6]})_{1:d} \alpha_{j'+1} \Id\{w_{j'}=2\ell^*+1\}\\
            \p_{n+1}
        \end{pmatrix}.
    \end{align*}
    By definition of $\alpha_j$ in \eqref{eq: alpha}, it follows that
    \begin{align*}
        (\h_n^{[7]})_{1:d} &= \sum_{j' \in [T]} (\h_{j'}^{[6]})_{1:d} \alpha_{j'+1} \Id\{w_{j'}=2\ell^*+1\}\\
        &= \sum_{j' \in [T]} \Id\{w_{j'}=2\ell^*+1\} (\h_{j'}^{[4]})_{1:d} (\h_{j'+1}^{[4]})_{1:d}^\top \sum_{j'' \in [n]: w_{j''} = 1 - w_{j'+1}/2} (\h_{j''}^{[4]})_{1:d}\\
        &= \sum_{j' \in [T-1]} \Id\{w_{j'}=2\ell^*+1\} (\h_{j'}^{[4]})_{1:d} (\h_{j'+1}^{[4]})_{1:d}^\top \sum_{j'' \in [n]: w_{j''} = 1 - w_{j'+1}/2} (\h_{j''}^{[4]})_{1:d}\\
        &\quad+  \Id\{w_T=2\ell^*+1\}(\h_T^{[4]})_{1:d} (\h_{T+1}^{[4]})_{1:d}^\top \sum_{j'' \in [n]: w_{j''} = 1 - w_{T+1}/2} (\h_{j''}^{[4]})_{1:d}\\
        &= \sum_{j' \in [T-1]} \Id\{w_{j'}=2\ell^*+1\} (\h_{j'}^{[4]})_{1:d} (\h_{j'+1}^{[4]})_{1:d}^\top \sum_{j'' \in [n] \setminus [T]: w_{j''}=1-w_{j'+1}/2} (\h_{j''}^{[4]})_{1:d}\\
        &\quad+ \Id\{w_T=2\ell^*+1\} (\h_T^{[4]})_{1:d} (\h_{T+1}^{[4]})_{1:d}^\top \sum_{j'' \in [T]: w_{j''} = 1} (\h_{j''}^{[4]})_{1:d},
    \end{align*}
    where the last equality follows since $w_{j'+1} \geq 1$ for $j' \in [T-1]$, and $w_{T+1} = 0$. Furthermore, since $(\h_j^{[4]})_{1:d} = (\h_j^{[2]})_{1:d} \delta_j = (\h_j)_{1:d}$ for $j \in [T]$ and $(\h_j^{[4]})_{1:d} = \y_{i^*}^{(\ell^*)} \Id\{\xi(j) = \ell^*, \tau(j) = i^*\}$ for $j \in [n]\setminus[T]$, we obtain
    \begin{align*}
        (\h_n^{[7]})_{1:d} &= \sum_{j' \in [T-1]} \Id\{w_{j'}=2\ell^*+1\} (\h_{j'})_{1:d} (\h_{j'+1})_{1:d}^\top\\
        &\quad\quad\cdot \sum_{j'' \in [n] \setminus [T]: w_{j''}=1-w_{j'+1}/2} \y_{i^*}^{(\ell^*)} \Id\{\xi(j'') = \ell^*, \tau(j'') = i^*\}\\
        &\quad+ \Id\{w_T=2\ell^*+1\} (\h_T)_{1:d} \y_{i^*}^{(\ell^*) \top} \Id\{\xi(T+1) = \ell^*, \tau(T+1) = i^*\} \sum_{j'' \in [T]: w_{j''} = 1} (\h_{j''})_{1:d}\\
        &= \sum_{j' \in [T-1]} (\h_{j'})_{1:d} \Id\{w_{j'}=2\ell^*+1\} (\h_{j'+1})_{1:d}^\top \y_{i^*}^{(\ell^*)}\\
        &\quad\quad\cdot \sum_{j'' \in [n] \setminus [T]} \Id\{w_{j''}=1-w_{j'+1}/2, \xi(j'') = \ell^*, \tau(j'') = i^*\}\\
        &\quad+ (\h_T)_{1:d} \Id\{w_T=1\} \sum_{j'' \in [T]: w_{j''} = 1} (\h_{j''})_{1:d}^\top \y_{i^*}^{(\ell^*)} \Id\{\ell^*=0, i^*=1\}\\
        &= \biggl\{\sum_{j' \in [T-1]} (\h_{j'})_{1:d} \Id\{w_{j'}=2\ell^*+1\} (\h_{j'+1})_{1:d}^\top \Id\{ w_{j'+1}=2\ell^*+2\}\\
        &\quad+ \Id\{\ell^*=0, i^*=1\}(\h_T)_{1:d} \Id\{w_T=1\} \sum_{j' \in [T]} (\h_{j'})_{1:d}^\top \Id\{w_{j'} = 1\}\biggr\} \y_{i^*}^{(\ell^*)} ,
    \end{align*}
    where the second equality follows since $\xi(T+1)=0$ and $\tau(T+1)=1$, and the third equality follows since $w_{j''} = -\xi(j'')$ for $j'' \geq T+1$, and there exists a unique $j'' \in [n]\setminus[T]$ so that $\xi(j'')=\ell^*$ and $\tau(j'')=i^*$. By definition of $W^\pro_\ell$ and $W^\pro_{1,i}$ in \eqref{eq: W ell random}, we obtain 
    \begin{align*}
        (\h_n^{[7]})_{1:d} = \begin{cases}
             W^\pro_{\ell^*+1} \y_{i^*}^{(\ell^*)} & \text{ when $\ell^* \geq 1$},\\
             W^\pro_{1,i^*} \y_{i^*}^{(\ell^*)} & \text{ when $\ell^* = 0$},
        \end{cases}
    \end{align*}
    and hence $(\h_{n+1})_{1:d} = (\h_n^{[7]})_{1:d} = \z_{i^*}^{(\ell^*+1)}$. Since $\ell^*+1=\xi(n-N+1)+1=\xi(n+1)$, and $i^* = \tau(n-N+1) = \tau(n+1)$ by definition of $\xi$ and $\tau$ in \eqref{eq: xi and tau}, we finally obtain
    \begin{align*}
        \TF_{\theta_7^*}(H_v^{[6]})_n = \begin{pmatrix}
            \z_{i^*}^{(\ell^*+1)}\\
            \zero_{D-d-4}\\
            1\\
            S\\
            S w_{n+1}\\
            S (n+1)
        \end{pmatrix} = \begin{pmatrix}
            \z_{\tau(n+1)}^{(\xi(n+1))}\\
            \p_{n+1}
        \end{pmatrix} = \h_{n+1}.
    \end{align*}
    We can verify that $\theta_7^*$ consists of constants, $\TF_{\theta_7^*}$ has $4$ heads with FFN width of $O(d)$, and
    \begin{align}
        \|\theta_7^*\|_0 \lesssim d, \ \ \|\theta_7^*\|_{\max} \lesssim 1.\label{eq: theta 7 bound}
    \end{align}
    Choosing $\Theta^* = (\theta_1^*, \theta_2^*, \dots, \theta_7^*)$ yields the desired transformer $\TF_{\Theta^*}$ described in the property (P3) with at most $8$ heads and FFN width of $O(d)$. This gives property (P1). In addition, we can verify that the property (P2) holds from equations~\ref{eq: theta 1 2 bound}, \ref{eq: theta 3 4 bound}, \ref{eq: theta 5 bound}, \ref{eq: h 6 bound}, and \ref{eq: theta 7 bound}. This concludes the proof.
\end{proof}

\begin{proof}[Proof of Theorem~\ref{thm: prompt engineering NN ap}]
    We write $W_\ell = \sum_{k \in [r_\ell]} \tilde \bu_k^{(\ell)} \bu_k^{(\ell) \top}$ with $(\bu_k^{(\ell)})_{k \in [r_\ell]}, (\tilde \bu_k^{(\ell)})_{k \in [r_\ell]} \subset \mathcal{U}$ for $\ell \in [L]$.
    Consider a prompt $H^\pro = [\h_1, \dots, \h_T] \in \mathcal{P}_{\mathcal{U}}(T,L,S)$ with $T=2\sum_{\ell\in[L]} r_\ell$ of the following form:
    For any $s \in [T/2]$, let $\ell \in [L]$ be the number satisfying $\sum_{\ell' < \ell} r_{\ell'} + 1 \leq s \leq \sum_{\ell' \leq \ell} r_{\ell'}$. Then, we define $w_{2s-1} = 2\ell-1$, $w_{2s}=2\ell$,
    \begin{align}
        \h_{2s-1} = \begin{pmatrix}
            \tilde \bu_{s-\sum_{\ell' < \ell} r_{\ell'}}^{(\ell)}\\
            p(w_{2s-1}, 2s-1, S)
        \end{pmatrix}, \text{ and } \ 
        \h_{2s} = \begin{pmatrix}
            \bu_{s-\sum_{\ell' < \ell} r_{\ell'}}^{(\ell)}\\
            p(w_{2s}, 2s, S)
        \end{pmatrix}.\label{eq: h for W}
    \end{align}
    
    From Theorem~\ref{thm: prompt engineering ap}, we already know that when $S \geq d B^{4L} T^{2L} \vee 2L$, given the initial input $H = [H^\pro, h^\dat(\z_1, T+1, S), \dots, h^\dat(\z_N, T+N, S)]$, the transformer $\TF_{\Theta^*}$ sequentially generates $\h_{T+N+1}, \h_{T+N+2}, \dots$ satisfying
    \begin{align*}
        &\h_{T+N\ell+i} = \begin{pmatrix}
            W^\pro_\ell \sigma(W^\pro_{\ell-1} \sigma( \dots \sigma(W^\pro_{1,i} \z_i)\dots))\\
            p(-\ell, T+N\ell+i, S)
        \end{pmatrix}
    \end{align*}
    for all $i \in [N]$ and $\ell \in [L]$, where $W_{1,i}$ and $W_\ell$ for $\ell \in [L]\setminus\{1\}$ are defined in \eqref{eq: W ell random}.
    From \eqref{eq: h for W}, and since $w_T \neq 1$, it follows that
    \begin{align*}
        W^\pro_{1,i} &= \sum_{j' \in [T-1]} (\h_{j'})_{1:d} \Id\{w_{j'}=1\} (\h_{j'+1})_{1:d}^\top \Id\{w_{j'+1}=2\} = \sum_{j \in [r_1]} \tilde \bu_j^{(1)} \bu_j^{(1) \top},
    \end{align*}
    and
    \begin{align*}
        W^\pro_\ell &= \sum_{j' \in [T-1]} (\h_{j'})_{1:d} \Id\{w_{j'}=2\ell-1\} (\h_{j'+1})_{1:d}^\top \Id\{w_{j'+1}=2\ell\} = \sum_{j \in [r_\ell]} \tilde \bu_j^{(\ell)} \bu_j^{(\ell) \top}.
    \end{align*}
    for $\ell \in [L]\setminus\{1\}$. Hence $W_\ell^\pro = W_\ell$. Thus we only need to verify that the transformer $\TF_{\Theta^*}$ works as intended for $S \geq d \bar{r} B^{4L} \vee 2L$ by carefully bounding the intermediate tokens.

    Define data tokens as in \eqref{eq: data tokens}:
    For $j \geq T+1$, let $w_j = -\xi(j)$, $\p_j = p(w_j, j, S)$, and
    \begin{align*}
        \h_j = \begin{pmatrix}
            \z_{\tau(j)}^{(\xi(j))}\\
            \p_j
        \end{pmatrix},
    \end{align*}
    where $\z_i^{(0)} := \z_i$, $\z_i^{(\ell)} := W_\ell \y_i^{(\ell-1)}$ for $\ell \geq 1$ with $\y_i^{(0)} := \z_i$ and $\y_i^{(\ell)} := \sigma(\z_i^{(\ell)})$ for $\ell \geq 1$.
    Let $n := T+N+v$ and $D := 4d+8$.     
    Define $\ell^* := \xi(n-N+1)$ and $i^* := \tau(n-N+1)$.
    Note that since $\z_i \in [0, 1]^d$, and $\|W_\ell\| \leq B^2$, we have $\|\z_i^{(\ell)}\| \leq d^{1/2} B^{2\ell} \leq d^{1/2} B^{2L}$ for all $i \in [N]$ and $\ell \leq L$.

    From the proof of Theorem~\ref{thm: prompt engineering ap}, since $\|(\h_j)_{1:d}\|_{\max} \leq B \vee \|\z_{i^*}^{(\ell^*)}\| \leq B \vee d^{1/2} B^{2\ell^*} \leq S$ for all $j \in [n]$, $H_v^{[2]} = [\h_1^{[2]}, \dots, \h_n^{[2]}] := \TF_{(\theta_1^*,\theta_2^*)}(H_v)$ satisfies
    \begin{align*}
        \h_j^{[2]} = \h_j \ \text{for $j \leq T$,\ \ and }
        \h_j^{[2]} = \begin{pmatrix}
            \y_{\tau(j)}^{(\xi(j))}\\
            \p_j
        \end{pmatrix} \ \text{for $j \geq T+1$},
    \end{align*}
    and hence $\max_{j \in [n]} \|(\h_j^{[2]})_{1:d}\|_{\max} \leq B \vee d^{1/2} B^{2\ell^*} \leq S$.
    Furthermore, since $\|(\h_j^{[3]})_{1:d}\|_{\max} = \|(\h_j^{[2]})_{1:d}\|_{\max} \leq B \vee d^{1/2} B^{2\ell^*} \leq S$ holds for $[\h_1^{[3]}, \h_2^{[3]}, \dots, \h_n^{[3]}] = \TF_{(\theta_3^*)}(H_v^{[2]})$, by the proof of Theorem~\ref{thm: prompt engineering ap}, $H_v^{[4]} = [\h_1^{[4]}, \dots, \h_n^{[4]}] := \TF_{(\theta_4^*)}(H_v^{[3]})$ satisfies
    \begin{align}
        \h_j^{[4]} &= \begin{pmatrix}
            (\h_j^{[2]})_{1:d} \delta_j\\
            \p_j^{[4]}
        \end{pmatrix}, \ \ \p_j^{[4]} := \begin{pmatrix}
            \zero_{D-d-5}\\
            S N\\
            1\\
            S\\
            S w_j\\
            S j
        \end{pmatrix},\label{eq: h [4] j weak}
    \end{align}
    and $\max_{j \in [n]} \|(\h_j^{[4]})_{1:d}\|_{\max} \leq B \vee d^{1/2} B^{2\ell^*}$.
    
    Then, the transformer $\TF_{\theta_5^*}$ takes $H_v^{[4]}$ in \eqref{eq: h [4] j weak} as input and generates $H_v^{[5]} = [\h_1^{[5]}, \dots, \h_n^{[5]}] := \TF_{\theta_5^*}(H_v^{[4]})$ given by
    \begin{small}
    \begin{align*}   
        \h_j^{[5]} &= \h_j^{[4]} + \sum_{j' \in [n]} \phi_{1,S}(\alpha_{j,j'}; w_j, -2w_{j'}+2) \e_{D-5},
    \end{align*}
    \end{small}\noindent
    where $\alpha_{j,j'} = (\h_j^{[4]})_{1:d}^\top (\h_{j'}^{[4]})_{1:d}$. 
    Note that $|\alpha_{j,j'}| \leq (B \vee d^{1/2} B^{2\ell^*})^2 \leq S$ for all $j, j' \in [n]$.
    We now investigate the value of $\phi_{1,S}(\alpha_{j,j'}; w_j, -2w_{j'}+2)$.
    For $j \leq T$, since $w_j$ is positive and odd for odd $j \leq T$, and $-2w_{j'}+2$ is even for $j' \in [n]$, it follows that $\phi_{1,S}(\alpha_{j,j'}; w_j, -2w_{j'}+2) = 0$ for odd $j \leq T$, and thus $\Attn_{\mu_5^*}(H_v^{[4]})_{2s-1} = \h_{2s-1}^{[4]}$ for $s \in [T/2]$. 
    In addition, since at most $\bar{r}$ of $(\phi_{1,S}(\alpha_{j,j'}; w_j, -2w_{j'}+2))_{j' \in [n]}$ is non-zero for any $j \in [n]\setminus[T]$, we obtain that
    \begin{align*}
        \max_{j \in [n]\setminus[T]} |(\h_j^{[5]})_{D-5}| \leq \bar{r} \max_{j \in [n]\setminus[T],j' \in [T]} |\alpha_{j,j'}| \leq \bar{r} B ( B \vee d^{1/2} B^{2\ell^*} ) \leq S.
    \end{align*}
    Furthermore, it follows that for $\ell \in [L]$ and $s \in [T/2]$ satisfying $\sum_{\ell'<\ell} r_{\ell'} + 1 \leq s \leq \sum_{\ell'\leq\ell} r_{\ell'}$,
    \begin{align*}
        \sum_{j' \in [n]} \phi_{1,S}(\alpha_{2s,j'}; w_{2s}, -2 w_{j'} +2) &= \sum_{j' \in [n]} \phi_{1,S}(\alpha_{2s,j'}; 2\ell, -2 w_{j'} +2)\\
        &= \sum_{j' \in [n]} \alpha_{2s,j'} \Id\{\ell = -w_{j'}+1\}\\
        &= \sum_{j' \in [n]\setminus[T]: w_{j'} = 1 - \ell} \bu_{s-\sum_{\ell'<\ell} r_{\ell'}}^{(\ell) \top} \y_{\tau(j')}^{(\xi(j'))} \Id\{\xi(j')=\ell^*,\tau(j')=i^*\}\\
        &= \bu_{s-\sum_{\ell'<\ell} r_{\ell'}}^{(\ell) \top} \y_{i^*}^{(\ell^*)},
    \end{align*}
    where we used $1 - \ell \leq 0$ in the third equality, and the uniqueness of $j' \in [n]\setminus[T]$ satisfying $\xi(j') = \ell^*$ and $\tau(j') = i^*$ in the last equality.
    In summary,
    \begin{align*}
        \h_j^{[5]} = \begin{pmatrix}
            (\h_j^{[4]})_{1:d}\\
            \p_j^{[5]}
        \end{pmatrix}.
    \end{align*}
    Here for any $s \in [T/2]$ and $\ell \in [L]$ satisfying $\sum_{\ell'<\ell} r_{\ell'} + 1 \leq s \leq \sum_{\ell'\leq\ell} r_{\ell'}$,
    \begin{align*}
        \p_{2s-1}^{[5]} &= \begin{pmatrix}
            \zero_{D-d-6}\\
            0\\
            S N\\
            1\\
            S\\
            S w_{2s-1}\\
            S (2s-1)
        \end{pmatrix}, \ \ \p_{2s}^{[5]} = \begin{pmatrix}
            \zero_{D-d-6}\\
            \bu_{s-\sum_{\ell'<\ell} r_{\ell'}}^{(\ell) \top} \y_{i^*}^{(\ell^*)}\\
            S N\\
            1\\
            S\\
            S w_{2s}\\
            2S s
        \end{pmatrix}, \text{ and } \p_j^{[5]} = \begin{pmatrix}
            \zero_{D-d-6}\\
            *\\
            S N\\
            1\\
            S\\
            S w_j\\
            S j
        \end{pmatrix} \text{ for $j \geq T+1$}.
    \end{align*}
    We also have $\max_{j \in [n]} |(\h_j^{[5]})_{D-5}| \leq \bar{r} B (B \vee d^{1/2} B^{2\ell^*}) \leq S$.

    Since $\max_{j \in [n]} |S w_j-S| \leq 2SL \leq S^2$, we can verify that $H_v^{[6]} = [\h_1^{[6]}, \dots, \h_n^{[6]}] := \TF_{\theta_6^*}(H_v^{[5]})$ satisfies
    \begin{align*}
        \h_j^{[6]} &= \begin{pmatrix}
            (\h_j^{[5]})_{1:d}\\
            \p_j^{[6]}
        \end{pmatrix},
    \end{align*}
    where $\max_{j \in [n]} |(\h_j^{[6]})_{D-5}| \leq \bar{r} B (B \vee d^{1/2} B^{2\ell^*})$ and $\p_j^{[6]}$ given by
    \begin{align*}
        \p_{2s-1}^{[6]} &= \begin{pmatrix}
            \zero_{D-7}\\
            *\\
            \bu_{s-\sum_{\ell'<\ell} r_{\ell'}}^{(\ell) \top} \y_{i^*}^{(\ell^*)}\\
            S N\\
            1\\
            S\\
            S w_{2s-1}\\
            S (2s-1)
        \end{pmatrix}, \ \ 
        \p_{2s}^{[6]} = \begin{pmatrix}
            \zero_{D-7}\\
            *\\
            *\\
            S N\\
            1\\
            S\\
            S w_{2s}\\
            2S s
        \end{pmatrix} \text{ for $\sum_{\ell'<\ell} r_{\ell'} + 1 \leq s \leq \sum_{\ell'\leq\ell} r_{\ell'}$},\\
        \p_j^{[6]} &= \begin{pmatrix}
            \zero_{D-d-7}\\
            *\\
            *\\
            S N\\
            1\\
            S\\
            S w_j\\
            S j
        \end{pmatrix} \ \text{for $T+1 \leq j \leq T+N-1$, and }
        \p_j^{[6]} = \begin{pmatrix}
            \zero_{D-d-7}\\
            S w_{j+1}\\
            *\\
            S N\\
            1\\
            S\\
            S w_j\\
            S j
        \end{pmatrix} \ \text{ for $T + N \leq j$}.
    \end{align*}
    From the proof of Theorem~\ref{thm: prompt engineering ap}, we can verify that 
    \begin{align*}
        \h_n^{[7]} = \begin{pmatrix}
            W_{\ell^*+1} \y_{i^*}^{(\ell^*)}\\
            \p_{n+1}
        \end{pmatrix}.
    \end{align*}
    This completes the proof.
\end{proof}

\begin{proof}[Proof of Theorem~\ref{thm: prompt engineering EUAF ap}]
    Define $\z_i^{(0)} := \z_i$, $\z_i^{(1)} := W^\pro_{1,i} \y_i^{(0)}$, $\z_i^{(\ell)} := W^\pro_\ell \y_i^{(\ell-1)}$ for $\ell \geq 2$ with $\y_i^{(0)} = \z_i$ and $\y_i^{(\ell)} = \sigma_\EUAF(\z_i^{(\ell)})$ for $\ell \geq 1$.
    To show the existence of the desired transformer, we only need to modify Part (a) in the proof of Theorem~\ref{thm: prompt engineering ap}, since we only change the activation function in the first layer of the transformer.
    Specifically, we show that there exists a $3$-layer transformer $\TF_{(\theta_1^\#,\theta_2^\#, \theta_3^\#)}$ that takes $H_v$ ($v \in \{0\} \cup [NL-1]$) in \eqref{eq: H} as input and generates $H_v^{[3]} := \TF_{(\theta_1^\#,\theta_2^\#, \theta_3^\#)}(H_v) =: [\h_1^{[3]}, \dots, \h_n^{[3]}]$, where
    \begin{align*}
        \h_j^{[3]} = \h_j \ \text{ for $j \leq T$,\ \ and }
        \h_j^{[3]} = \begin{pmatrix}
            \y_{\tau(j)}^{(\xi(j))}\\
            \p_j
        \end{pmatrix} \ \text{ for $j \geq T+1$}.
    \end{align*}   
    Here $\xi$ and $\tau$ are defined in the proof of Theorem~\ref{thm: prompt engineering ap}, and $\p_j := p(w_j, j, S)$ with $w_j := -\xi(j)$ for $j \geq T+1$.
    Let $n := T+N+v$ and $D := 4d+8$ and define $\ell^* = \xi(n-N+1)$ and $i^*=\tau(n-N+1)$.
    We first set the parameter $\nu_1^\# = (W_{1,1}^\#, W_{1,2}^\#)$ such that
    \begin{align*}
        W_{1,2}^\# \sigma_\EUAF(W_{1,1}^\# \h) &= W_{1,2}^\# \sigma_\EUAF((\h)_{1:d}) = \begin{pmatrix}
            \zero_d\\
            \sigma_\EUAF((\h)_{1:d})\\
            \zero_{D-2d}
        \end{pmatrix}.
    \end{align*}
    Let $\theta_1^\# = (\mu_\id^\#, \nu_1^\#)$. Then
    \begin{align*}
        \TF_{\theta_1^\#}(H_v)_j = \begin{pmatrix}
            (\h)_{1:d}\\
            \sigma_\EUAF((\h)_{1:d})\\
            \zero_{D-2d-4}\\
            1\\
            S\\
            S w_j\\
            S j
        \end{pmatrix}.
    \end{align*}
    Next we set the parameters $\nu_2^\# = (W_{2,1}^\#, W_{2,2}^\#)$ as 
    \begin{align*}
        W_{2,2}^\# \sigma(W_{2,1}^\# \h) &= W_{2,2}^\# \begin{pmatrix}
            \sigma(-(\h)_{1:d} + k_{\h} \1_d)\\
            \sigma((\h)_{1:d} + k_{\h} \1_d)\\
            \sigma(-(\h)_{(d+1):(2d)} + k_{\h'} \1_d)\\
            \sigma((\h)_{(d+1):(2d)} + k_{\h'} \1_d)\\
            -\sigma((\h)_{1:d})\\
            \sigma((\h)_{1:d})\\
            -\sigma((\h)_{(d+1):(2d)})\\
            \sigma((\h)_{(d+1):(2d)})\\
            \sigma(k_{\h} \1_d)\\
            \sigma(k_{\h'} \1_d)
        \end{pmatrix}\\
        &= \begin{pmatrix}
            -(\h)_{1:d} + \sigma(-(\h)_{1:d} + k_{\h} \1_d) - \sigma(k_{\h} \1_d)\\
            -(\h)_{(d+1):(2d)} + \sigma((\h)_{1:d} + k_{\h} \1_d) - \sigma(k_{\h} \1_d)\\
            \sigma(-(\h)_{(d+1):(2d)} + k_{\h'} \1_d) - \sigma(k_{\h'} \1_d)\\
            \sigma((\h)_{(d+1):(2d)} + k_{\h'} \1_d) - \sigma(k_{\h'} \1_d)\\
            \zero_{D-4d}
        \end{pmatrix}.
    \end{align*}
    where $k_{\h} := (\h)_{D-1}$ and $k_{\h}' := - (\h)_{D-2} - (\h)_{D-1}$.
    Let $\theta_2^\# = (\mu_\id, \nu_2^\#)$. Then
    \begin{align*}
        \TF_{(\theta_1^\#,\theta_2^\#)}(H_v)_j = \begin{pmatrix}
            \bzeta_j^{(1)}\\
            \bzeta_j^{(2)}\\
            \bzeta_j'^{(1)}\\
            \bzeta_j'^{(2)}\\
            \zero_{D-4d-4}\\
            1\\
            S\\
            S w_j\\
            S j
        \end{pmatrix},
    \end{align*}
    where
    \begin{small}
    \begin{align*}
        \bzeta_j^{(1)} &= \sigma(-(\h_j)_{1:d} + k_{\h_j} \1_d) - \sigma(k_{\h_j} \1_d), \ \ 
        \bzeta_j^{(2)} = \sigma((\h_j)_{1:d} + k_{\h_j} \1_d) - \sigma(k_{\h_j} \1_d),\\
        \bzeta_j'^{(1)} &= \sigma(-\sigma_\EUAF((\h_j)_{1:d}) + k_{\h_j}' \1_d) - \sigma(k_{\h_j}' \1_d), \ \ 
        \bzeta_j'^{(2)} = \sigma(\sigma_\EUAF((\h_j)_{1:d}) + k_{\h_j}' \1_d) - \sigma(k_{\h_j}' \1_d).
    \end{align*}
    \end{small}\noindent
    Next we set the parameter $\nu_3^\# = (W_{3,1}^\#, W_{3,2}^\#)$ such that
    \begin{align*}
        W_{3,2}^\# \sigma(W_{3,1}^\# \h) &= W_{3,2}^\# \begin{pmatrix}
            \sigma(-(\h)_{1:d})\\
            \sigma((\h)_{1:d})\\
            \sigma(-(\h)_{(d+1):(4d)})\\
            \sigma((\h)_{(d+1):(4d)})\\
            \sigma(-(\h)_{(2d+1):(3d)})\\
            \sigma((\h)_{(2d+1):(3d)})\\
            \sigma(-(\h)_{(3d+1):(4d)})\\
            \sigma((\h)_{(3d+1):(4d)})
        \end{pmatrix}\\
        &= \begin{pmatrix}
            -(\h)_{1:d} - \sigma((\h)_{1:d}) + \sigma((\h)_{(d+1):2d}) - \sigma((\h)_{(2d+1):3d}) + \sigma((\h)_{(3d+1):4d})\\
            -(\h)_{(d+1):(2d)}\\
            -(\h)_{(2d+1):(3d)}\\
            -(\h)_{(3d+1):(4d)}\\
            \zero_{D-4d}
        \end{pmatrix}.
    \end{align*}
    Let $\theta_3^\# = (\mu_\id, \nu_3^\#)$. Then
    \begin{align*}
        \TF_{(\theta_1^\#,\theta_2^\#,\theta_3^\#)}(H_v)_j = \begin{pmatrix}
            \sigma(\bzeta_j^{(2)}) - \sigma(\bzeta_j^{(1)}) + \sigma(\bzeta_j'^{(2)}) - \sigma(\bzeta_j'^{(1)})\\
            \zero_{D-d-4}\\
            1\\
            S\\
            S w_j\\
            S j
        \end{pmatrix}.
    \end{align*}
    Since $\|(\h_j)_{1:d}\| \vee \|\sigma_\EUAF((\h_j)_{1:d})\| \leq B \vee B^{\ell^*+1} \leq S$, and $k_{\h_j}/S, k_{\h_j}'/S \in \Z$ for all $j \in [n]$,
    \begin{align*}
        \sigma(\bzeta_j^{(2)}) - \sigma(\bzeta_j^{(1)}) &= \qty{\sigma((\h_j)_{1:d}) - \sigma(-(\h_j)_{1:d})} \Id\{k_{\h_j} \geq 0\}\\
        &= (\h_j)_{1:d} \Id\{w_j \geq 0\},\\
        \sigma(\bzeta_j'^{(2)}) - \sigma(\bzeta_j'^{(1)}) &= \qty{\sigma(\sigma_\EUAF((\h_j)_{1:d})) - \sigma(-\sigma_\EUAF((\h_j)_{1:d}))} \Id\{k_{\h_j}' \geq 0\}\\
        &= \sigma_\EUAF((\h_j)_{1:d}) \Id\{w_j \leq -1\}.
    \end{align*}
    This yields
    \begin{align*}
        \TF_{(\theta_1^\#,\theta_2^\#,\theta_3^\#)}(H_v)_j &= \begin{pmatrix}
            (\h_j)_{1:d} \Id\{w_j \geq 0\} + \sigma_\EUAF((\h_j)_{1:d}) \Id\{w_j \leq -1\}\\
            \zero_{D-d-4}\\
            1\\
            S\\
            S w_j\\
            S j
        \end{pmatrix}
    \end{align*}
    and thus $\TF_{(\theta_1^\#, \theta_2^\#,\theta_3^\#)}(H_v)_j = \h_j$ for $j \leq T$ and $\TF_{(\theta_1^\#,\theta_2^\#,\theta_3^\#)}(H_v)_j = (\y_{\tau(j)}^{(\xi(j)) \top}, \p_j^\top)^\top$ for $j \geq T+1$.
    We can verify that $\theta_1^\#,\theta_2^\#$, and $\theta_3^\#$ consist of constants, FFN width of $\TF_{(\theta_1^\#, \theta_2^\#,\theta_3^\#)}$ is of $O(d)$, 
    \begin{align*}
        \|(\theta_1^\#,\theta_2^\#,\theta_3^\#)\|_0 \lesssim d, \ \ \|(\theta_1^\#,\theta_2^\#,\theta_3^\#)\|_{\max} \lesssim 1, \ \ \text{and} \ \ \max_{j \in [n]} \|(\h_j^{[3]})_{1:d}\| \leq \|(\h_j)_{1:d}\| \leq d^{1/2} B^L.
    \end{align*}
    This completes the modification to Part (a) in  the proof of Theorem~\ref{thm: prompt engineering ap}.
    In contrast to the original Part (a), we have three transformer layers in this proof. Hence the desired transformer has an extra layer.
    The rest of the proof is a direct application of Part (b) and (c) in the proof of Theorem~\ref{thm: prompt engineering ap}. Choosing $\Theta^\# = (\theta_1^\#, \theta_2^\#, \theta_3^\#, \theta_3^*, \theta_4^*, \dots, \theta_7^*)$ concludes the proof.
\end{proof}

\section{Technical Details and proofs for Section~\ref{sec: approximation}}\label{sec: approximation ap}

Here we present our approximation results in Section~\ref{sec: approximation} followed by the proofs.

\subsection{Theoretical Results}

For any $d \geq r$, let $\mathcal{U}_{d,r} := \spn\{\e_1, \e_2, \dots, \e_r\}$ denote the subspace spanned by the first $r$ standard basis vectors in $\R^d$.

\begin{cor}\label{cor: approximation error in terms of L ap}
    Fix $p, \beta \in \N$ and $d \geq 17\beta^{p+1}3^{p}p + 1$.
    For any $r \in \N$ with $17\beta^{p+1}3^{p}p \leq r \leq d - 1$, $T \geq 2(r+1)L$, and $L\geq 108\beta^2+3p$, the following holds:
    \begin{align*}
        &\sup_{f \in \mathcal{C}_\u^\beta([0, 1]^p)} \inf_{\substack{S>0\\H^\pro \in \mathcal{Q}_{\mathcal{U}_{d,r+1}}(T,L,S)}} \|\hat f_{\Theta^*,H^\pro,L,p} - f\|_{\mathcal{L}^\infty([0, 1]^p)}\\
        &\quad\leq C_1 \left(\frac{r}{C_2\log(8r+8)}\right)^{-2\beta/p} \left(\frac{L-3p}{C_3\log(4L)}\right)^{-2\beta/p},
    \end{align*}
    where $C_1$, $C_2$, and $C_3$ are the constants appearing in Lemma~\ref{lem: approximation} that depend on $p$ and $\beta$.
\end{cor}
Corollary~\ref{cor: approximation error in terms of L ap} is a direct consequence of Theorem~\ref{thm: prompt engineering NN ap} and the result for neural network approximation by \citet{lu2021deep}. 

We also state the following corollary for convenience.
\begin{cor}[Restatement of Corollary~\ref{cor: approximation error}]\label{cor: approximation error ap}
    Fix $p \in \N$, $\beta \in \N$ and $\varepsilon \in (0, 1)$.
    Set $d \geq 1+17\beta^{p+1} 3^p p$. 
    Then, there exist a constant $C(\beta, p, d) > 0$ and a universal constant $c_L > 0$ such that
    \begin{align*}
        \sup_{f \in \mathcal{C}_\u^\beta([0, 1]^p)} \inf_{\substack{S>0\\H^\pro \in \mathcal{Q}_{\R^d}(T,L,S)}} \|\hat f_{\Theta^*,H^\pro,L,p} - f\|_{\mathcal{L}^\infty([0, 1]^p)} \leq C \varepsilon
    \end{align*}
    holds for any $L \geq c_L(p+\beta^2+(p/\beta) \varepsilon^{-p/(2\beta)}\log(1/\varepsilon))$ and $T \geq 2d L$.
\end{cor}

We then present the following corollary without proof.
\begin{cor}[Restatement of Corollary~\ref{cor: approximation error EUAF}]\label{cor: approximation error EUAF ap}
    Fix $p \in \N$ and $\varepsilon > 0$. Set $d = 36p(2p+1)$ and $L=12$.
    Then, for transformer $\TF_{\Theta^\#}$ introduced in Theorem~\ref{thm: prompt engineering EUAF ap}, the following holds for $T = 24d$:
    \begin{align*}
        \sup_{f \in \mathcal{C}([0, 1]^p)} \inf_{\substack{S > 0\\H^\pro \in \mathcal{Q}_{\R^d}(T,L,S)}} \|\hat f_{\Theta^\#,H^\pro,L,p} - f\|_{\mathcal{L}^\infty([0, 1]^p)} \leq \varepsilon
    \end{align*}
\end{cor}
The proof of Corollary~\ref{cor: approximation error EUAF ap} is a direct analogy of the proof of Corollary~\ref{cor: approximation error ap} using Theorem~\ref{thm: prompt engineering EUAF NN ap}, and Lemma~\ref{lem: approximation euaf}; Any functions implemented by neural networks in $\mathcal{N}_\EUAF(d,p,L)$ can be realized by the transformer $\TF_{\Theta^\#}$ given some prompt $H^\pro \in \mathcal{Q}_{\R^d}(T,L,S)$ with $L=12$, $T=24d$, and a sufficiently large $S > 0$.

\subsection{Proofs}

Let $\mathcal{N}(r, p, L)$ be the class of ReLU neural networks defined in \eqref{eq: NN class}.
We use the following result from \citet{lu2021deep}. 
\begin{lem}[Corollary 1.2 from \citet{lu2021deep}]\label{lem: approximation}
    Fix $p,\beta \in \N$.
    For any smooth function $f \in \mathcal{C}^\beta([0,1]^p)$, $r, L \in \N$ with $r\geq 17\beta^{p+1}3^{p+2}p$, and $L\geq 108\beta^2+2p+1$, there exists some $\bar \phi \in \mathcal{N}(r, p, L)$ such that
    \begin{align*}
        \| \bar \phi - f \|_{\mathcal{L}^\infty([0,1]^p)} \leq C_1 \| f \|_{\mathcal{C}^\beta([0,1]^p)} \left(\frac{r}{C_2\log(8r+8)}\right)^{-2\beta/p} \left(\frac{L-2p-1}{C_3\log(4L)}\right)^{-2\beta/p},
    \end{align*}
    where $C_1 = 85(\beta+1)^{p} 8^\beta$, $C_2 = 68\beta^{p+1} 3^{p}p$, and $C_3 =72\beta^2$.
\end{lem}

We also define the class of functions implemented by neural networks using the EUAF activation function with depth $L$ and width $r$ from $\R^p$ to $\R$ as
\begin{align*}
    \mathcal{N}_\EUAF(r, p, L) := \{ &\mathcal{A}(\cdot; \bar W_L, 0) \circ \sigma_\EUAF \circ \mathcal{A}(\cdot; \bar W_{L-1}, \zero_r) \circ \sigma_\EUAF \circ \dots \circ \sigma_\EUAF \circ \mathcal{A}(\cdot; \bar W_1, \zero_r)\\
    &\quad : \bar W_1 \in \R^{r \times p}, \{\bar W_2, \dots, \bar W_{L-1}\} \subset \R^{r \times r}, \bar W_L \in \R^{1 \times r} \},
\end{align*}
where $\mathcal{A}(\x; W_\ell, \b_\ell) := W_\ell \x + \b_\ell$.
We also borrow the following result from \citet{zhang2022deep} for the approximation error of neural networks using the EUAF activation function, with a slight modification.
\begin{lem}[Theorem 1 from \citet{zhang2022deep}]\label{lem: approximation euaf}
    Fix $p \in \N$ and $\varepsilon > 0$. Set $d = 36p(2p+1)$ and $L=12$.
    For any continuous function $f \in \mathcal{C}([0,1]^p)$, there exists some $\bar \phi \in \mathcal{N}_\EUAF(d, p, L)$ such that
    \begin{align*}
        \| \bar \phi - f \|_{\mathcal{L}^\infty([0,1]^p)} \leq \varepsilon.
    \end{align*}
\end{lem}

\begin{proof}[Proof of Corollary~\ref{cor: approximation error in terms of L ap}]
    As shown in Theorem~\ref{thm: prompt engineering NN ap}, the transformer $\TF_{\Theta^*}$ can emulate a certain class of virtual ReLU neural network. We first show that any ReLU neural network with width $r$ and depth $L$ from $\R^p$ to $\R$ can be realized by a prompt in $\mathcal{P}_{\mathcal{U}_{d,{r+1}}}(T,L,S)$ when $p \leq r \leq d-1$, $T \geq 2(r+1)L$ and $S$ is sufficiently large.
    Take any $\bar \phi \in \mathcal{N}(r, p, L)$. Write $\bar \phi = \mathcal{A}(\cdot; \bar W_L, \bar b_L) \circ \sigma \circ \mathcal{A}(\cdot; \bar W_{L-1}, \bar \b_{L-1}) \circ \sigma \circ  \dots \circ \sigma \circ \mathcal{A}(\cdot; \bar W_1, \bar \b_1)$ with some $\bar W_1 \in \R^{r \times p}$, $\{\bar W_2, \bar W_3, \dots, \bar W_{L-1}\} \subset \R^{r \times r}$, $\bar W_L \in \R^{1 \times r}$, $\{\bar \b_1, \bar \b_2, \dots, \bar \b_{L-1}\} \subset \R^r$, and $\bar b_L \in \R$.

    Now define
    \begin{align*}
        W_1 &= \begin{pmatrix}
            \bar W_1 & \bar \b_1 & O_{r \times(d-p-1)}\\
            \zero_p^\top & 1 & \zero_{d-p-1}^\top\\
            O_{(d-r-1)\times p} & \zero_{d-r-1} & O_{(d-r-1)\times(d-p-1)}
        \end{pmatrix} \in \R^{d \times d},\\
        W_\ell &= \begin{pmatrix}
            \bar W_\ell & \bar \b_\ell & O_{r \times(d-r-1)}\\
            \zero_r^\top & 1 & \zero_{d-r-1}^\top\\
            O_{(d-r-1)\times r} & \zero_{d-r-1} & O_{(d-r-1)\times(d-r-1)}
        \end{pmatrix} \in \R^{d \times d} \ \text{ for $j \in [L-1]\setminus\{1\}$},\\
        W_L &= \begin{pmatrix}
            \bar W_L & \bar b_L & \zero_{d-r-1}^\top\\
            O_{(d-1)\times r} & \zero_{d-1} & O_{(d-1)\times(d-r-1)}
        \end{pmatrix} \in \R^{d \times d}.
    \end{align*}
    Then, notice that $(W_L \sigma( W_{L-1} \sigma( \dots \sigma( W_1 [\x^\top, 1, \zero_{d-p-1}^\top]^\top) \dots )))_1 = \bar \phi(\x)$. Also note that all singular vectors of $(W_\ell)_{\ell \in [L]}$ lies in $\mathcal{U}_{d,r+1}$.
    By Definition~\ref{def: function approx} and Theorem~\ref{thm: prompt engineering NN ap}, there exists some prompt $H^\pro \in \mathcal{P}_{\mathcal{U}_{d,r+1}}(T,L,S)$ with $T = 2 (r+1) L$ and sufficiently large $S > 0$ such that $\hat f_{\Theta^*,H^\pro,L,p} = \bar \phi$ holds. 
    Since $\bar \phi \in \mathcal{N}(r,p,L)$ is arbitrary, any $\bar \phi \in \mathcal{N}(r,p,L)$ has a corresponding prompt in $\mathcal{Q}_{\mathcal{U}_{d,r+1}}(T,L,S)$ as long as $T \geq 2(r+1)L$ and $S$ is sufficiently large. This yields
    \begin{align}
        &\sup_{f \in \mathcal{C}_\u^\beta([0, 1]^p)} \inf_{\substack{S > 0\\H^\pro \in \mathcal{Q}_{\mathcal{U}_{d,r+1}}(T,L,S)}} \|\hat f_{\Theta^*,H^\pro,L,p} - f\|_{\mathcal{L}^\infty([0, 1]^p)}\nonumber\\
        &\quad\leq \sup_{f \in \mathcal{C}_\u^\beta([0, 1]^p)} \inf_{\bar \phi \in \mathcal{N}(r,p,L)} \|\bar \phi - f\|_{\mathcal{L}^\infty([0, 1]^p)}\label{eq: upper bound rhs}
    \end{align}
    when $T \geq 2(r+1)L$.
    Applying Lemma~\ref{lem: approximation} to the right hand side of \eqref{eq: upper bound rhs} under the assumption that $r \geq 17\beta^{p+1}3^{p}p$, $L\geq 108\beta^2+3p$, we obtain
    \begin{align*}
        &\sup_{f \in \mathcal{C}_\u^\beta([0, 1]^p)} \inf_{\substack{S>0\\H^\pro \in \mathcal{Q}_{\mathcal{U}_{d,r+1}}(T,L,S)}} \|\hat f_{\Theta^*,H^\pro,L,p} - f\|_{\mathcal{L}^\infty([0, 1]^p)}\\
        &\quad\leq C_1 \left(\frac{r}{C_2\log(8r+8)}\right)^{-2\beta/p} \left(\frac{L-3p}{C_3\log(4L)}\right)^{-2\beta/p},
    \end{align*}
    where $C_1$, $C_2$, and $C_3$ are constants appearing in Lemma~\ref{lem: approximation}. This concludes the proof.
\end{proof}

\begin{proof}[Proof of Corollary~\ref{cor: approximation error ap}]
    Choosing $r = d-1$ in Corollary~\ref{cor: approximation error in terms of L ap} gives
    \begin{align*}
        &\sup_{f \in \mathcal{C}_\u^\beta([0, 1]^p)} \inf_{\substack{S>0\\H^\pro \in \mathcal{Q}_{\R^d}(T,L,S)}} \|\hat f_{\Theta^*,H^\pro,L,p} - f\|_{\mathcal{L}^\infty([0, 1]^p)}\\
        &\quad\leq C_1 \left(\frac{d-1}{C_2\log(8d)}\right)^{-2\beta/p} \left(\frac{L-3p}{C_3\log(4L)}\right)^{-2\beta/p}.
    \end{align*}
    Therefore, by the choice of
    \begin{align*}
        C = C_1 \left(\frac{d-1}{C_2 C_3\log(8d)}\right)^{-2\beta/p}, \quad \varepsilon \geq \left(\frac{L-3p}{\log(4L)}\right)^{-2\beta/p},
    \end{align*}
    we conclude that 
    $$
        \sup_{f \in \mathcal{C}_\u^\beta([0, 1]^p)} \inf_{\substack{S > 0\\H^\pro \in \mathcal{Q}_{\R^d}(T,L,S)}} \|\hat f_{\Theta^*,H^\pro,L,p} - f\|_{\mathcal{L}^\infty([0, 1]^p)} \leq C\varepsilon.
    $$
    Note that by the choice of $\varepsilon$, we must have
    \begin{align}
        L \geq 3p + (2+\log L)\varepsilon^{-p/(2\beta)}.\label{eq: L ub 1}
    \end{align}
    If $L\geq 108\beta^2 + 6p$, then $L-3p\geq \frac{1}{2}(L+1)$ and $2\log(4L)\leq \sqrt{L+1}$, hence 
    $$
        \varepsilon^{-p/2\beta} = \frac{L-3p}{\log(4L)} \geq \sqrt{L+1}.
    $$
    By taking the logarithm on both sides, we obtain
    \begin{align}
        \log(L+1) \leq \frac{p}{\beta}\log(\frac{1}{\varepsilon}).\label{eq: L ub 2}
    \end{align}
    From \eqref{eq: L ub 1} and \eqref{eq: L ub 2}, it suffices to choose $L \geq c_L (p+\beta^2+(p/\beta)\varepsilon^{-p/(2\beta)}\log(1/\varepsilon))$ for some universal constant $c_L > 0$, in order to achieve an $\varepsilon$-approximation. The constant term $C > 0$ can be written as
    \begin{align*}
        &C = C_1 \left(\frac{d-1}{C_2C_3\log(8r+8)}\right)^{-2\beta/p} = 85(\beta+1)^p8^\beta \left(\frac{d-1}{4896\beta^{p+3}3^p p\log(8r+8)}\right)^{-2\beta/p}.
    \end{align*}
\end{proof}

\section{Proofs for Section~\ref{sec: application}}\label{sec: application ap}

Here we directly present proofs for the results in Section~\ref{sec: application}.
Throughout the proof, we repeatedly leverage the observation that the problem of bounding the approximation error via prompts can be equivalently reformulated as the problem of bounding the approximation error of neural networks.

\begin{proof}[Proof of Corollary~\ref{cor: approximation error by prompt length}]
    The proof directly follows by choosing $r = d-1$ and $L=\floor{T/(2d)}$ in Corollary~\ref{cor: approximation error in terms of L ap}.
\end{proof}

\begin{proof}[Proof of Corollary~\ref{cor: irrelevant}]
    It suffices to prove the following lower bound for $f \equiv 1 \in \mathcal{C}_\u^\beta([0, 1]^p)$ and any $H^\pro \in \mathcal{P}_{\mathcal{B}_d(B)}(T,L,S)$.
    \begin{align*}
        \E[\|\hat f_{\Theta^*,H^\pro \oplus [\bv_1,\bv_2,\dots,\bv_K],L+1,p} - f\|_{\mathcal{L}^\infty([0, 1]^p)}] &\gtrsim 1,
    \end{align*}
    Fix any $H^\pro \in \mathcal{P}_{\mathcal{B}_d(B)}(T,L,S)$.
    We temporarily fix $K=2$. From Corollary~\ref{cor: prompt engineering random ap} and Definition~\ref{def: function approx}, there exist $(W^\pro_\ell)_{\ell \in [L+1]} \subset \R^{d \times d}$ with $W^\pro_{L+1} = \bv_1 \bv_2^\top$ such that
    \begin{align*}
        \hat f_{\Theta^*,H^\pro\oplus[\bv_1,\bv_2],L,p}(\x) &= \e_1^\top W^\pro_{L+1} \sigma(W^\pro_L \sigma(W^\pro_{L-1} \sigma( \dots \sigma(W^\pro_1 [\x^\top, 1, \zero_{d-p-1}^\top]^\top)\dots)))\\
        &=: \e_1^\top \bv_1 \bv_2^\top g(\x).
    \end{align*}
    This gives
    \begin{align*}
        \E_{\bv_1,\bv_2}[\|\hat f_{\Theta^*,H^\pro\oplus[\bv_1,\bv_2],L,p} - f\|_{\mathcal{L}^\infty([0, 1]^p)}] &\geq \E_{\bv_1,\bv_2}[|\hat f_{\Theta^*,H^\pro\oplus[\bv_1,\bv_2],L,p}(\zero_p) - 1|]\\
        &= \E_{\bv_1,\bv_2}[|\e_1^\top \bv_1 \bv_2^\top g(\zero_p) - 1|]\\
        &\geq |\E_{\bv_1,\bv_2}[\e_1^\top \bv_1 \bv_2^\top g(\zero_p)] - 1|\\
        &= 1,
    \end{align*}
    where the second inequality follows from Jensen's inequality, and the last equality follows by symmetry of uniform distribution.
    Now, since $K \sim \operatorname{Poi}(\lambda)$,
    \begin{align}
        &\E_{K, \bv_1,\bv_2,\dots}[\|\hat f_{\Theta^*,H^\pro\oplus[\bv_1,\bv_2,\dots,\bv_K],L,p} - f\|_{\mathcal{L}^\infty([0, 1]^p)}]\nonumber\\
        &\quad\geq \E_{\bv_1,\bv_2}\|\hat f_{\Theta^*,H^\pro\oplus[\bv_1,\bv_2],L,p} - f\|_{\mathcal{L}^\infty([0, 1]^p)} \mid K=2] \P(K=2)\nonumber\\
        &\quad\geq \lambda^2 e^{-\lambda}.\label{eq: irrelevant proof 3}
    \end{align}
    Since $\lambda = \Theta(1)$, \eqref{eq: irrelevant proof 3} completes the proof.
\end{proof}

\begin{proof}[Proof of Corollary~\ref{cor: diversity}]
    The bound in \eqref{eq: diversity ub} directly follows from Corollary~\ref{cor: approximation error in terms of L ap} by rewriting $r+1$ as $r$.
\end{proof}

\begin{proof}[Proof of Corollary~\ref{cor: agents}]
    We first assume $(T^a)_{a \in \mathcal{A}} \in \{2, 4, \dots\}$.
    Setting $r = d \wedge (T^\mathcal{A}/2) - 1$ in Corollary~\ref{cor: approximation error in terms of L ap} gives that for any $f \in \mathcal{C}_\u^\beta([0,1]^p)$, there exist $S^* > 0$ and $H^\pro \in \mathcal{Q}_{\mathcal{U}_{d,T^\mathcal{A}/2}}(T,L,S)$ such that for any $S \geq S^*$,
    \begin{align*}
        \|\hat f_{\Theta^*,H^\pro,L,p} - f\|_{\mathcal{L}^\infty([0, 1]^p)} \leq C \left(\frac{d \wedge (T^\mathcal{A}/2)-1}{\log(8d)}\right)^{-2\beta/p} \left(\frac{L-3p}{\log(4L)}\right)^{-2\beta/p}
    \end{align*}
    holds for some constant $C = C(\beta, p) > 0$. 
    From the proof of Corollary~\ref{cor: approximation error in terms of L ap}, we can take $H^\pro$ such that
    \begin{align*}
        \hat f_{\Theta^*,H^\pro,L,p}(\x) = \e_1^\top W_L \sigma( W_{L-1} \sigma( \dots \sigma( W_1 [\x^\top, 1, \zero_{d-p-1}^\top]^\top ) \dots ))
    \end{align*}
    with some $(W_\ell)_{\ell \in [L]} \subset \R^{d\times d}$ and $\rank(W_\ell) \leq T^\mathcal{A}/2$ for all $\ell \in [L]$.
    Write $W_\ell = \sum_{j \in [\rank(W_\ell)]} \tilde \bu_j^{(\ell)} \bu_j^{(\ell) \top}$.
    We then distribute the prompt tokens in $H^\pro$ to each $H^a$.
    Denote by $\mathcal{A}^{(\ell)} = \{a \in \mathcal{A}: \ell^a = \ell\}$ the group of agents with $\ell^a = \ell$. Without loss of generality, we write $\mathcal{A}^{(\ell)} = \{a_1^{(\ell)}, a_2^{(\ell)}, \dots, a_{|\mathcal{A}^{(\ell)}|}^{(\ell)}\}$. 
    For $\ell \in [L]$, $k \in [|\mathcal{A}^{(\ell)}|]$, define $s_k^{(\ell)} := \min\{T^{a_k^{(\ell)}}/2, \rank(W_\ell) - \sum_{k' < k} T^{a_{k'}^{(\ell)}}/2\}$. We set $H^{a_k^{(\ell)}} \subset \R^{(4d+8)\times \rank(W_\ell)}$ as
    \begin{align*}
        (H^{a_k^{(\ell)}})_{2s-1} &= \begin{pmatrix}
            \tilde \bu^{(\ell)}_{s+\sum_{k' < k} T^{a_{k'}^{(\ell)}}}\\
            p(2\ell-1, 2s-1, S)
        \end{pmatrix}, \ \ 
        (H^{a_k^{(\ell)}})_{2s} = \begin{pmatrix}
            \bu^{(\ell)}_{s+\sum_{k' < k} T^{a_{k'}^{(\ell)}}}\\
            p(2\ell, 2s, S)
        \end{pmatrix} \ \text{ for $s \in [s_k^{(\ell)}]$},\\
        (H^{a_k^{(\ell)}})_{2s-1} &= \begin{pmatrix}
            \zero_d\\
            p(2\ell-1, 2s-1, S)
        \end{pmatrix}, \ \ 
        (H^{a_k^{(\ell)}})_{2s} = \begin{pmatrix}
            \zero_d\\
            p(2\ell, 2s, S)
        \end{pmatrix} \ \text{ for $s \in [T^{a_k^{(\ell)}}/2]\setminus[s_k^{(\ell)}]$}.
    \end{align*}
    Then, we can verify that $H^a \in \mathcal{R}_{\R^d}(T^a, \ell^a, S)$ for all $a \in \mathcal{A}$.
    From Corollary~\ref{cor: prompt engineering random ap}, \eqref{eq: H A}, and Definition~\ref{def: function approx}, if $S \geq S^* \vee d B^{4L} T^{2L} \vee 2L$, with $T := \sum_{a \in \mathcal{A}} T^a$, then,
    \begin{align*}
        \hat f_{\Theta^*,H^\mathcal{A},L,p}(\x) = \hat f_{\Theta^*,H^\pro,L,p}(\x),
    \end{align*}
    and hence
    \begin{align*}
        \|\hat f_{\Theta^*,H^\mathcal{A},L,p} - f\|_{\mathcal{L}^\infty([0, 1]^p)} \leq C \left(\frac{d \wedge (T^\mathcal{A}/2)-1}{\log(8d)}\right)^{-2\beta/p} \left(\frac{L-3p}{\log(4L)}\right)^{-2\beta/p}.
    \end{align*}
    Since $f$ and $S > 0$ are arbitrary, we obtain the desired result.
    The proof when $(T^a)_{a \in \mathcal{A}} \in \N$ follows by a similar argument.
\end{proof}

\section{Details of Experiments}\label{sec: experiments ap}

In this section, we provide details of experiments presented in Section~\ref{sec: application}. We use three mathematical reasoning datasets for the empirical evaluation, including GSM8K \citep{cobbe2021gsm8k}, AQUA-RAT \citep{ling2017program}, and MATH \citep{hendrycksmath2021} datasets. For the GSM8K and MATH datasets, we randomly sample 200 questions; for the AQUA-RAT dataset, we use the test dataset, which consists of 254 questions. Following \citet{naik2023diversity}, we select the questions from the counting and probability category from the MATH dataset. Example questions are shown as follows:

\paragraph{GSM8K} \citet{cobbe2021gsm8k}
\begin{itemize}
    \item \textbf{Question}: Natalia sold clips to 48 of her friends in April, and then she sold half as many clips in May. How many clips did Natalia sell altogether in April and May?
    \item \textbf{Answer}: Natalia sold 48/2 = <<48/2=24>>24 clips in May. Natalia sold 48+24 = <<48+24=72>>72 clips altogether in April and May. The answer is 72. 
\end{itemize}
\paragraph{AQUA-RAT} \citep{ling2017program}
\begin{itemize}
    \item \textbf{Question}: A grocery sells a bag of ice for \$1.25, and makes 20\% profit. If it sells 500 bags of ice, how much total profit does it make? 
    \item \textbf{Options}: A)125, B)150, C)225, D)250, E)275
    \item  \textbf{Rationale}: Profit per bag = 1.25 * 0.20 = 0.25. Total profit = 500 * 0.25 = 125. Answer is A.
    \item  \textbf{Correct}: A
\end{itemize}
\paragraph{MATH} \citep{hendrycksmath2021}
\begin{itemize}
    \item \textbf{Problem}: The probability of rain tomorrow is $\frac{1}{11}$.  What is the probability that it will not rain tomorrow?  Express your answer as a common fraction.
    \item \textbf{Solution}: It must either rain tomorrow or not rain tomorrow, so the sum of the probability that it rains and the probability it doesn't rain is 1.  Therefore, the probability it doesn't rain is $1 - \frac{1}{11} = \frac{10}{11}$.
\end{itemize}

For the experiment in Table \ref{tab: GPT length}, we attach the question to the selected examples and prompt the language model to answer the question with
\begin{verbatim}
    Follow the given examples, think step by step, and answer the question 
    in the format: the answer is <your answer>.
\end{verbatim}

For the filter out experiment in Table \ref{tab: filter out gsm8k}, we prompt the language model as follows
\begin{verbatim}
 Please solve the following question. Feel free to ignore irrelevant text
 given in the question. Provide the answer in the format: the answer is
 <your answer>   
\end{verbatim}
to encourage the language model to filter out irrelevant information. 

For the diversity of thoughts experiment in Table \ref{tab: diversity of thoughts}, we use the following prompt
\begin{verbatim}
Question: {question}. Please think about 3 distinct approaches that are suitable 
for yourself to solve the problem. Provide the name of these approaches only.    
\end{verbatim}
 to obtain three different approaches that are suitable for the problem, and then prompt the language models to solve the problem as follows 
\begin{verbatim}
Question: {question}
Answer: As a math professor, use 3 given approaches to solve the given problem: 
<approaches obtained from previous step>, 4. directly solve the question via
thinking step by step approach. Provide detailed and step-by-step computation 
in each approach. Output format:
Approach 1 <name of the approach> : < solution using approach 1 >
Approach 2 <name of the approach> : < solution using approach 2 >
Approach 3 <name of the approach> : < solution using approach 3 >
Approach 4 <name of the approach> : < solution using approach 4 >
the answer is <your answer>.    
\end{verbatim}

\end{document}